\documentclass{article}

\PassOptionsToPackage{square,sort, comma,numbers, compress}{natbib}

\usepackage{amsthm,amsmath,amssymb,dsfont}
\usepackage[hidelinks]{hyperref}

\newcommand{\cA}{\mathcal{A}}

\newcommand{\cS}{\mathcal{S}}
\newcommand{\cF}{\mathcal{F}}
\newcommand{\cM}{\mathcal{M}}

\newcommand{\cP}{\mathcal{P}}

\newcommand{\cX}{\mathcal{X}}

\newcommand{\expec}{\mathbb{E}}

\newcommand{\real}{\mathbb{R}}
\newcommand{\reals}{\mathbb{R}}

\newcommand{\norm}[1]{\left\|{#1}\right\|}
\newcommand{\abs}[1]{\left|{#1}\right|}
\newcommand*{\para}[1]{\left({#1}\right)}
\newcommand*{\infn}[1]{\left\|{#1}\right\|_{\infty}}
\DeclareMathOperator*{\argmin}{argmin}
\DeclareMathOperator*{\argmax}{argmax}

\newtheorem{theorem}{Theorem}

\newtheorem{lemma}{Lemma}
\newtheorem{fact}{Fact}

\usepackage[final]{neurips_2023}

\usepackage[utf8]{inputenc} % allow utf-8 input
\usepackage[T1]{fontenc}    % use 8-bit T1 fonts
\usepackage{hyperref}       % hyperlinks
\usepackage{url}            % simple URL typesetting
\usepackage{booktabs}       % professional-quality tables
\usepackage{amsfonts}       % blackboard math symbols
\usepackage{nicefrac}       % compact symbols for 1/2, etc.
\usepackage{microtype}      % microtypography
\usepackage{xcolor}         % colors

\title{Accelerating Value Iteration with Anchoring}

\author{
  Jongmin Lee${}^{1}$
  \qquad
  \qquad
  Ernest K. Ryu${}^{1,2}$
  \\[0.2em]
    ${}^1$Department of Mathematical Science, Seoul National University\\
${}^2$Interdisciplinary Program in Artificial Intelligence, Seoul National University
}

\begin{document}

\maketitle

\begin{abstract}
Value Iteration (VI) is foundational to the theory and practice of modern reinforcement learning, and it is known to converge at a $\mathcal{O}(\gamma^k)$-rate, where $\gamma$ is the discount factor. Surprisingly, however, the optimal rate in terms of Bellman error for the VI setup was not known, and finding a general acceleration mechanism has been an open problem. In this paper, we present the first accelerated VI for both the Bellman consistency and optimality operators. Our method, called Anc-VI, is based on an \emph{anchoring} mechanism (distinct from Nesterov's acceleration), and it reduces the Bellman error faster than standard VI. In particular, Anc-VI exhibits a $\mathcal{O}(1/k)$-rate for $\gamma\approx 1$ or even $\gamma=1$, while standard VI has rate $\mathcal{O}(1)$ for $\gamma\ge 1-1/k$, where $k$ is the iteration count. We also provide a complexity lower bound matching the upper bound up to a constant factor of $4$, thereby establishing optimality of the accelerated rate of Anc-VI. Finally, we show that the anchoring mechanism provides the same benefit in the approximate VI and Gauss--Seidel VI setups as well.
\end{abstract}

\section{Introduction}

Value Iteration (VI) is foundational to the theory and practice of modern dynamic programming (DP) and reinforcement learning (RL). It is well known that when a discount factor $\gamma<1$ is used, (exact) VI is a contractive iteration in the $\infn{\cdot}$-norm and therefore converges. The progress of VI is measured by the Bellman error in practice (as the distance to the fixed point is not computable), and much prior work has been dedicated to analyzing the rates of convergence of VI and its variants.

Surprisingly, however, the optimal rate in terms of Bellman error for the VI setup was not known, and finding a general acceleration mechanism has been an open problem. The classical $\mathcal{O}(\gamma^k)$-rate of VI is inadequate as many practical setups use $\gamma\approx1$ or $\gamma=1$ for the discount factor. (Not to mention that VI may not converge when $\gamma=1$.) Moreover, most prior works on accelerating VI focused on the Bellman consistency operator (policy evaluation) as its linearity allows eigenvalue analyses, but the Bellman optimality operator (control) is the more relevant object in modern RL.

\paragraph{Contribution.}
In this paper, we present the first accelerated VI for both the Bellman consistency and optimality operators. Our method, called  \ref{eq:anc-vi}, is based on an ``anchoring'' mechanism (distinct from Nesterov's acceleration), and it reduces the Bellman error faster than standard VI. In particular,  \ref{eq:anc-vi} exhibits a $\mathcal{O}(1/k)$-rate for $\gamma\approx 1$ or even $\gamma=1$, while standard VI has rate $\mathcal{O}(1)$ for $\gamma\ge 1-1/k$, where $k$ is the iteration count. We also provide a complexity lower bound matching the upper bound up to a constant factor of $4$, thereby establishing optimality of the accelerated rate of \ref{eq:anc-vi}. Finally, we show that the anchoring mechanism provides the same benefit in the approximate VI and Gauss--Seidel VI setups as well.

% where such choice is occasionally necessary in reinforcement learning setup.under the mild condition, Anc-VI also converges to a fixed point for the case $\gamma=1$.

%Theoretically, Anchored VI is optiaml method up to constant.

\subsection{Notations and preliminaries}

We quickly review basic definitions and concepts of Markov decision processes 
 (MDP) and reinforcement learning (RL). For further details, refer to standard references such as \cite{10.5555/528623, szepesvarialgorithms, sutton2018reinforcement}. 

\paragraph{Markov Decision Process.}
Let $\cM(\cX)$ be the space of probability distributions over $\cX$. Write $(\cS, \cA, P, r, \gamma)$ to denote the MDP with state space $\cS$, action space $\cA$, transition probability $P\colon \cS \times \cA \rightarrow \cM(\cS)$, reward $r\colon  \cS \times \cA \rightarrow \real$, and discount factor $\gamma \in (0,1]$. Denote $\pi\colon \cS \rightarrow \cM(\cA)$ for a policy, $V^{\pi}(s)=\expec_{\pi}[\sum^{\infty}_{t=0} \gamma^t r(s_t, a_t) \,|\, s_0=s]$ and $Q^{\pi}(s, a)=\expec_{\pi}[\sum^{\infty}_{t=0} \gamma^t r(s_t, a_t) \,|\, s_0=s, a_0=a]$ for $V$- and $Q$-value functions,  where $\expec_{\pi}$ denotes the expected value over all trajectories $(s_0, a_0, s_1, a_1, \dots)$ induced by $P$ and $\pi$. We say $V^{\star}$ and $Q^{\star}$ are optimal $V$- and $Q$- value functions if $V^{\star}=\sup_{\pi}V^{\pi}$and $Q^{\star}=\sup_{\pi}Q^{\pi}$. We say $\pi_V^{\star}$ and $ \pi_Q^{\star}$ are optimal policies if $\pi_V^{\star}=\argmax_{\pi}{V^{\pi}}$ and $ \pi_Q^{\star}= \argmax_{\pi}{Q^{\pi}}$. (If argmax is not unique, break ties arbitrarily.) 
 
 \paragraph{Value Iteration.}
  Let $\cF(\cX)$ denote the space of bounded measurable real-valued functions over $\cX$. With the given MDP $(\cS, \cA, P, r, \gamma)$, for $V \in \cF(\cS)$ and $Q \in \cF(\cS \times \cA)$, define the Bellman consistency operators $T^{\pi}$ as
\begin{align*}
T^{\pi}V(s)&=\mathbb{E}_{a \sim \pi(\cdot\,|\,s), s'\sim P(\cdot\,|\,s,a) }\left[r(s,a)+\gamma V(s')\right],\\
T^{\pi}Q(s,a)&=r(s,a)+\gamma\mathbb{E}_{s'\sim P(\cdot\,|\,s,a),a' \sim \pi(\cdot\,|\,s')}\left[Q(s',a')\right]
\end{align*}
for all $s \in \cS, a \in \cA$, and the Bellman optimality operators $T^{\star}$ as
\begin{align*}
T^{\star}V(s)&=\sup_{a \in \cA} \left\{r(s,a)+\gamma\mathbb{E}_{s'\sim P(\cdot\,|\,s,a) }\left[V(s')\right]\right\},\\
T^{\star}Q(s,a)&=r(s,a)+\gamma\mathbb{E}_{s'\sim P(\cdot\,|\,s,a)}\left[\sup_{a' \in \cA} Q(s',a')\right]
\end{align*}
for all $s \in \cS, a \in \cA$. For notational conciseness, we write $T^{\pi}V=r^{\pi}+\gamma\cP^{\pi}V$ and $T^{\pi}Q=r+\gamma\cP^{\pi}Q$, where $r^{\pi}(s)=\mathbb{E}_{a \sim \pi(\cdot\,|\,s) }\left[r(s,a)\right]$ is the reward induced by policy $\pi$ and $\cP^\pi(s)$ and $\cP^\pi(s,a)$ defined as
\begin{align*}
    \cP^{\pi}(s\rightarrow s')&=
\mathrm{Prob}(s\rightarrow s'\,|\,
a \sim \pi(\cdot\,|\,s), s'\sim P(\cdot\,|\,s,a))
\\
\cP^{\pi}((s,a)\rightarrow (s',a'))&=
\mathrm{Prob}((s,a)\rightarrow (s',a')\,|\,
s'\sim P(\cdot\,|\,s,a),a' \sim \pi(\cdot\,|\,s')),
\end{align*}
are the transition probabilities induced by policy $\pi$. We define VI for Bellman consistency and optimality operators as 
\[V^{k+1}=T^{\pi}V^{k}, \quad Q^{k+1}=T^{\pi}Q^{k}, \quad V^{k+1}=T^{\star}V^{k}, \quad Q^{k+1}=T^{\star}Q^{k} \qquad\text{ for } k=0,1,\dots,\]
 where $V^0, Q^0$ are initial points.
VI for control, after executing $K$ iterations, returns the near-optimal policy $\pi_K$ as a greedy policy satisfying \[
T^{\pi_K}V^K= T^{\star}V^K,\quad T^{\pi_K}Q^K= T^{\star}Q^K.
\]
 % We also denote VI for Bellman consistency and optimality operators as VI for policy evaluation and control, by convention.
 For $\gamma<1$, both Bellman consistency and optimality operators are contractions, and, by Banach's fixed-point theorem \cite{banach1922operations}, the VIs converge to the unique fixed points $V^{\pi}$, $Q^{\pi}, V^{\star}$, and $Q^{\star}$ with $\mathcal{O}(\gamma^k)$-rate. For notational unity, we use the symbol $U$ when both $V$ and $Q$ can be used.
 Since $\infn{TU^k-U^k} \le \infn{TU^k -U^\star}+\infn{U^k -U^\star} \le (1+\gamma)\infn{U^k-U^{\star}}$,  VI exhibits the rate on the Bellman error:
\begin{align*}
    \infn{TU^k-U^k} &\le (1+\gamma)\gamma^k\infn{U^0-U^{\star}} \qquad\text{ for } k=0,1,\dots,  \tag{$1$} \label{eq:rate-vi-1}
\end{align*} 
% and if $U^0 \le TU^0$ or $U^0 \ge TU^0$, by Lemma X in Appendix, VI exhibits the rate on the Bellman error:
% \begin{align*}
%     \infn{TU^k-U^k} &\le \gamma^k\infn{U^0-U^{\star}} \qquad\text{ for } k=0,1,\dots,  \tag{$2$} \label{eq:rate-vi-2}
% \end{align*} 
where $T$ is Bellman consistency or optimality operator, $U^0$ is a starting point, and $U^{\star}$ is fixed point of $T$. We say $V \le V'$ or $Q\le Q'$ if $V(s)\le V'(s)$ or $Q (s,a)\le Q'(s,a)$ for all $s \in \cS$ and $a \in \cA$, respectively.

\paragraph{Fixed-point iterations.}
Given an operator $T$, we say $x^{\star}$ is fixed point if $Tx^{\star}=x^{\star}$. 
Since Banach \cite{banach1922operations}, the standard fixed-point iteration 
\[x^{k+1}=Tx^{k}\qquad\text{ for } k=0,1,\dots\]
has been commonly used to find fixed points.
Note that VI for policy evaluation and control are fixed-point iterations with Bellman consistency and optimality operators.
In this work, we also consider the Halpern iteration
\[ x^{k+1}=\beta_{k+1} x^0+ (1-\beta_{k+1}) Tx^{k}
\qquad\text{ for } k=0,1,\dots,\]
where $x^0$ is an initial point and $\{\beta_k\}_{k \in \mathbf{N}} \in (0,1)$.

\subsection{Prior works}

 \paragraph{Value Iteration.} Value iteration (VI) was first introduced in the DP literature \cite{bellman1957markovian} for finding optimal value function, and its variant approximate VI \cite{bertsekas1995neuro, ernst2005approximate, munos2005error, farahmand2010error, de2000existence, van2006performance, sutton2018reinforcement} considers approximate evaluations of the Bellman optimality operator. In RL, VI and approximate VI have served as the basis of RL algorithms such as fitted value iteration \cite{ernst2005tree, munos2008finite, massoud2009regularized, tosatto2017boosted, Lutter2021ValueII, gordon1995stable} and temporal difference learning \cite{sutton1988learning, van2016deep, hessel2018rainbow, watkins1992q, mnih2015human}. There is a line of research that emulates VI by learning a model of the MDP dynamics \cite{tamar2016value, sykora2020multi, niu2018generalized} and applying a modified Bellman operator \cite{bellemare2016increasing, fellows2021bayesian}. Asynchronous VI, another variation of VI updating the coordinate of value function in asynchronous manner, has also been studied in both RL and DP literature \cite{bertsekas1995neuro, bertsekas2015parallel, tsitsiklis1994asynchronous, zeng2020asyncqvi}.

\paragraph{Fixed-point iterations.}
The Banach fixed-point theorem \cite{banach1922operations} establishes the convergence of the standard fixed-point iteration with a contractive operator. 
%As generalization of Picard iteration, Kransnosel'skii-Mann iteration (KM) \cite{mann1953mean, krasnosel1955two} was introduced and its convergence with general nonexpansive operators was shown by \cite{Martinet1970ppm}. 
The Halpern iteration \cite{halpern1967fixed} converges for \textit{nonexpansive} operators on Hilbert spaces \cite{wittmann1992approximation} and uniformly smooth Banach spaces \cite{reich1980strong, xu2002iterative}.
(To clarify, the $\|\cdot\|_\infty$-norm in $\mathbb{R}^n$ is not uniformly smooth.)

The fixed-point residual $\norm{Tx_k-x_k}$ is a commonly used error measure for fixed-point problems. In general normed spaces, 
%KM iteration with nonexpansive opeator was proven to exhibit $\mathcal{O}(1/\sqrt{k})$-rate \cite{baillon1992optimal, cominetti2014rate, bravo2018sharp}. For 
the Halpern iteration was shown to exhibit $\mathcal{O}(1/\log(k))$-rate for (nonlinear) nonexpansive operators \cite{leustean2007rates} and $\mathcal{O}(1/k)$-rate for linear nonexpansive operators \cite{contreras2022optimal} on the fixed-point residual.
In Hilbert spaces,
%KM iteration with nonexpansive operator was shown to exhibit $o(1/\sqrt{k})$-rate by \cite{ matsushita2017convergence}. %$\mathcal{O}(1/k)$-rate of Halpern iteration was proven by \cite{kohlenbach2011quantitative}.
\cite{sabach2017first} first established a $\mathcal{O}(1/k)$-rate for the Halpern iteration and the constant was later improved by \cite{Lieder2021halpern, kim2021accelerated}. For contractive operators, \cite{park2022exact} proved exact optimality of Halpern iteration through an exact matching complexity lower bound. 
 
% We discuss the relationship of our work with \cite{park2022exact} and
% \cite{contreras2022optimal}. \cite{park2022exact} presented \emph{OC-Halpern}, exact optimal method in $\real^n$ with Euclidean norm. Anc-VI and OC-Halpern share same parameters but Anc-VI iterates in $\real^n$ with $\infn{\cdot}$. Therefore, convergence analyses for OC-Halpern could not be applied for Anc-VI. \cite{contreras2022optimal}
% proved $\mathcal{O}(1/k)$ convergence rate of Halpern iteration with nonexpansive linear operator and proposed complexity lower bound in general normed space. Some result of \cite{contreras2022optimal} holds for nonexpansive Bellman consistency operator, but in this work, we mainly focus on $\gamma$-contractive operator and our results also hold for Bellman optimality operator which is non linear operator.

\paragraph{Acceleration.}
Since Nesterov's seminal work \cite{nesterov1983method}, there has been a large body of research on acceleration in convex minimization. Gradient descent \cite{cauchy1847methode} can be accelerated to efficiently reduce function value and squared gradient magnitude for smooth convex minimization problems \cite{nesterov1983method,kim2016optimized, kim2021optimizing, lee2021geometric, zhou2022practical, doi:10.1137/21M1395302, nesterov2021primal} and smooth strongly convex minimization problems \cite{nesterov2003Introductory, van2017fastest, park2021factor, taylor2021optimal, salim2021optimal}. Motivated by Nesterov acceleration, inertial fixed-point iterations \cite{mainge2008convergence, dong2018modified, shehu2018convergence, reich1980strong, iutzeler2019generic} have also
been suggested to accelerate fixed-point iterations. Anderson acceleration \cite{anderson1965}, another acceleration scheme
for fixed-point iterations, has recently been studied with interest \cite{barre2020convergence, scieur2020regularized, walker2011anderson, zhang2020globally}. 

In DP and RL, prioritized sweeping \cite{moore1993prioritized} is a well-known method that changes the order of updates to accelerate convergence, and several variants \cite{peng1993efficient, mcmahan2005fast,wingate2005prioritization, andre1997generalized, dai2011topological} have been proposed. Speedy Q-learning \cite{azar2011speedy} modifies the update rule of Q-learning and uses aggressive learning rates for acceleration. 
% \cite{shlakhter2010acceleration} suggested variants of VI using class of operators satisfying $Zv \le v$ for all $v$, and experimentally show acceleration. 
Recently, there has been a line of research that applies acceleration techniques of other areas to VI: \cite{geist2018anderson, sun2021damped, ermis2020a3dqn, park2022anderson, ermis2021anderson1, shi2019regularized} uses Anderson acceleration of fixed-point iterations, \cite{vieillard2020momentum, goyal2022first, grand2021convex, pmlr-v161-bowen21a, doi:10.1137/20M1367192} uses Nesterov acceleration of convex optimization, and \cite{farahmand2021pid} uses ideas inspired by PID controllers in control theory. Among those works, \cite{goyal2022first, grand2021convex, doi:10.1137/20M1367192} applied Nesterov acceleration to obtain theoretically accelerated convergence rates, but those analyses require certain reversibility conditions or restrictions on eigenvalues of the transition probability induced by the policy.    

The \emph{anchor acceleration}, a new acceleration mechanism distinct from Nesterov's, lately gained attention in convex optimization and fixed-point theory. The anchoring mechanism, which retracts iterates towards the initial point, has been used to accelerate algorithms for minimax optimization and fixed-point problems \cite{ryu2019ode, lee2021fast, yoon2021accelerated, park2022exact, kim2021accelerated, diakonikolas2020halpern, yoon2022accelerated, suh2023continuous}, and we focus on it in this paper.

\paragraph{Complexity lower bound.}
With the information-based
complexity analysis \cite{nemirovski1992information}, complexity lower bound on first-order methods for convex minimization problem has been thoroughly studied \cite{nesterov2003Introductory, drori2017exact, drori2022oracle, carmon2020stationary1, carmon2021stationary2, drori2020stoc-complexity}. If a complexity lower bound
matches an algorithm's convergence rate, it establishes optimality
of the algorithm \cite{nemirovski1992information, kim2016optimized,salim2021optimal, taylor2021optimal, drori2016optimal-kelley, park2022exact}. In fixed-point problems, \cite{colao2021rate} established $\Omega(1/k^{1-\sqrt{2/q}})$ lower bound on distance to solution for Halpern iteration with a nonexpansive operator in $q$-uniformly smooth Banach spaces. In \cite{contreras2022optimal}, a  $\Omega(1/k)$ lower bound on the fixed-point residual for the general Mann iteration with a nonexpansive linear operator, which includes standard fixed-point iteration and Halpern iterations, in the  $\ell^{\infty}$-space was provided.  In Hilbert spaces, 
%\cite{doi:10.1137/21M1395302} proved $\Omega\para{1/\sqrt{k}}$ lower bound for the KM iteration among 1-SCLI algorithms \cite{arjevani2016lower}. 
\cite{park2022exact} showed exact complexity lower bound on fixed-point residual for deterministic fixed-point iterations with $\gamma$-contractive and nonexpansive operators. Finally, \cite{goyal2022first} provided lower bound on distance to optimal value function for fixed-point iterations satisfying span condition with Bellman consistency and optimality operators and we discussed this lower bound in section \ref{sec::complexity}. 

\section{Anchored Value Iteration}\label{sec::Anc-VI}

Let $T$ be a $\gamma$-contractive (in the $\|\cdot\|_\infty$-norm) Bellman consistency or optimality operator. The \emph{Anchored Value Iteration} \eqref{eq:anc-vi} is
\begin{align}
   U^k=\beta_kU^0+(1-\beta_k)TU^{k-1}\qquad \tag{Anc-VI}
   \label{eq:anc-vi}
\end{align} 
for $k=1,2,\dots,$ where $\beta_k=1/(\sum_{i=0}^k \gamma^{-2i})$ and $U^0$ is an initial point.  In this section, we present accelerated convergence rates of \ref{eq:anc-vi} for \emph{both} Bellman consistency and optimality operators for both $V$- and $Q$-value iterations.
For the control setup, where the Bellman optimality operator is used, \ref{eq:anc-vi} returns the near-optimal policy $\pi_K$ as a greedy policy satisfying $T^{\pi_K}U^K= T^{\star}U^K$ after executing $K$ iterations.

% Of course, the standard VI has the form $U^k=TU^{k-1} $, 

Notably, \ref{eq:anc-vi} obtains the next iterate as a convex combination between the output of $T$ and the starting point $U^0$.
We call the $\beta_k U_0$ term the \emph{anchor term} since, loosely speaking,  it serves to pull the iterates toward the starting point $U_0$. The strength of the anchor mechanism diminishes as the iteration progresses since $\beta_k$ is a decreasing sequence.

% The anchor mechanism was introduced in prior \cite{halpern1967fixed,sabach2017first,Lieder2021halpern,park2022exact, contreras2022optimal, leustean2007rates} for $\|\cdot\|_2$-nonexpansive and contractive operators, where $\|\cdot\|_2$ denotes the Euclidean norm, and $\|\cdot\|_{\infty}$-nonexpansive operators. While our anchor mechanism does bear a formal resemblance to those of prior works, our convergence rates are neither a direct application nor a direct adaptation of the prior convergence analyses. The prior analyses apply to generic $\|\cdot\|_2$-nonexpansive or contractive operators and $\|\cdot\|_\infty$-nonexpansive operators. Our analyses do not apply to generic non-Bellman $\|\cdot\|_\infty$-nonexpansive or contractive operators. Rather, we specifically utilize the structure of Bellman operators to obtain the rates specifically for Bellman operators.
The anchor mechanism was introduced \cite{halpern1967fixed,sabach2017first,Lieder2021halpern,park2022exact, contreras2022optimal, leustean2007rates} for general nonexpansive operators and $\|\cdot\|_2$-nonexpansive and contractive operators. The optimal method for $\|\cdot\|_2$-nonexpansive and contractive operators in \cite{park2022exact} shares the same coefficients with  \ref{eq:anc-vi}, and convergence results for general nonexapnsive operators in \cite{contreras2022optimal, leustean2007rates} are applicable to  \ref{eq:anc-vi} for nonexpansive Bellman optimality and consistency operators. While our anchor mechanism does bear a formal resemblance to those of prior works, our convergence rates and point convergence are neither a direct application nor a direct adaptation of the prior convergence analyses. The prior analyses for $\|\cdot\|_2$-nonexpansive and contractive operators do not apply to Bellman operators, and prior analyses for general nonexpansive operators have slower rates and do not provide point convergence while our Theorem~\ref{thm::nonexp_finite} does. Our analyses specifically utilize the structure of Bellman operators to obtain the faster rates and point convergence. 

The accelerated rate of \ref{eq:anc-vi} for the Bellman \emph{optimality} operator is more technically challenging and is, in our view, the stronger contribution. However, we start by presenting the result for the Bellman \emph{consistency} operator because it is commonly studied in the prior RL theory literature on accelerating value iteration \cite{goyal2022first, grand2021convex, doi:10.1137/20M1367192, farahmand2021pid} and because the analysis in the Bellman consistency setup will serve as a good conceptual stepping stone towards the analysis in the Bellman optimality setup.

\subsection{Accelerated rate for Bellman consistency operator}
First, for general state-action spaces, we present the accelerated convergence rate of \ref{eq:anc-vi} for the Bellman consistency operator.

 \begin{theorem}\label{thm::Anc-VIE}
Let $0<\gamma<1$ be the discount factor and $\pi$ be a policy.
Let $T^{\pi}$ be the Bellman consistency operator for $V$ or $Q$. Then,
  \ref{eq:anc-vi} exhibits the rate 
\begin{align*}
     \infn{T^{\pi}U^k-U^k} &\le \frac{\para{\gamma^{-1}-\gamma}\para{1+2\gamma-\gamma^{k+1}}}{\para{\gamma^{k+1}}^{-1}-\gamma^{k+1}}\infn{U^0-U^{\pi}}\\&= \para{\frac{2}{k+1}+\frac{k-1}{k+1}\epsilon+O(\epsilon^2)} \infn{U^0-U^{\pi}}
     \qquad\text{ for }k=0,1,\dots,
%  \\
% \|U^{\star}-x^n\| &\le \frac{1}{\sum_{k=0}^n \gamma^{k}}\left(1+\frac{1}{\gamma}\right)\left(1+\frac{2}{\gamma}\right)\left(1-\frac{1}{\gamma}\right)^{-1}\|x^0-U^{\star}\|.
 \end{align*}
where $\epsilon = 1-\gamma$ and the big-$\mathcal{O}$ notation considers the limit $\epsilon\rightarrow 0$.
If, furthermore, $U^0 \le T^{\pi}U^0$ or $U^0\ge T^{\pi}U^0$, then \ref{eq:anc-vi} exhibits the rate 
\begin{align*}
     \infn{T^{\pi}U^k-U^k} &\le \frac{\para{\gamma^{-1}-\gamma}\para{1+\gamma-\gamma^{k+1}}}{\para{\gamma^{k+1}}^{-1}-\gamma^{k+1}}\infn{U^0-U^{\pi}}\\&= \para{\frac{1}{k+1}+\frac{k}{k+1}\epsilon+O(\epsilon^2)} \infn{U^0-U^{\pi}}
     \qquad\text{ for }k=0,1,\dots.
     %\quad \text{if}\,\,  T^{\pi}U^0\ge U^0 \,\,\text{or}\,\,T^{\pi}U^0\le U^0 .
%  \\
% \|U^{\star}-x^n\| &\le \frac{1}{\sum_{k=0}^n \gamma^{k}}\left(1+\frac{1}{\gamma}\right)\left(1+\frac{2}{\gamma}\right)\left(1-\frac{1}{\gamma}\right)^{-1}\|x^0-U^{\star}\|.
 \end{align*}
\end{theorem}  
  %If $n \ge \frac{\log{\para{2\gamma-1}}}{\log{\gamma}}$, the theoretical convergence rate of Anc-VI, is faster than rate of VI for $\gamma > \frac{1}{2}$.
  If $\gamma \ge \frac{1}{2}$, both rates of Theorem~\ref{thm::Anc-VIE} are strictly faster than the standard rate \eqref{eq:rate-vi-1} of VI, since 
  \begin{align*}
      \frac{\para{\gamma^{-1}-\gamma}\para{1+2\gamma-\gamma^{k+1}}}{\para{\gamma^{k+1}}^{-1}-\gamma^{k+1}}=\gamma^k \frac{\para{1-\gamma^2}\para{1+2\gamma-\gamma^{k+1}}}{\para{1-\gamma^{2k+2}}} < \gamma^k (1+\gamma)
  .
  \end{align*}
  %However, tight upper bound of Anc-VI, $ \frac{\gamma^{k}\para{1+\gamma}\para{1+2\gamma-\gamma^{k+1}}}{\para{\sum_{i=0}^k \gamma^{i}}\para{1+\gamma^{k+1}}}\infn{U^0-U^{\pi}}$, is faster than VI as long as $\gamma \ge 1/2$. 
  %If $\gamma \ge (\sqrt{5}-1)/2$ and $k \ge \frac{\log{\para{\gamma^2+\gamma-1}}}{2\log{\gamma}}-\frac{1}{2}$, 
  The second rate of Theorem~\ref{thm::Anc-VIE}, which has the additional requirement, is faster than the standard rate \eqref{eq:rate-vi-1} of VI for all $0<\gamma<1$. 
  Interestingly, in the  $\gamma\approx1$ regime, \ref{eq:anc-vi} achieves $\mathcal{O}(1/k)$-rate while VI has a $\mathcal{O}(1)$-rate. We briefly note that the condition $U^0 \le TU^0$ and $U^0 \ge TU^0$ have been used in analyses of variants of VI \protect{\cite[Theorem~6.3.11]{10.5555/528623}, \cite[p.3]{shlakhter2010acceleration}}.

In the following, we briefly outline the proof of Theorem~\ref{thm::Anc-VIE} while deferring the full description to Appendix~\ref{s::omitted-Anc-VI-proofs}. In the outline, we highlight a particular step, labeled $\blacktriangle$, that crucially relies on the linearity of the Bellman consistency operator. In the analysis for the Bellman optimality operator of Theorem~\ref{thm::Anc-VIC}, resolving the $\blacktriangle$ step despite the nonlinearity is the key technical challenge.

\begin{proof}[Proof outline of Theorem \ref{thm::Anc-VIE}]
Recall that we can write Bellman consistency operator as $T^{\pi}V=r^{\pi}+\gamma\cP^{\pi}V$ and $T^{\pi}Q=r+\gamma\cP^{\pi}Q$.
Since $T^{\pi}$ is a linear operator\footnote{Arguably, $T^\pi$ is affine, not linear, but we follow the convention of \cite{10.5555/528623} say $T^\pi$ is linear.}, we get
\begin{align*}
    T^{\pi}U^k-U^k&=T^{\pi}U^{k}-(1-\beta_k)T^{\pi}U^{k-1}-\beta_kT^{\pi}U^{\pi}-\beta_k(U^0-U^{\pi})\\
    &\stackrel{\blacktriangle}{=}\gamma\cP^{\pi}(U^{k}-(1-\beta_k)U^{k-1}-\beta_kU^{\pi})-\beta_k(U^0-U^{\pi})\\
    &=\gamma\cP^{\pi}(\beta_k(U^0-U^{\pi})+(1-\beta_k)(T^{\pi}U^{k-1}-U^{k-1}))-\beta_k(U^0-U^{\pi})\\
    &=\sum_{i=1}^k\left[\para{\beta_i-\beta_{i-1}(1-\beta_{i})}\para{\Pi^k_{j=i+1}(1-\beta_j)}\para{\gamma\cP^{\pi}}^{k-i+1}(U^0-U^{\pi})\right]\\
    &\quad -\beta_k(U^0-U^{\pi})+\para{\Pi^k_{j=1}(1-\beta_j)} \para{\gamma\cP^{\pi}}^{k+1}(U^0-U^{\pi}),
\end{align*}
%\[U^k=\sum_{i=0}^k \beta_i \para{\Pi^k_{j=i+1}(1-\beta_j)} \para{T^{\pi}}^{k-i}U^0,\quad T^{\pi}U^k=\sum_{i=0}^k \beta_i \para{\Pi^k_{j=i+1}(1-\beta_j)} \para{T^{\pi}}^{k-i+1}U^0.\]
where the first equality follows from the definition of \ref{eq:anc-vi} and the property of fixed point, while the last equality follows from induction.
% Using property of fixed point, we have    
% \begin{align}
%     T^{\pi}U^k-U^k&=\sum_{k=0}^n\beta_i\para{\Pi^k_{j=i+1}(1-\beta_j)} \left((\gamma\cP^{\pi})^{k-i+1}(U^0-U^{\star})-(\gamma\cP^{\pi})^{k-i}(U^0-U^{\star})\right)\label{Bellman-consitency}
%     \\&=\sum_{k=1}^n[\beta_k-\beta_{k-1}(1-\beta_{k})]\Pi^n_{i=k+1}(1-\beta_i)\para{\gamma\cP^{\pi}}^{n-k+1}(U^0-U^{\star})\nonumber\\&\quad -\beta_n(U^0-U^{\star})+\Pi^n_{i=1}(1-\beta_i) \para{\gamma\cP^{\pi}}^{n+1}(U^0-U^{\star})\nonumber.
%     \end{align} 
%     where second equality is from rearrangement. 
    Taking the $\infn{\cdot}$-norm of both sides,
    we conclude
\begin{align*}
       \infn{T^{\pi}U^k-U^k} \le \frac{\para{\gamma^{-1}-\gamma}\para{1+2\gamma-\gamma^{k+1}}}{\para{\gamma^{k+1}}^{-1}-\gamma^{k+1}}\infn{U^0-U^{\pi}}.
\end{align*}
\end{proof}

\subsection{Accelerated rate for Bellman optimality operator}
We now present the accelerated convergence rate of \ref{eq:anc-vi} for the Bellman optimality operator.

Our analysis uses what we call the \emph{Bellman anti-optimality operator}, defined as
\begin{align*}
 \hat{T}^{\star}V(s)&=\inf_{a \in \cA} \left\{r(s,a)+\gamma\mathbb{E}_{s'\sim P(\cdot\,|\,s,a) }\left[V(s')\right]\right\}\\\hat{T}^{\star}Q(s,a)&=r(s,a)+\gamma\mathbb{E}_{s'\sim P(\cdot\,|\,s,a)}\left[\inf_{a' \in \cA} Q(s',a')\right],
\end{align*}
for all $s \in \cS$ and $ a \in \cA$. (The sup is replaced with a inf.) When $0<\gamma<1$, the Bellman anti-optimality operator is $\gamma$-contractive and has a unique fixed point $\hat{U}^{\star}$ by the exact same arguments that establish $\gamma$-contractiveness of the standard Bellman optimality operator.

% We now present our convergence rates of Anc-VI for Bellman optimality operator. 
\begin{theorem}\label{thm::Anc-VIC}
Let $0<\gamma<1$ be the discount factor.
Let $T^{\star}$ and $\hat{T}^{\star}$ respectively be the Bellman optimality and anti-optimality operators for $V$ or $Q$.
Let $U^\star$ and $\hat{U}^\star$ respectively be the fixed points of $T^{\star}$ and $\hat{T}^{\star}$.
Then, \ref{eq:anc-vi} exhibits the rate 
\begin{align*}
    \infn{T^{\star}U^k-U^k} &\le \frac{\para{\gamma^{-1}-\gamma}\para{1+2\gamma-\gamma^{k+1}}}{\para{\gamma^{k+1}}^{-1}-\gamma^{k+1}}\max{\left\{\infn{U^0-U^{\star}},\infn{U^0-\hat{U}^\star}\right\}}
    \end{align*}
% \begin{align*}
%     \infn{T^{\star}U^k-U^k} &\le \frac{\para{1+\gamma}\para{1+2\gamma-\gamma^{k+1}}}{1+\gamma^{k+1}}\frac{\gamma^{k}}{\sum_{i=0}^k \gamma^{i}}\max{\left\{\infn{U^0-U^{\star}},\infn{U^0-\hat{U}^\star}\right\}}
%     \end{align*}
for $k=0,1,\dots$.
If, furthermore, $U^0 \le T^{\star}U^0$ or $U^0\ge T^{\star}U^0$, then \ref{eq:anc-vi} exhibits the rate 
\begin{align*}
     \infn{T^{\star}U^k-U^k} &\le \frac{\para{\gamma^{-1}-\gamma}\para{1+\gamma-\gamma^{k+1}}}{\para{\gamma^{k+1}}^{-1}-\gamma^{k+1}}\infn{U^0-U^{\star}} \quad \,\, \text{if}\,\, U^0 \le T^{\star}U^0\\
     \infn{T^{\star}U^k-U^k} &\le \frac{\para{\gamma^{-1}-\gamma}\para{1+\gamma-\gamma^{k+1}}}{\para{\gamma^{k+1}}^{-1}-\gamma^{k+1}} \infn{U^0-\hat{U}^\star} \quad\,\, \text{if}\,\, U^0 \ge T^{\star}U^0
 \end{align*}
% \begin{align*}
%      \infn{T^{\star}U^k-U^k} &\le \frac{\para{1+\gamma}\para{1+\gamma-\gamma^{k+1}}}{1+\gamma^{k+1}}\frac{\gamma^{k}}{\sum_{i=0}^k \gamma^{i}} \infn{U^0-U^{\star}} \quad \,\, \text{if}\,\, U^0 \le T^{\star}U^0\\
%       \infn{T^{\star}U^k-U^k} &\le \frac{\para{1+\gamma}\para{1+\gamma-\gamma^{k+1}}}{1+\gamma^{k+1}}\frac{\gamma^{k}}{\sum_{i=0}^k \gamma^{i}} \infn{U^0-\hat{U}^\star} \quad\,\, \text{if}\,\, U^0 \ge T^{\star}U^0
%  \end{align*}   
 for $k=0,1,\dots$.
\end{theorem}
\ref{eq:anc-vi} with the Bellman optimality operator exhibits the same accelerated convergence rate as \ref{eq:anc-vi} with the Bellman consistency operator. As in Theorem~\ref{thm::Anc-VIE}, the rate of Theorem~\ref{thm::Anc-VIC} also becomes $\mathcal{O}(1/k)$ when $\gamma \approx 1$, while VI has a $\mathcal{O}(1)$-rate.

% Therefore, \ref{eq:anc-vi} for Bellman optimality operator has theoretically faster convergence rate than VI and exhibits 
 % We note that these results are independent of whether domain of operator is state action or state-action space or finite dimension or infinite dimension.

\begin{proof}[Proof outline of Theorem \ref{thm::Anc-VIC}]
The key technical challenge of the proof comes from the fact that the Bellman optimality operator is non-linear.
% \begin{lemma} There exist $(\pi_n,\dots, \pi_0)$ and $(\hat{\pi}_n,\dots, \hat{\pi}_0)$ such that
%     \begin{align*}
%    T^{\star}U^k-U^k
%     &\le \sum_{k=0}^n\beta_k\Pi^{n}_{i=k+1}(1-\beta_i) \left(\Pi_{j=k}^{n}(\gamma\cP^{\pi_j})(U^0-U^{\star})-\Pi^{n}_{j=k+1}(\gamma\cP^{\pi_j})(U^0-U^{\star})\right),\\
%    T^{\star}U^k-U^k
%     &\ge \sum_{k=0}^n\beta_k\Pi^{n}_{i=k+1}(1-\beta_i) \left(\Pi_{j=k}^{n}(\gamma\cP^{\hat{\pi}_j})(U^0-\hat{U}^{\star})-\Pi^{n}_{j=k+1}(\gamma\cP^{\hat{\pi}_j})(U^0-\hat{U}^{\star})\right). 
% \end{align*} 
% \end{lemma}
% Note that if $\pi_j$ or $\hat{\pi}_j$ are all same independent of $j$, we have same form of (\ref{Bellman-consitency}). Thus, with this Lemma, through similar argument of Bellman consistency operator case, we could obtain accelerated rate (deferring the full proof to Appendix). The most key part of proof is proof of Lemma.
Similar to the Bellman consistency operator case, we have
\begin{align*}
    T^{\star}U^k-U^k
    &=T^{\star}U^k-(1-\beta_k)T^{\star}U^{k-1}-\beta_kT^{\star}U^{\star}-\beta_k(U^0-U^{\star})\\
    &\stackrel{\blacktriangle}{\le} \gamma\cP^{\pi_k}\para{U^k-(1-\beta_k)U^{k-1}-\beta_k U^{\star}}-\beta_k(U^0-U^{\star})\\
    &= \gamma\cP^{\pi_k}(\beta_k \para{U^{0}-U^{\star}}+(1-\beta_k)(T^{\star}U^{k-1}-U^{k-1}))-\beta_k(U^0-U^{\star})\\
    & \le \sum_{i=1}^k \left[(\beta_i-\beta_{i-1}(1-\beta_{i}))\para{\Pi^k_{j=i+1}(1-\beta_j)}\para{\Pi^i_{l=k}\gamma\cP^{\pi_l}}(U^0-U^{\star})\right]
    \\&\quad -\beta_k(U^0-U^{\star})+\para{\Pi^k_{j=1}(1-\beta_j)} \para{\Pi^0_{l=k}\gamma\cP^{\pi_l}}(U^0-U^{\star}),
\end{align*}
where $\pi_k$ is the greedy policy satisfying $T^{\pi_k}U^k=T^{\star}U^k$, we define $\Pi^i_{l=k}\gamma\cP^{\pi_l}=\gamma\cP^{\pi_k}\gamma\cP^{\pi_{k-1} }\cdots\gamma\cP^{\pi_i}$, and last inequality follows by induction and monotonicity of Bellman optimality operator. The key step $\blacktriangle$ uses greedy policies $\{\pi_l\}_{l=0,1,\dots,k}$, which are well defined when the action space is finite. When the action space is infinite, greedy policies may not exist, so we use the Hahn--Banach extension theorem to overcome this technicality. The full argument is provided in Appendix~\ref{s::omitted-Anc-VI-proofs}.
% \begin{lemma}\label{lem::min_max_policy}
% If $0\le \alpha$, for V or Q, there exist policy $\hat{\pi}, \pi$ such that  
%     \[T^{\hat{\pi}}U-\alpha T^{\hat{\pi}}U' \le T^{\star}U-\alpha T^{\star}U' \le  T^{\pi}U-\alpha T^{\pi}U'.\]
% \end{lemma}
% \begin{lemma}\label{lem::min_max_Bellman} 
% For any policy $\pi$ and V or Q, we have
%     \begin{align*}
%       \hat{T}^{\star}U \le T^{\pi}U \le T^{\star}U.
%     \end{align*}
% \end{lemma}

To lower bound $T^{\star}U^k-U^k$, we use a similar line of reasoning with the Bellman anti-optimality operator. Combining the upper and lower bounds of $T^{\star}U^k-U^k$, we conclude the accelerated rate of Theorem \ref{thm::Anc-VIC}.
\end{proof}

For $\gamma<1$, the rates of Theorems~\ref{thm::Anc-VIE} and \ref{thm::Anc-VIC} can be translated to a bound on the distance to solution:
\begin{align*}
    \infn{U^k-U^{\star}}  \le \gamma^k \frac{\para{1+\gamma}\para{1+2\gamma-\gamma^{k+1}}}{\para{1-\gamma^{2k+2}}} \infn{U^0-U^{\star}}
\end{align*}
for $k=1,2,\dots$. This $O(\gamma^k)$ rate is worse than the rate of (classical) VI by a constant factor. Therefore, \ref{eq:anc-vi} is better than VI in terms of the Bellman error, but it is not better than VI in terms of distance to solution.

% As Theorem \ref{thm::Anc-VIE}, \ref{thm::Anc-VIC} shows, Anc-VI indeed accelerates VI for both Bellman consistency and optimality operators. As a final remark, the analysis of VI for Bellman optimality operator is also valid for Bellman anti-optimality operator in symmetric manner and could be obtained same accelerated convergence rate.

%\paragraph{Continuous state and action spaces.} To clarify, Theorem~\ref{thm::Anc-VIE} and \ref{thm::Anc-VIC} do not require the state and action spaces to be finite, so long as $T^\pi$ and $T^\star$ are well defined. Even if $\pi_k, \hat{\pi_k}$ in the core inequality of proof outline may not exist  in continuous action spaces, we can establish the inequality which lead us to same accelerated convergence rate as we showed in proof of Appendix.

%As an aside, the assumption that the state and action spaces are finite is used only to ensure $T^\pi$ and $T^\star$ are well defined. The proofs of Theorems~\ref{thm::Anc-VIE} and \ref{thm::Anc-VIC} do not otherwise rely on any finiteness condition, and Theorems~\ref{thm::Anc-VIE} and \ref{thm::Anc-VIC} apply to the setup of continuous state and action spaces when $T^\pi$ and $T^\star$ are well defined. Even though in continuous state and action spaces, $\pi, \hat{\pi}$ in proof outline may not exist. But, we can solve this problem in Appendix. 

\section{Convergence when $\gamma=1$}\label{sec::nonexp}
Undiscounted MDPs are not commonly studied in the DP and RL theory literature due to the following difficulties: Bellman consistency and optimality operators may not have fixed points, VI is a nonexpansive (not contractive) fixed-point iteration and may not convergence to a fixed point even if one exist, and the interpretation of a fixed point as the (optimal) value function becomes unclear when the fixed point is not unique. However, many modern deep RL setups actually do not use discounting,%
\footnote{
As a specific example, the classical policy gradient theorem  \cite{sutton1999}
calls for the use of
$\nabla J(\theta)=\mathbb{E}
\left[
\sum^\infty_{t=0}\gamma^t
\nabla_\theta\log \pi_\theta(a_t\,|\,s_t)Q^\phi_\gamma(s_t,a_t)
\right]$, but many modern deep policy gradient methods use $\gamma=1$ in the first instance of $\gamma$ (so $\gamma^t=1$) while using $\gamma<1$ in $Q^\phi_\gamma(s_t,a_t)$ \cite{nota2020}.}
and this empirical practice makes the theoretical analysis with $\gamma=1$ relevant.

In this section, we show that \ref{eq:anc-vi} converges to fixed points of the Bellman consistency and optimality operators of undiscounted MDPs. While a full treatment of undiscounted MDPs is beyond the scope of this paper, we show that fixed points, if one exists, can be found, and we therefore argue that the inability to find fixed points should not be considered an obstacle in studying the $\gamma=1$ setup.

 % In this section, we consider undiscounted MDP with assumption on fixed points of Bellman consistency and optimality operators. First, we present convergence result of \hyperlink{Anc-VI}{Anc-VI} for $\gamma=1$, well-defined by its definition, in finite dimensional domain.
 We first state our convergence result for finite state-action spaces.
\begin{theorem}\label{thm::nonexp_finite}
   Let $\gamma =1$.
    Let $T \colon \real^n \rightarrow \real^n$ be the nonexpansive Bellman consistency or optimality operator for $V$ or $Q$. Assume a fixed point exists (not necessarily unique). 
    % Then, \ref{eq:anc-vi} exhibits the rate
 %    \begin{align*}
 %     \infn{TU^k-U^k} \le \frac{2}{k+1}\infn{U^0-U^{\star}}\qquad\text{ for }k=0,1,\dots.
 % \end{align*}
  If,  $U^0 \le TU^0$, then \ref{eq:anc-vi} exhibits the rate
    \begin{align*}
     \infn{TU^k-U^k} \le \frac{1}{k+1}\infn{U^0-U^{\star}}\qquad\text{ for }k=0,1,\dots.
 \end{align*}
  for any fixed point $U^{\star}$ satisfying $ U^0 \le U^{\star}$. Furthermore, $U^k \rightarrow U^{\infty}$ for some fixed point $U^{\infty}$.
\end{theorem}
% \begin{proof}
%     Convergence rate comes from Theorem. If $\{x_{n_i}\}$'s are convergence subsequence, $x_{n_i} \rightarrow U^{\star}$ since $T^{\pi}-I$ is continuous and this implies that $\para{T^{\pi}-I}x_{n_i} \rightarrow 0$. By induction, $\|U^k-U^0\| \le \|U^{\star}-U^0\|.$ Since $U^k$ is bounded, by Bolzano Weierstrass theorem, ${U^k} $ has at least one accumulation points. Thus $U^k \rightarrow U^{\star}$.
% \end{proof}
% need following lemma. Recall that $\hat{T}^{\star}$ is Bellman anti-optimality operators. 
% \begin{lemma}
% If $Fix(T^{\star}) \neq \emptyset$, $Fix(\hat{T}^{\star}) \neq \emptyset$. 
% \end{lemma}
% \begin{proof}
% If $U^{\star} \in Fix(T)$, $U^{\star} \ge \hat{T}^{\star}U^{\star}$ by definition. Then, by induction, if $U^0=U^{\star}$, $\para{\hat{T}^{\star}}^kU^{\star}$ is decreasing sequence by monotonicity and converge to some $\hat{x}_{\star}$. By previous theorem, $\hat{x}_{\star}$ is fixed point of $\hat{T}^{\star}$.  
% \end{proof}
% Then we have following Theorem. 
% \begin{theorem}
%     Let $T^{\star}: \real^n \rightarrow \real^n$ be non-expensive Bellman optimality operators and$Fix(T^{\star})\neq \emptyset$. If $T^{\star}U^0 \ge U^0$, Anc-VI exhibits the rate
%     \begin{align*}
%     \infn{T^{\star}U^k-U^k} &\le \frac{1}{n+1}\infn{U^0-U^{\star}} \quad \text{if}\,\, U^0 \le T^{\star}U^0
%  \end{align*}
%   and $U^k \rightarrow U^{\star}$ for some $U^{\star} \in Fix(T^{\pi})$. 
% \end{theorem}

If rewards are nonnegative, then the condition $U^0 \le TU^0$ is satisfied with $U^0=0$. So, under this mild condition, \ref{eq:anc-vi} with $\gamma=1$ converges with $\mathcal{O}(1/k)$-rate on the Bellman error. To clarify, the convergence $U^k\rightarrow U^\infty$ has no rate, i.e., $\|U^k-U^\infty\|_\infty=o(1)$, while $\infn{TU^k-U^k} =\mathcal{O}(1/k)$. In contrast, standard VI does not guarantee convergence in this setup.

We also point out that the convergence of Bellman error does not immediately imply point convergence, i.e., $TU^k-U^k\rightarrow 0$ does not immediately imply $U^k \rightarrow U^{\star}$, when $\gamma=1$. Rather, we show (i) $U^k$ is a bounded sequence, (ii) any convergent subsequence $U^{k_j}$ converges to a fixed point $U^{\infty}$, and (iii) $U^k$ is elementwise monotonically nondecreasing and therefore has a single limit.

% Under mild condtions, Anc-VI guarantee convergence with $\mathcal{O}(1/k)$ rate even though VI does not guarantee convergence under same condition. 

%We clarify that our analysis depends on finite-dimensionality only through the well-definedness of $T$ and through the compactness of the $\|\cdot\|_\infty$-ball. 

Next, we state our convergence result for general state-action spaces.
\begin{theorem}\label{thm::nonexp_infinite}
         Let $\gamma=1$. Let the state and action spaces be general (possibly infinite) sets.
         Let $T$ be the nonexpansive Bellman consistency or optimality operator for $V$ or $Q$, and assume $T$ is well defined.%
         \footnote{Well-definedness of $T$ requires a $\sigma$-algebra on state and action spaces, expectation with respect to transition probability and policy to be well defined, boundedness and measurability of the output of Bellman operators, etc.}
         Assume a fixed point exists (not necessarily unique). If $U^0 \le TU^0$, then \ref{eq:anc-vi} exhibits the rate
    \begin{align*}
     \infn{TU^k-U^k} \le \frac{1}{k+1}\infn{U^0-U^{\star}} \qquad\text{ for }k=0,1,\dots
 \end{align*}
   for any fixed point $U^{\star}$ satisfying $ U^0 \le U^{\star}$. Furthermore, $U^k\rightarrow U^{\infty}$ pointwise monotonically for some fixed point $U^{\infty}$.
 %If $\{U^k\}_{k=0,1,\dots}$ has a convergent subsequence,  then $U^k \rightarrow U^{\star}$ such that $U^{\star}=TU^{\star}$.
\end{theorem}
% To clarify, $\|\cdot\|_\infty$ denotes the uniform-norm, boundedness of $\{U^k\}_{k=0,1,\dots}$ means uniform boundedness, and $U^k \rightarrow U^{\star}$ denotes uniform convergence. The final convergence result can be equivalently phrased as: If $\{U^k\}_{k=0,1,\dots}$ has a convergent subsequence, then, in fact, the full sequence converges and the limit is a fixed point of $T$. Note that if $\cS$ and $\cA$ are finite sets equipped with the discrete topology, then $\{U_k\}$ is an equicontinuous family, and we can use Arzel\`a--Ascoli to show the existence of a convergent subsequence. In this sense, Theorem~\ref{thm::nonexp_infinite} is generalization of Theorem~\ref{thm::nonexp_finite}.
The convergence $U^k\rightarrow U^{\infty}$ pointwise in infinite state-action spaces is, in our view, a non-trivial contribution.
When the state-action space is finite, pointwise convergence directly implies convergence in $\|\cdot\|_\infty$, and in this sense, Theorem~\ref{thm::nonexp_infinite} is generalization of Theorem~\ref{thm::nonexp_finite}.
However, when the state-action space is infinite, pointwise convergence does not necessarily imply uniform convergence, i.e., $U^k\rightarrow U^\infty$ pointwise does not necessarily imply $U^k\rightarrow U^\infty$ in $\|\cdot\|_\infty$.

% Similarly, for Bellman optimality operators, we have following theorem. 
% \begin{theorem}
%     Let $T^{\star}$ be non-expensive Bellman optimality operators. If $Fix(T^{\star}) \neq \emptyset$ and $U^0 \le T^{\star} U^0$ or $U^0 \ge T^{\star} U^0$, Anc-VI exhibits the rate
%     \begin{align*}
%      \infn{T^{\star}U^k-U^k} &\le \frac{1}{n+1}\infn{U^0-U^{\star}}, \quad \text{if}\,\, U^0 \le T^{\star} U^0\\
%      \infn{T^{\star}U^k-U^k} &\le \frac{1}{n+1}\infn{U^0-\hat{U}^\star}, \quad \text{if}\,\, U^0 \ge T^{\star} U^0
%  \end{align*}
%   and $U^k \rightarrow U^{\star}$ for some $U^{\star} \in Fix(T)$. 
% \end{theorem}

\section{Complexity lower bound
}\label{sec::complexity}
We now present a complexity lower bound establishing optimality of \ref{eq:anc-vi}.

% \subsection{linear operator}

% \begin{theorem}[linear operator]
%     For $n \ge N +1$ and any initial point $U^0 \in \real^n$,
% there exists an $\gamma$-contractive operator $T\colon \real^n \rightarrow \real^n$ with a
% fixed point $x^{\star} \in Fix T$ such that
% \[\infn{Tx_N-x_N} \ge \frac{1+\gamma^{N+1}}{\sum_{k=0}^N\gamma^k}\gamma^N\infn{U^0-U^{\star}}\]
% for any iterates $\{x_k\}^{N}_{k=0}$ satisfying 
% \[x_k \in U^0+span\{U^0-TU^0, U^1-TU^1, \dots, x_{k-1}-Tx_{k-1} \} \]
% for $k=1,\dots,N.$
% \end{theorem}

\begin{theorem}\label{thm::complexity_contractive}
Let $k\ge 0$, $n \ge k+2$, $0<\gamma\le 1$, and $U^0 \in \real^n$. Then there exists an MDP with $|\mathcal{S}|=n$ and $|\mathcal{A}|=1$ (which implies the Bellman consistency and optimality operator for V and Q all coincide as $T\colon \real^n \rightarrow \real^n$) such that $T$ has a fixed point $U^\star$ satisfying $U^0 \le U^{\star}$ and
\[\infn{TU^k-U^k} \ge \frac{\gamma^k}{\sum_{i=0}^k\gamma^i}\infn{U^0-U^{\star}}\]
for any iterates $\{U^i\}^{k}_{i=0}$ satisfying 
\[U^i \in U^0+\mathrm{span}\{TU^0-U^0, TU^1-U^1, \dots, TU^{i-1}-U^{i-1} \}
\qquad\text{ for }i=1,\dots,k.
\]
\end{theorem}
\begin{proof}[Proof outline of Theorem \ref{thm::complexity_contractive}]
Without loss of generality, assume $n=k+2$ and $U^0=0$. Consider the MDP $(\cS, \cA, P, r, \gamma)$ such that  
\begin{align*}
        \cS=\{s_1,\dots, s_{k+2}\}, \quad \cA=\{a_1\}, \quad P(s_i\,|\,s_j,a_1)=\mathds{1}_{\{i=j=1, \,j=i+1\}}, \quad r(s_i, a_1)=\mathds{1}_{\{i=2\}}.
\end{align*}
Then, $T= \gamma\cP^{\pi}U+[0,1,0, \dots, 0]^{\intercal}, U^{\star}=[0,1,\gamma,\dots, \gamma^k]^{\intercal}$, and $\infn{U^0-U^{\star}}=1$. Under the span condition, we can show that $\para{U^k}_{1}=\para{U^k}_{k+2}=0$. Then, we get
\begin{align*}
    TU^k-U^k =\left(0, 1-\para{U^k}_2,\gamma\para{U^k}_2 -\para{U^k}_3, \dots, \gamma\para{U^k}_{k}-\para{U^k}_{k+1}, \gamma\para{U^k}_{k+1}\right)
\end{align*}
and this implies 
\[\para{TU^k-U^k}_1+ \para{TU^k-U^k}_2+\gamma^{-1}\para{TU^k-U^k}_3+\cdots+ \gamma^{-k}\para{TU^k-U^k}_{k+2}=1.\] 
Taking the absolute value on both sides,
\begin{align*}
     (1+\cdots+\gamma^{-k})\max_{1\le i\le k+2} {\{|TU^k-U^k|_i\}} \ge 1.
\end{align*}
Therefore, we conclude
\begin{align*}\|TU^{k}-U^{k}\|_{\infty} \ge \frac{\gamma^k}{\sum_{i=0}^k \gamma^{i}} \infn{U^0-U^{\star}}.\\[-0.4in]
\end{align*}
\end{proof}
% \vspace{-0.2in}

Note that the case $\gamma=1$ is included in Theorem~\ref{thm::complexity_contractive}. When $\gamma=1$, the lower bound of Theorem~\ref{thm::complexity_contractive} \emph{exactly} matches the upper bound of Theorem~\ref{thm::nonexp_finite}.

Since 
\[
\frac{\gamma^k}{\sum_{i=0}^k\gamma^i}\le
\frac{\para{\gamma^{-1}-\gamma}\para{1+\gamma-\gamma^{k+1}}}{\para{\gamma^{k+1}}^{-1}-\gamma^{k+1}} 
\le \frac{4\gamma^k}{\sum_{i=0}^k\gamma^i}\qquad\text{ for all }0<\gamma< 1,
\]
the lower bound establishes optimality of the second rates Theorems~\ref{thm::Anc-VIE} and \ref{thm::Anc-VIC} up to a constant of factor $4$. Theorem~\ref{thm::complexity_contractive} improves upon the prior state-of-the-art complexity lower bound established in the proof of \cite[Theorem 3]{goyal2022first} by a factor $1-\gamma^{k+1}$.
(In \cite[Theorem 3]{goyal2022first}, a lower bound on the distance to optimal value function is provided. Their result has an implicit dependence on the initial distance to optimal value function $\|U^0-U^{\star}\|_\infty$, so we make the dependence explicit, and we translate their result to a lower bound on the Bellman error. Once this is done, the difference between our lower bound of Theorem~\ref{thm::complexity_contractive} and of \cite[Theorem 3]{goyal2022first} is a factor of $1-\gamma^{k+1}$. The worst-case MDP of \cite[Theorem 3]{goyal2022first} and our worst-case MDP primarily differ in the rewards, while the states and the transition probabilities are almost the same.)

The so-called ``span condition'' of Theorem~\ref{thm::complexity_contractive} is arguably very natural and  is satisfied by standard VI and \ref{eq:anc-vi}. The span condition is commonly used in the construction of complexity lower bounds on first-order optimization methods \cite{nesterov2003Introductory,drori2017exact,drori2022oracle,carmon2020stationary1, carmon2021stationary2,park2022exact} and has been used in the prior state-of-the-art lower bound for standard VI \cite[Theorem 3]{goyal2022first}. However, designing an algorithm that breaks the lower bound of Theorem~\ref{thm::complexity_contractive} by violating the span condition remains a possibility.  In optimization theory, there is precedence of lower bounds being broken by violating seemingly natural and minute conditions \cite{hannah2018,golowich2020last,yoon2021accelerated}.

\section{Approximate Anchored Value Iteration}\label{sec::Apx-Anc-VI}

In this section, we show that the anchoring mechanism is robust against evaluation errors of the Bellman operator, just as much as the standard approximate VI.

Let $0<\gamma<1$ and let $T^{\star}$ be the Bellman optimality operator. The \emph{Approximate Anchored Value Iteration} \eqref{eq:Apx-Anc-VI} is
\begin{align}
   \begin{aligned}
U_{\epsilon}^{k}&=T^{\star}U^{k-1}+\epsilon^{k-1}\\
   U^k&=\beta_kU^0+ (1-\beta_k)U_{\epsilon}^k 
   \end{aligned}
   \tag{Apx-Anc-VI}
   \label{eq:Apx-Anc-VI}
\end{align} 
for $k=1,2,\dots,$ where $\beta_k=1/(\sum_{i=0}^k \gamma^{-2i})$, $U^0$ is an initial point, and the $\{\epsilon^k\}_{k=0}^\infty$ is the error sequence modeling approximate evaluations of $T^\star$.

Of course, the classical Approximate Value Iteration \eqref{eq:Apx-VI} is
\begin{align*}
U^{k}&=T^{\star}U^{k-1}+\epsilon^{k-1}
\tag{Apx-VI}
\label{eq:Apx-VI}
\end{align*} 
for $k=1,2,\dots,$ where $U^0$ is an initial point.

% Apx-Anc-VI approximates Bellman optimality operator for every iteration. In general, approximation error comes from the situation where we can only partially access Bellman optimality operator or function space $U^k_{\epsilon} \in \cF$ we have is not representative enough. 

     \begin{fact}[Classical result, \protect{\cite[p.333]{bertsekas1995neuro}}]
 Let $0<\gamma<1$ be the discount factor.
Let $T^{\star}$ be the Bellman optimality for $V$ or $Q$.
Let $U^\star$ be the fixed point of $T^{\star}$.
Then \ref{eq:Apx-VI} exhibits the rate
\begin{align*}
    \infn{T^{\star}U^k-U^k} &\le (1+\gamma)\gamma^k\infn{U^0-U^{\star}}+ \para{1+\gamma}\frac{1-\gamma^{k}}{1-\gamma}\max_{0\le i\le k-1 } \infn{\epsilon^i}
     \,\,\,\text{ for }k=1,2,\dots.
\end{align*} 
 \end{fact}
\begin{theorem}\label{thm::Apx-Anc-VI}
 Let $0<\gamma<1$ be the discount factor.
 Let $T^{\star}$ and $\hat{T}^{\star}$ respectively be the Bellman optimality and anti-optimality operators for $V$ or $Q$.
Let $U^\star$ and $\hat{U}^\star$ respectively be the fixed points of $T^{\star}$ and $\hat{T}^{\star}$.
Then \ref{eq:Apx-Anc-VI} exhibits the rate
\begin{align*}
  \infn{T^{\star}U^k-U^k} &\le \frac{\para{\gamma^{-1}-\gamma}\para{1+2\gamma-\gamma^{k+1}}}{\para{\gamma^{k+1}}^{-1}-\gamma^{k+1}} \max{\left\{\infn{U^0-U^{\star}},\infn{U^0-\hat{U}^\star}\right\}}\\&\quad+ \frac{1+\gamma}{1+\gamma^{k+1}}\frac{1-\gamma^{k}}{1-\gamma}\max_{0\le i\le k-1 } \infn{\epsilon^i}     \qquad\text{ for }k=1,2,\dots.
\end{align*}
If, furthermore, $U^0 \ge T^{\star}U^0$, then \eqref{eq:Apx-Anc-VI} exhibits the rate \begin{align*}
      \infn{T^{\star}U^k-U^k} &\le \frac{\para{\gamma^{-1}-\gamma}\para{1+\gamma-\gamma^{k+1}}}{\para{\gamma^{k+1}}^{-1}-\gamma^{k+1}}  \infn{U^0-\hat{U}^{\star}}+\frac{1+\gamma}{1+\gamma^{k+1}}\frac{1-\gamma^{k}}{1-\gamma}\max_{0\le i\le k-1 } \infn{\epsilon^i}
\end{align*}
for $k=1,2,\dots$.
% Also, Anc-VI exhibits the rate
% \begin{align*}
%      \infn{T^{\star}U^k-U^k} &\le \frac{\para{1+\gamma}^2}{\sum_{i=0}^k \gamma^{i}}\gamma^k \infn{U^0-U^{\star}}, \quad \,\, \text{if}\,\, U^0 \le T^{\star}U^0\\
%       \infn{T^{\star}U^k-U^k} &\le \frac{\para{1+\gamma}^2}{\sum_{i=0}^k \gamma^{i}}\gamma^k \infn{U^0-\hat{U}^\star}, \quad \text{if}\,\, U^0 \ge T^{\star}U^0.
%  \end{align*}   
\end{theorem}

% As one would expect, the rates are adversaly affected by $\max \infn{\epsilon_i}$. 
The dependence on $\max \infn{\epsilon_i}$ of \ref{eq:Apx-Anc-VI} is no worse than that of \ref{eq:Apx-VI}. In this sense, \ref{eq:Apx-Anc-VI} is robust against evaluation errors of the Bellman operator, just as much as the standard \ref{eq:Apx-VI}. Finally, we note that a similar analysis can be done for \ref{eq:Apx-Anc-VI} with the Bellman consistency operator.

\section{Gauss--Seidel Anchored Value Iteration}\label{sec::GS-Anc-VI}
In this section, we show that the anchoring mechanism can be combined with Gauss--Seidel-type updates in finite state-action spaces.
Let $0<\gamma<1$ and let $T^{\star}\colon \reals^n \rightarrow \reals^n$ be the Bellman optimality operator. Define $T_{GS}^{\star}\colon \reals^n\rightarrow \reals^n$ as
\[T_{GS}^{\star}=T^{\star}_{n}\cdots T^{\star}_{2}T^{\star}_{1},\]
where $T_j^{\star}: \reals^n \rightarrow \reals^n$ is defined as
\begin{align*}
T^{\star}_j(U)=(U_1,\dots,U_{j-1},\para{T^{\star}(U)}_j,U_{j+1},\dots,U_n)
    \end{align*}
for $j=1,\dots,n$.
\begin{fact}\label{lem::GS}[Classical result, \protect{\cite[Theorem~6.3.4]{10.5555/528623}}]
$T_{GS}^{\star}$ is a $\gamma$-contractive operator and has the same fixed point as $T^{\star}$.
\end{fact}
The \emph{Gauss--Seidel Anchored Value Iteration} \eqref{eq:GS-Anc-VI} is
\vspace{0.011in}
\begin{align}
    U^{k}&=\beta_kU^{0}+(1-\beta_k)T_{GS}^{\star}U^{k-1}
    \tag{GS-Anc-VI}
    \label{eq:GS-Anc-VI}
\end{align}
for $k=1,2,\dots,$ where $\beta_k=1/(\sum_{i=0}^k \gamma^{-2i})$ and $U^0$ is an initial point.
\begin{theorem}\label{thm::GS-Anc-VI}
 Let the state and action spaces be finite sets. Let $0<\gamma<1$ be the discount factor.
 Let $T^{\star}$ and $\hat{T}^{\star}$ respectively be the Bellman optimality and anti-optimality operators for $V$ or $Q$.
Let $U^\star$ and $\hat{U}^\star$ respectively be the fixed points of $T^{\star}$ and $\hat{T}^{\star}$.
Then \ref{eq:GS-Anc-VI} exhibits the rate
\vspace{0.05in}
\begin{align*}
    \infn{T_{GS}^{\star}U^k-U^k} &\le \frac{\para{\gamma^{-1}-\gamma}\para{1+2\gamma-\gamma^{k+1}}}{\para{\gamma^{k+1}}^{-1}-\gamma^{k+1}} \max{\left\{\infn{U^0-U^{\star}},\infn{U^0-\hat{U}^\star}\right\}}
\end{align*}
\vspace{0.05in}
for $k=0,1,\dots$. If, furthermore, $U^0 \le T_{GS}^{\star}U^0$ or $U^0 \ge T_{GS}^{\star}U^0$, then \ref{eq:GS-Anc-VI} exhibits the rate
\vspace{0.02in}
\begin{align*}
     \infn{T_{GS}^{\star}U^k-U^k} &\le \frac{\para{\gamma^{-1}-\gamma}\para{1+\gamma-\gamma^{k+1}}}{\para{\gamma^{k+1}}^{-1}-\gamma^{k+1}} \infn{U^0-U^{\star}}\quad\,\, \text{if}\,\, U^0 \le T_{GS}^{\star}U^0\\
      \infn{T_{GS}^{\star}U^k-U^k} &\le \frac{\para{\gamma^{-1}-\gamma}\para{1+\gamma-\gamma^{k+1}}}{\para{\gamma^{k+1}}^{-1}-\gamma^{k+1}} \infn{U^0-\hat{U}^{\star}} \quad\,\, \text{if}\,\, U^0 \ge T_{GS}^{\star}U^0
 \end{align*}   
 for $k=0,1,\dots$.
\end{theorem}

We point out that \ref{eq:GS-Anc-VI} cannot be directly extended to infinite action spaces since Hahn--Banach extension theorem is not applicable in the Gauss--Seidel setup. 
Furthermore, we note that a similar analysis can be carried out for \ref{eq:GS-Anc-VI} with the Bellman consistency operator.

% But under the assumption of existence of greedy policies, \ref{eq:GS-Anc-VI} can be extended to a \emph{block} Gauss-Seidel version in infinite state-action space, if the state and action spaces are decomposed into $n<\infty$ disjoint blocks.
%, with same argument in proof of Theorem $\ref{thm::GS-Anc-VI}$. 

%First, let  partition the domain of Anc-VI into $m$ blocks, $B_1,\dots, B_m$ and $\sigma: \{1,\dots,m\} \rightarrow \{1,\dots,m\}$ be any fixed permutation. 

\section{Conclusion}
We show that the classical value iteration (VI) is, in fact, suboptimal and that the anchoring mechanism accelerates VI to be optimal in the sense that the accelerated rate matches a complexity lower bound up to a constant factor of $4$. We also show that the accelerated iteration provably converges to a fixed point even when $\gamma=1$, if a fixed point exists. Being able to provide a substantive improvement upon the classical VI is, in our view, a surprising contribution.

One direction of future work is to study the empirical effectiveness of Anc-VI. Another direction is to analyze Anc-VI in a model-free setting and, more broadly, to investigate the effectiveness of the anchor mechanism in more practical RL methods.

Our results lead us to believe that many of the classical foundations of dynamic programming and reinforcement learning may be improved with a careful examination based on an optimization complexity theory perspective. The theory of optimal optimization algorithms has recently enjoyed significant developments \cite{kim2016optimized,kim2021accelerated,kim2021optimizing,yoon2021accelerated,park2023}, the anchoring mechanism being one such example \cite{Lieder2021halpern,park2022exact}, and the classical DP and RL theory may benefit from a similar line of investigation on iteration complexity.

% %%%%%%%%%%%%%%%%%%%%%%%%%%%%%%%%%
% %%% Acknowledgement
% %%%%%%%%%%%%%%%%%%%%%%%%%%%%%%%%%
\begin{ack}
This work was supported by the the Information \& communications Technology Planning \& Evaluation (IITP) grant funded by the Korea government(MSIT) [NO.2021-0-01343, Artificial Intelligence Graduate School Program (Seoul National University)] and the Samsung Science and Technology Foundation (Project Number SSTF-BA2101-02). We thank Jisun Park for providing valuable feedback.
\end{ack}

\bibliographystyle{abbrv}
\bibliography{AnchoredVI}

\begin{thebibliography}{100}

\bibitem{doi:10.1137/20M1367192}
M.~Akian, S.~Gaubert, Z.~Qu, and O.~Saadi.
\newblock Multiply accelerated value iteration for non-symmetric affine fixed
  point problems and application to {m}arkov decision processes.
\newblock {\em SIAM Journal on Matrix Analysis and Applications},
  43(1):199--232, 2022.

\bibitem{anderson1965}
D.~G. Anderson.
\newblock Iterative procedures for nonlinear integral equations.
\newblock {\em Journal of the Association for Computing Machinery},
  12(4):547--560, 1965.

\bibitem{andre1997generalized}
D.~Andre, N.~Friedman, and R.~Parr.
\newblock Generalized prioritized sweeping.
\newblock {\em Neural Information Processing Systems}, 1997.

\bibitem{azar2011speedy}
M.~G. Azar, R.~Munos, M.~Ghavamzadeh, and H.~Kappen.
\newblock Speedy {Q}-learning.
\newblock {\em Neural Information Processing Systems}, 2011.

\bibitem{banach1922operations}
S.~Banach.
\newblock Sur les op{\'e}rations dans les ensembles abstraits et leur
  application aux {\'e}quations int{\'e}grales.
\newblock {\em Fundamenta Mathematicae}, 3(1):133--181, 1922.

\bibitem{barre2020convergence}
M.~Barr{\'e}, A.~Taylor, and A.~d'Aspremont.
\newblock Convergence of a constrained vector extrapolation scheme.
\newblock {\em SIAM Journal on Mathematics of Data Science}, 4(3):979--1002,
  2022.

\bibitem{bellemare2016increasing}
M.~G. Bellemare, G.~Ostrovski, A.~Guez, P.~Thomas, and R.~Munos.
\newblock Increasing the action gap: New operators for reinforcement learning.
\newblock {\em Association for the Advancement of Artificial Intelligence},
  2016.

\bibitem{bellman1957markovian}
R.~Bellman.
\newblock A {M}arkovian decision process.
\newblock {\em Journal of Mathematics and Mechanics}, 6(5):679--684, 1957.

\bibitem{bertsekas2015parallel}
D.~Bertsekas and J.~Tsitsiklis.
\newblock {\em Parallel and Distributed Computation: Numerical Methods}.
\newblock Athena Scientific, 2015.

\bibitem{bertsekas2015dynamic}
D.~P. Bertsekas.
\newblock {\em Dynamic Programming and Optimal Control, volume II}.
\newblock 4th edition, 2012.

\bibitem{bertsekas1995neuro}
D.~P. Bertsekas and J.~N. Tsitsiklis.
\newblock {\em Neuro-Dynamic Programming}.
\newblock Athena Scientific, 1995.

\bibitem{pmlr-v161-bowen21a}
W.~Bowen, X.~Huaqing, Z.~Lin, L.~Yingbin, and Z.~Wei.
\newblock Finite-time theory for momentum {Q}-learning.
\newblock {\em Conference on Uncertainty in Artificial Intelligence}, 2021.

\bibitem{carmon2020stationary1}
Y.~Carmon, J.~C. Duchi, O.~Hinder, and A.~Sidford.
\newblock Lower bounds for finding stationary points {I}.
\newblock {\em Mathematical Programming}, 184(1--2):71--120, 2020.

\bibitem{carmon2021stationary2}
Y.~Carmon, J.~C. Duchi, O.~Hinder, and A.~Sidford.
\newblock Lower bounds for finding stationary points {II}: first-order methods.
\newblock {\em Mathematical Programming}, 185(1--2):315--355, 2021.

\bibitem{cauchy1847methode}
A.-L. Cauchy.
\newblock M{\'e}thode g{\'e}n{\'e}rale pour la r{\'e}solution des systemes
  d'{\'e}quations simultan{\'e}es.
\newblock {\em Comptes rendus de l'Acad{\'e}mie des Sciences}, 25:536--538,
  1847.

\bibitem{colao2021rate}
V.~Colao and G.~Marino.
\newblock On the rate of convergence of {H}alpern iterations.
\newblock {\em Journal of Nonlinear and Convex Analysis}, 22(12):2639--2646,
  2021.

\bibitem{contreras2022optimal}
J.~P. Contreras and R.~Cominetti.
\newblock Optimal error bounds for non-expansive fixed-point iterations in
  normed spaces.
\newblock {\em Mathematical Programming}, 199(1--2):343--374, 2022.

\bibitem{dai2011topological}
P.~Dai, D.~S. Weld, J.~Goldsmith, et~al.
\newblock Topological value iteration algorithms.
\newblock {\em Journal of Artificial Intelligence Research}, 42:181--209, 2011.

\bibitem{de2000existence}
D.~P. De~Farias and B.~Van~Roy.
\newblock On the existence of fixed points for approximate value iteration and
  temporal-difference learning.
\newblock {\em Journal of Optimization theory and Applications}, 105:589--608,
  2000.

\bibitem{diakonikolas2020halpern}
J.~Diakonikolas.
\newblock Halpern iteration for near-optimal and parameter-free monotone
  inclusion and strong solutions to variational inequalities.
\newblock {\em Conference on Learning Theory}, 2020.

\bibitem{doi:10.1137/21M1395302}
J.~Diakonikolas and P.~Wang.
\newblock Potential function-based framework for minimizing gradients in convex
  and min-max optimization.
\newblock {\em SIAM Journal on Optimization}, 32(3):1668--1697, 2022.

\bibitem{dong2018modified}
Q.~Dong, H.~Yuan, Y.~Cho, and T.~M. Rassias.
\newblock Modified inertial {M}ann algorithm and inertial {CQ}-algorithm for
  nonexpansive mappings.
\newblock {\em Optimization Letters}, 12(1):87--102, 2018.

\bibitem{drori2017exact}
Y.~Drori.
\newblock The exact information-based complexity of smooth convex minimization.
\newblock {\em Journal of Complexity}, 39:1--16, 2017.

\bibitem{drori2020stoc-complexity}
Y.~Drori and O.~Shamir.
\newblock The complexity of finding stationary points with stochastic gradient
  descent.
\newblock {\em International Conference on Machine Learning}, 2020.

\bibitem{drori2022oracle}
Y.~Drori and A.~Taylor.
\newblock On the oracle complexity of smooth strongly convex minimization.
\newblock {\em Journal of Complexity}, 68, 2022.

\bibitem{drori2016optimal-kelley}
Y.~Drori and M.~Teboulle.
\newblock An optimal variant of {K}elley's cutting-plane method.
\newblock {\em Mathematical Programming}, 160(1--2):321--351, 2016.

\bibitem{ermis2021anderson1}
M.~Ermis, M.~Park, and I.~Yang.
\newblock On {A}nderson acceleration for partially observable {M}arkov decision
  processes.
\newblock {\em IEEE Conference on Decision and Control}, 2021.

\bibitem{ermis2020a3dqn}
M.~Ermis and I.~Yang.
\newblock {A3DQN}: Adaptive {A}nderson acceleration for deep {Q}-networks.
\newblock {\em IEEE Symposium Series on Computational Intelligence}, 2020.

\bibitem{ernst2005tree}
D.~Ernst, P.~Geurts, and L.~Wehenkel.
\newblock Tree-based batch mode reinforcement learning.
\newblock {\em Journal of Machine Learning Research}, 6:503--556, 2005.

\bibitem{ernst2005approximate}
D.~Ernst, M.~Glavic, P.~Geurts, and L.~Wehenkel.
\newblock Approximate value iteration in the reinforcement learning context.
  {A}pplication to electrical power system control.
\newblock {\em International Journal of Emerging Electric Power Systems}, 3(1),
  2005.

\bibitem{farahmand2021pid}
A.-m. Farahmand and M.~Ghavamzadeh.
\newblock {PID} accelerated value iteration algorithm.
\newblock {\em International Conference on Machine Learning}, 2021.

\bibitem{farahmand2010error}
A.-m. Farahmand, C.~Szepesv{\'a}ri, and R.~Munos.
\newblock Error propagation for approximate policy and value iteration.
\newblock {\em Neural Information Processing Systems}, 2010.

\bibitem{fellows2021bayesian}
M.~Fellows, K.~Hartikainen, and S.~Whiteson.
\newblock Bayesian {B}ellman operators.
\newblock {\em Neural Information Processing Systems}, 2021.

\bibitem{geist2018anderson}
M.~Geist and B.~Scherrer.
\newblock Anderson acceleration for reinforcement learning.
\newblock {\em European Workshop on Reinforcement Learning}, 2018.

\bibitem{golowich2020last}
N.~Golowich, S.~Pattathil, C.~Daskalakis, and A.~Ozdaglar.
\newblock Last iterate is slower than averaged iterate in smooth convex-concave
  saddle point problems.
\newblock {\em Conference on Learning Theory}, 2020.

\bibitem{gordon1995stable}
G.~J. Gordon.
\newblock Stable function approximation in dynamic programming.
\newblock {\em International Conference on Machine Learning}, 1995.

\bibitem{goyal2022first}
V.~Goyal and J.~Grand-Cl\'ement.
\newblock A first-order approach to accelerated value iteration.
\newblock {\em Operations Research}, 71(2):517--535, 2022.

\bibitem{grand2021convex}
J.~Grand-Cl{\'e}ment.
\newblock From convex optimization to {MDP}s: A review of first-order,
  second-order and quasi-newton methods for {MDP}s.
\newblock {\em arXiv:2104.10677}, 2021.

\bibitem{halpern1967fixed}
B.~Halpern.
\newblock Fixed points of nonexpanding maps.
\newblock {\em Bulletin of the American Mathematical Society}, 73(6):957--961,
  1967.

\bibitem{hannah2018}
R.~Hannah, Y.~Liu, D.~O'Connor, and W.~Yin.
\newblock Breaking the span assumption yields fast finite-sum minimization.
\newblock {\em Neural Information Processing Systems}, 2018.

\bibitem{hessel2018rainbow}
M.~Hessel, J.~Modayil, H.~Van~Hasselt, T.~Schaul, G.~Ostrovski, W.~Dabney,
  D.~Horgan, B.~Piot, M.~Azar, and D.~Silver.
\newblock Rainbow: Combining improvements in deep reinforcement learning.
\newblock {\em Association for the Advancement of Artificial Intelligence},
  2018.

\bibitem{iutzeler2019generic}
F.~Iutzeler and J.~M. Hendrickx.
\newblock A generic online acceleration scheme for optimization algorithms via
  relaxation and inertia.
\newblock {\em Optimization Methods and Software}, 34(2):383--405, 2019.

\bibitem{kim2021accelerated}
D.~Kim.
\newblock Accelerated proximal point method for maximally monotone operators.
\newblock {\em Mathematical Programming}, 190(1--2):57--87, 2021.

\bibitem{kim2016optimized}
D.~Kim and J.~A. Fessler.
\newblock Optimized first-order methods for smooth convex minimization.
\newblock {\em Mathematical Programming}, 159(1--2):81--107, 2016.

\bibitem{kim2021optimizing}
D.~Kim and J.~A. Fessler.
\newblock Optimizing the efficiency of first-order methods for decreasing the
  gradient of smooth convex functions.
\newblock {\em Journal of Optimization Theory and Applications},
  188(1):192--219, 2021.

\bibitem{lee2021geometric}
J.~Lee, C.~Park, and E.~K. Ryu.
\newblock A geometric structure of acceleration and its role in making
  gradients small fast.
\newblock {\em Neural Information Processing Systems}, 2021.

\bibitem{lee2021fast}
S.~Lee and D.~Kim.
\newblock Fast extra gradient methods for smooth structured
  nonconvex-nonconcave minimax problems.
\newblock {\em Neural Information Processing Systems}, 2021.

\bibitem{leustean2007rates}
L.~Leustean.
\newblock Rates of asymptotic regularity for {H}alpern iterations of
  nonexpansive mappings.
\newblock {\em Journal of Universal Computer Science}, 13(11):1680--1691, 2007.

\bibitem{Lieder2021halpern}
F.~Lieder.
\newblock On the convergence rate of the {H}alpern-iteration.
\newblock {\em Optimization Letters}, 15(2):405--418, 2021.

\bibitem{Lutter2021ValueII}
M.~Lutter, S.~Mannor, J.~Peters, D.~Fox, and A.~Garg.
\newblock Value iteration in continuous actions, states and time.
\newblock {\em International Conference on Machine Learning}, 2021.

\bibitem{mainge2008convergence}
P.-E. Maing{\'e}.
\newblock Convergence theorems for inertial {KM}-type algorithms.
\newblock {\em Journal of Computational and Applied Mathematics},
  219(1):223--236, 2008.

\bibitem{massoud2009regularized}
A.~massoud Farahmand, M.~Ghavamzadeh, C.~Szepesv{\'a}ri, and S.~Mannor.
\newblock Regularized fitted {Q}-iteration for planning in continuous-space
  {M}arkovian decision problems.
\newblock {\em American Control Conference}, 2009.

\bibitem{mcmahan2005fast}
H.~B. McMahan and G.~J. Gordon.
\newblock Fast exact planning in {M}arkov decision processes.
\newblock {\em International Conference on Automated Planning and Scheduling},
  2005.

\bibitem{mnih2015human}
V.~Mnih, K.~Kavukcuoglu, D.~Silver, A.~A. Rusu, J.~Veness, M.~G. Bellemare,
  A.~Graves, M.~Riedmiller, A.~K. Fidjeland, G.~Ostrovski, S.~Petersen,
  C.~Beattie, A.~Sadik, I.~Antonoglou, H.~King, D.~Kumaran, D.~Wierstra,
  S.~Legg, and D.~Hassabis.
\newblock Human-level control through deep reinforcement learning.
\newblock {\em Nature}, 518(7540):529--533, 2015.

\bibitem{moore1993prioritized}
A.~W. Moore and C.~G. Atkeson.
\newblock Prioritized sweeping: Reinforcement learning with less data and less
  time.
\newblock {\em Machine Learning}, 13:103--130, 1993.

\bibitem{munos2005error}
R.~Munos.
\newblock Error bounds for approximate value iteration.
\newblock {\em Association for the Advancement of Artificial Intelligence},
  2005.

\bibitem{munos2008finite}
R.~Munos and C.~Szepesv{\'a}ri.
\newblock Finite-time bounds for fitted value iteration.
\newblock {\em Journal of Machine Learning Research}, 9(27):815--857, 2008.

\bibitem{nemirovski1992information}
A.~S. Nemirovski.
\newblock Information-based complexity of linear operator equations.
\newblock {\em Journal of Complexity}, 8(2):153--175, 1992.

\bibitem{nesterov2003Introductory}
Y.~Nesterov.
\newblock {\em Lectures on Convex Optimization}.
\newblock Springer, 2nd edition, 2018.

\bibitem{nesterov2021primal}
Y.~Nesterov, A.~Gasnikov, S.~Guminov, and P.~Dvurechensky.
\newblock Primal--dual accelerated gradient methods with small-dimensional
  relaxation oracle.
\newblock {\em Optimization Methods and Software}, 36(4):773--810, 2021.

\bibitem{nesterov1983method}
Y.~E. Nesterov.
\newblock A method for solving the convex programming problem with convergence
  rate $\mathcal{O}(1/k^2)$.
\newblock {\em Doklady Akademii Nauk SSSR}, 269(3):543--547, 1983.

\bibitem{niu2018generalized}
S.~Niu, S.~Chen, H.~Guo, C.~Targonski, M.~Smith, and J.~Kova{\v{c}}evi{\'c}.
\newblock Generalized value iteration networks: Life beyond lattices.
\newblock {\em Association for the Advancement of Artificial Intelligence},
  2018.

\bibitem{nota2020}
C.~Nota and P.~Thomas.
\newblock Is the policy gradient a gradient?
\newblock {\em International Conference on Autonomous Agents and Multiagent
  Systems}, 2020.

\bibitem{park2021factor}
C.~Park, J.~Park, and E.~K. Ryu.
\newblock Factor-$\sqrt{2}$ acceleration of accelerated gradient methods.
\newblock {\em Applied Mathematics \& Optimization}, 2023.

\bibitem{park2022exact}
J.~Park and E.~K. Ryu.
\newblock Exact optimal accelerated complexity for fixed-point iterations.
\newblock {\em International Conference on Machine Learning}, 2022.

\bibitem{park2023}
J.~Park and E.~K. Ryu.
\newblock Accelerated infeasibility detection of constrained optimization and
  fixed-point iterations.
\newblock {\em International Conference on Machine Learning}, 2023.

\bibitem{park2022anderson}
M.~Park, J.~Shin, and I.~Yang.
\newblock Anderson acceleration for partially observable {M}arkov decision
  processes: A maximum entropy approach.
\newblock {\em arXiv:2211.14998}, 2022.

\bibitem{peng1993efficient}
J.~Peng and R.~J. Williams.
\newblock Efficient learning and planning within the {D}yna framework.
\newblock {\em Adaptive Behavior}, 1(4):437--454, 1993.

\bibitem{10.5555/528623}
M.~L. Puterman.
\newblock {\em Markov Decision Processes: Discrete Stochastic Dynamic
  Programming}.
\newblock John Wiley and Sons, 1994.

\bibitem{reich1980strong}
S.~Reich.
\newblock Strong convergence theorems for resolvents of accretive operators in
  {B}anach spaces.
\newblock {\em Journal of Mathematical Analysis and Applications},
  75(1):287--292, 1980.

\bibitem{ryu2019ode}
E.~K. Ryu, K.~Yuan, and W.~Yin.
\newblock Ode analysis of stochastic gradient methods with optimism and
  anchoring for minimax problems.
\newblock {\em arXiv:1905.10899}, 2019.

\bibitem{sabach2017first}
S.~Sabach and S.~Shtern.
\newblock A first order method for solving convex bilevel optimization
  problems.
\newblock {\em SIAM Journal on Optimization}, 27(2):640--660, 2017.

\bibitem{salim2021optimal}
A.~Salim, L.~Condat, D.~Kovalev, and P.~Richt{\'a}rik.
\newblock An optimal algorithm for strongly convex minimization under affine
  constraints.
\newblock {\em International Conference on Artificial Intelligence and
  Statistics}, 2022.

\bibitem{scieur2020regularized}
D.~Scieur, A.~d'Aspremont, and F.~Bach.
\newblock Regularized nonlinear acceleration.
\newblock {\em Mathematical Programming}, 179(1--2):47--83, 2020.

\bibitem{shehu2018convergence}
Y.~Shehu.
\newblock Convergence rate analysis of inertial {K}rasnoselskii--{M}ann type
  iteration with applications.
\newblock {\em Numerical Functional Analysis and Optimization},
  39(10):1077--1091, 2018.

\bibitem{shi2019regularized}
W.~Shi, S.~Song, H.~Wu, Y.-C. Hsu, C.~Wu, and G.~Huang.
\newblock Regularized {A}nderson acceleration for off-policy deep reinforcement
  learning.
\newblock {\em Neural Information Processing Systems}, 2019.

\bibitem{shlakhter2010acceleration}
O.~Shlakhter, C.-G. Lee, D.~Khmelev, and N.~Jaber.
\newblock Acceleration operators in the value iteration algorithms for {M}arkov
  decision processes.
\newblock {\em Operations Research}, 58(1):193--202, 2010.

\bibitem{suh2023continuous}
J.~J. Suh, J.~Park, and E.~K. Ryu.
\newblock Continuous-time analysis of anchor acceleration.
\newblock {\em Neural Information Processing Systems}, 2023.

\bibitem{sun2021damped}
K.~Sun, Y.~Wang, Y.~Liu, B.~Pan, S.~Jui, B.~Jiang, L.~Kong, et~al.
\newblock Damped {A}nderson mixing for deep reinforcement learning:
  Acceleration, convergence, and stabilization.
\newblock {\em Neural Information Processing Systems}, 2021.

\bibitem{sutton1988learning}
R.~S. Sutton.
\newblock Learning to predict by the methods of temporal differences.
\newblock {\em Machine Learning}, 3:9--44, 1988.

\bibitem{sutton2018reinforcement}
R.~S. Sutton and A.~G. Barto.
\newblock {\em Reinforcement Learning: An introduction}.
\newblock MIT press, 2nd edition, 2018.

\bibitem{sutton1999}
R.~S. Sutton, D.~McAllester, S.~Singh, and Y.~Mansour.
\newblock Policy gradient methods for reinforcement learning with function
  approximation.
\newblock {\em Neural Information Processing Systems}, 1999.

\bibitem{sykora2020multi}
Q.~Sykora, M.~Ren, and R.~Urtasun.
\newblock Multi-agent routing value iteration network.
\newblock {\em International Conference on Machine Learning}, 2020.

\bibitem{szepesvarialgorithms}
C.~Szepesv{\'a}ri.
\newblock {\em Algorithms for Reinforcement Learning}.
\newblock Springer, 1st edition, 2010.

\bibitem{tamar2016value}
A.~Tamar, Y.~Wu, G.~Thomas, S.~Levine, and P.~Abbeel.
\newblock Value iteration networks.
\newblock {\em Neural Information Processing Systems}, 2016.

\bibitem{taylor2021optimal}
A.~Taylor and Y.~Drori.
\newblock An optimal gradient method for smooth strongly convex minimization.
\newblock {\em Mathematical Programming}, 199(1-2):557--594, 2023.

\bibitem{tosatto2017boosted}
S.~Tosatto, M.~Pirotta, C.~d'Eramo, and M.~Restelli.
\newblock Boosted fitted {Q}-iteration.
\newblock {\em International Conference on Machine Learning}, 2017.

\bibitem{tsitsiklis1994asynchronous}
J.~N. Tsitsiklis.
\newblock Asynchronous stochastic approximation and {Q}-learning.
\newblock {\em Machine Learning}, 16:185--202, 1994.

\bibitem{van2016deep}
H.~Van~Hasselt, A.~Guez, and D.~Silver.
\newblock Deep reinforcement learning with double {Q}-learning.
\newblock {\em Association for the Advancement of Artificial Intelligence},
  2016.

\bibitem{van2006performance}
B.~Van~Roy.
\newblock Performance loss bounds for approximate value iteration with state
  aggregation.
\newblock {\em Mathematics of Operations Research}, 31(2):234--244, 2006.

\bibitem{van2017fastest}
B.~Van~Scoy, R.~A. Freeman, and K.~M. Lynch.
\newblock The fastest known globally convergent first-order method for
  minimizing strongly convex functions.
\newblock {\em IEEE Control Systems Letters}, 2(1):49--54, 2018.

\bibitem{vieillard2020momentum}
N.~Vieillard, B.~Scherrer, O.~Pietquin, and M.~Geist.
\newblock Momentum in reinforcement learning.
\newblock {\em International Conference on Artificial Intelligence and
  Statistics}, 2020.

\bibitem{walker2011anderson}
H.~F. Walker and P.~Ni.
\newblock Anderson acceleration for fixed-point iterations.
\newblock {\em SIAM Journal on Numerical Analysis}, 49(4):1715--1735, 2011.

\bibitem{watkins1992q}
C.~J. Watkins and P.~Dayan.
\newblock Q-learning.
\newblock {\em Machine Learning}, 8:279--292, 1992.

\bibitem{wingate2005prioritization}
D.~Wingate, K.~D. Seppi, and S.~Mahadevan.
\newblock Prioritization methods for accelerating {MDP} solvers.
\newblock {\em Journal of Machine Learning Research}, 6(25):851--881, 2005.

\bibitem{wittmann1992approximation}
R.~Wittmann.
\newblock Approximation of fixed points of nonexpansive mappings.
\newblock {\em Archiv der Mathematik}, 58(5):486--491, 1992.

\bibitem{xu2002iterative}
H.-K. Xu.
\newblock Iterative algorithms for nonlinear operators.
\newblock {\em Journal of the London Mathematical Society}, 66(1):240--256,
  2002.

\bibitem{yoon2021accelerated}
T.~Yoon and E.~K. Ryu.
\newblock Accelerated algorithms for smooth convex-concave minimax problems
  with $\mathcal{O}(1/k^2)$ rate on squared gradient norm.
\newblock {\em International Conference on Machine Learning}, 2021.

\bibitem{yoon2022accelerated}
T.~Yoon and E.~K. Ryu.
\newblock Accelerated minimax algorithms flock together.
\newblock {\em arXiv:2205.11093}, 2022.

\bibitem{zeng2020asyncqvi}
Y.~Zeng, F.~Feng, and W.~Yin.
\newblock {AsyncQVI}: Asynchronous-parallel {Q}-value iteration for discounted
  {M}arkov decision processes with near-optimal sample complexity.
\newblock {\em International Conference on Artificial Intelligence and
  Statistics}, 2020.

\bibitem{zhang2020globally}
J.~Zhang, B.~O'Donoghue, and S.~Boyd.
\newblock Globally convergent type-{I} {A}nderson acceleration for nonsmooth
  fixed-point iterations.
\newblock {\em SIAM Journal on Optimization}, 30(4):3170--3197, 2020.

\bibitem{zhou2022practical}
K.~Zhou, L.~Tian, A.~M.-C. So, and J.~Cheng.
\newblock Practical schemes for finding near-stationary points of convex
  finite-sums.
\newblock {\em International Conference on Artificial Intelligence and
  Statistics}, 2022.

\end{thebibliography}

\newpage
\appendix

\section{Preliminaries}\label{s::property-Bellman}
% \begin{lemma}
%     For $V$ or $Q$, and any policy $\pi$, $T^{\star}$ and $T^{\pi}$ are $\gamma$-contractive and has unique fixed point $U^{\pi}$ and $U^{\star}$, respectively.
% \end{lemma}
 For notational unity, we use the symbol $U$ when both $V$ and $Q$ can be used.
\begin{lemma}\label{lem::monotonicity}{\protect{\cite[Lemma~1.1.1]{bertsekas2015dynamic}}}
    Let $0<\gamma\le 1$. If $U \le \tilde{U}$, then $T^{\pi} U\le T^{\pi}\tilde{U}, T^{\star}U \le T^{\star}\tilde{U}$.
\end{lemma}
\begin{lemma}\label{lem::nonexp_prob_kernel}
  Let $0<\gamma\le 1$. For any policy $\pi$, $\cP^{\pi}$ is a nonexpansive linear operator such that if $U \le \tilde{U}$, $\cP^{\pi} U\le \cP^{\pi}\tilde{U}$.
\end{lemma}
\begin{proof} If $r(s,a)=0$ for all $s \in \cS$ and $a \in \cA$, $T^{\pi}=\gamma \cP^{\pi}$. Then by Lemma \ref{lem::monotonicity} and $\gamma$-contraction of $T^{\pi}$, we have the desired result.  
%By definition of $T^{\pi}$, $\cP^{\pi}$ is linear and $\infn{T^{\pi}U-T^{\pi}\tilde{U}}=\gamma\infn{\cP^{\pi}U- \cP^{\pi}\tilde{U}} \le \gamma\infn{\tilde{U}-U}$. 
% First, linearity comes from linearity of expectation. 
% If $U=V$, for all $s \in \cS$,
%     \begin{align*}
% \abs{\mathbb{E}_{a \sim \pi(\cdot\,|\,s), s'\sim P(\cdot\,|\,s,a) }\left[ V(s')\right]} &\le \mathbb{E}_{a \sim \pi(\cdot\,|\,s), s'\sim P(\cdot\,|\,s,a) }\left[ |V(s')|\right]\\&\le \mathbb{E}_{a \sim \pi(\cdot\,|\,s), s'\sim P(\cdot\,|\,s,a) }\left[ \infn{V}\right]\\&= \infn{V} 
% \end{align*}
% where the first inequality is from Jensen's inequality. Similarly, for all $(s,a) \in \cS \times \cA$,
% \begin{align*}
%     \abs{\mathbb{E}_{a' \sim \pi(\cdot\,|\,s'), s'\sim P(\cdot\,|\,s,a)}\left[Q(s',a')\right]} &\le \mathbb{E}_{a' \sim \pi(\cdot\,|\,s'), s'\sim P(\cdot\,|\,s,a)}\left[\abs{Q(s',a')}\right]\\
%     &\le \mathbb{E}_{a' \sim \pi(\cdot\,|\,s'), s'\sim P(\cdot\,|\,s,a)}\left[\infn{Q}\right]\\
%     &= \infn{Q}
% \end{align*}
% where the first inequality is from Jensen's inequality.
\end{proof}
\begin{lemma}\label{lem::anti-fixed-point}
  Let $0<\gamma< 1$. Let $T^{\star}$ and $\hat{T}^{\star}$ respectively be the Bellman optimality and anti-optimality operators.
Let $U^\star$ and $\hat{U}^\star$ respectively be the fixed points of $T^{\star}$ and $\hat{T}^{\star}$. Then $\hat{U}^{\star}  \le U^\star $. 
\end{lemma}
\begin{proof}
    By definition, $\hat{U}^{\star}=\hat{T}^{\star}\hat{U}^{\star} \le T^{\star}\hat{U}^{\star}$. Thus, $\hat{U}^{\star} \le \lim_{m\rightarrow \infty} \para{T^{\star}}^m\hat{U}^{\star}= U^{\star}$. 
\end{proof}

\section{Omitted proofs in Section \ref{sec::Anc-VI} }\label{s::omitted-Anc-VI-proofs}
 First, we prove the following lemma by induction. 
\begin{lemma}\label{lem::Anc-VIE-1}
Let $0<\gamma\le 1$, and if $\gamma=1$, assume a fixed point $U^{\pi}$ exists. For the iterates $\{U^k\}_{k=0,1,\dots}$ of \ref{eq:anc-vi},
    \begin{align*}
        T^{\pi}U^k-U^k&=\sum_{i=1}^k\left[\para{\beta_i-\beta_{i-1}(1-\beta_{i})}\para{\Pi^k_{j=i+1}(1-\beta_j)}\para{\gamma\cP^{\pi}}^{k-i+1}(U^0-U^{\pi})\right]\\
    &\quad -\beta_k(U^0-U^{\pi})+\para{\Pi^k_{j=1}(1-\beta_j)} \para{\gamma\cP^{\pi}}^{k+1}(U^0-U^{\pi})
    \end{align*}
    where %$\sum_{i=1}^{k=0}\left[\para{\beta_i-\beta_{i-1}(1-\beta_{i})}\para{\Pi^k_{j=i+1}(1-\beta_j)}\para{\gamma\cP^{\pi}}^{k-i+1}(U^0-U^{\pi})\right]=0, 
    $\para{\Pi^k_{j=k+1}(1-\beta_j)}=1$ and $\beta_0=1$.
\end{lemma}
\begin{proof}
If $k=0$, we have 
\begin{align*}
    T^{\pi}U^0-U^0&=T^{\pi}U^0-U^{\pi}-(U^0-U^{\pi})\\
    &=T^{\pi}U^0-T^{\pi}U^{\pi}-(U^0-U^{\pi})\\
    &=\gamma\cP^{\pi}\para{U^0-U^{\pi}}-(U^0-U^{\pi})\\
\end{align*}
If $k=m$, since $T^{\pi}$ is a linear operator,
\begin{align*}
    T^{\pi}U^m-U^m&=T^{\pi}U^{m}-(1-\beta_m)T^{\pi}U^{m-1}-\beta_mU^0\\
    &=T^{\pi}U^{m}-(1-\beta_{m})T^{\pi}U^{m-1}-\beta_{m}U^{\pi}-\beta_{m}(U^0-U^{\pi})\\
    &=T^{\pi}U^{m}-(1-\beta_{m})T^{\pi}U^{m-1}-\beta_{m}T^{\pi}U^{\pi}-\beta_{m}(U^0-U^{\pi})\\
    &=\gamma\cP^{\pi}(U^{m}-(1-\beta_{m})U^{m-1}-\beta_{m}U^{\pi})-\beta_{m}(U^0-U^{\pi})\\
    &=\gamma\cP^{\pi}(\beta_{m}(U^0-U^{\pi})+(1-\beta_{m})(T^{\pi}U^{m-1}-U^{m-1}))-\beta_{m}(U^0-U^{\pi})\\
    &=(1-\beta_m)\gamma\cP^{\pi}\sum_{i=1}^{m-1}\left[\para{\beta_{i}-\beta_{i-1}(1-\beta_{i})}\para{\Pi^{m-1}_{j=i+1}(1-\beta_j)}\para{\gamma\cP^{\pi}}^{m-1-i+1}(U^0-U^{\pi})\right]\\
    &\quad -(1-\beta_m)\gamma\cP^{\pi}\beta_{m-1}(U^0-U^{\pi})+(1-\beta_m)\gamma\cP^{\pi}\para{\Pi^{m-1}_{j=1}(1-\beta_j)} \para{\gamma\cP^{\pi}}^{m}(U^0-U^{\pi})\\
    &\quad +\beta_{m}\gamma\cP^{\pi}(U^0-U^{\pi})-\beta_{m}(U^0-U^{\pi})\\
    &=\sum_{i=1}^{m-1}\left[\para{\beta_{i}-\beta_{i-1}(1-\beta_{i})}\para{\Pi^{m}_{j=i+1}(1-\beta_j)}\para{\gamma\cP^{\pi}}^{m-i+1}(U^0-U^{\pi})\right]\\
    &\quad -\beta_{m-1}(1-\beta_m)\gamma\cP^{\pi}(U^0-U^{\pi})+\beta_{m}\gamma\cP^{\pi}(U^0-U^{\pi})\\
    &\quad -\beta_{m}(U^0-U^{\pi})+\para{\Pi^m_{j=1}(1-\beta_j)} \para{\gamma\cP^{\pi}}^{m+1}(U^0-U^{\pi})\end{align*}
\begin{align*}
    &=\sum_{i=1}^m\left[\para{\beta_{i}-\beta_{i-1}(1-\beta_{i})}\para{\Pi^m_{j=i+1}(1-\beta_j)}\para{\gamma\cP^{\pi}}^{m-i+1}(U^0-U^{\pi})\right]\\
    &\quad -\beta_{m}(U^0-U^{\pi})+\para{\Pi^m_{j=1}(1-\beta_j)} \para{\gamma\cP^{\pi}}^{m+1}(U^0-U^{\pi})  
\end{align*}    
\end{proof}
%     Since $T^{\pi}$ is linear operator,
% \[U^k=\sum_{k=0}^n \beta_k \Pi^n_{i=k+1}(1-\beta_i) \para{T^{\pi}}^{n-k}U^0,\quad T^{\pi}U^k=\sum_{k=0}^n\beta_k\Pi^n_{i=k+1}(1-\beta_i) \para{T^{\pi}}^{n-k+1}U^0.\]
% By property of fixed point, $U^{\star}=\para{T^{\pi}}^k U^{\star}$ and $\para{T^{\pi}}^kU^0-\para{T^{\pi}}^k U^{\star}=\para{\cP^{\pi}}^k U^0-\para{\cP^{\pi}}^kU^{\star}$ for $k \in \mathrm{N}$. Thus, with $\sum_{k=0}^n \beta_k \Pi^n_{i=k+1}(1-\beta_i)=1$, we get    
% \begin{align*}
%     U^k-U^{\star}&=\sum_{k=0}^n \beta_k \Pi^n_{i=k+1}(1-\beta_i) \para{T^{\pi}}^{n-k}U^0-\sum_{k=0}^n \beta_k\Pi^n_{i=k+1}(1-\beta_i) \para{T^{\pi}}^{n-k}U^{\star}\\&=\sum_{k=0}^n \beta_k \Pi^n_{i=k+1}(1-\beta_i) \para{\cP^{\pi}}^{n-k}(U^0-U^{\star})
%     \end{align*}
%     and similarly,
% \[T^{\pi}U^k-U^{\star}=\sum_{k=0}^n\beta_k\Pi^n_{i=k+1}(1-\beta_i) \para{\cP^{\pi}}^{n-k+1}(U^0-U^{\star}).\]

% Then, we have 
% \begin{align*}
%     T^{\pi}U^k-U^{\star}-\para{U^k-U^{\star}}&=\sum_{k=1}^n[\beta_k-\beta_{k-1}(1-\beta_{k})]\Pi^n_{i=k+1}(1-\beta_i)\para{\cP^{\pi}}^{n-k+1}(U^0-U^{\star})\\&\quad -\beta_n(U^0-U^{\star})+\Pi^n_{i=1}(1-\beta_i) \para{\cP^{\pi}}^{n+1}(U^0-U^{\star}).\end{align*}
Now, we prove the first rate of Theorem \ref{thm::Anc-VIE}. 

\begin{proof}[Proof of first rate in Theorem \ref{thm::Anc-VIE}]
Taking $\infn{\cdot}$-norm both sides of equality in Lemma \ref{lem::Anc-VIE-1}, we get 
\begin{align*}
       \infn{T^{\pi}U^k-U^k} &\le\sum_{i=1}^k\left|\beta_i-\beta_{i-1}(1-\beta_{i})\right|\para{\Pi^k_{j=i+1}(1-\beta_j)}\infn{\para{\gamma\cP^{\pi}}^{k-i+1}(U^0-U^{\pi})}\\&\quad +\beta_k\infn{U^0-U^{\pi}}+\para{\Pi^k_{i=1}(1-\beta_i)} \infn{\para{\gamma\cP^{\pi}}^{k+1}(U^0-U^{\pi})}\\
    &\le \bigg(\sum_{i=1}^k\gamma^{k-i+1}\left|\beta_i-\beta_{i-1}(1-\beta_{i})\right|\para{\Pi^k_{j=i+1}(1-\beta_j)}+\beta_k+\gamma^{k+1}\Pi^k_{j=1}(1-\beta_j)\bigg)\\&\quad \infn{U^0-U^{\star}}\\
    &=\para{\sum_{i=1}^k\gamma^{k+i-1}\frac{\para{1-\gamma^{2}}^2}{1-\gamma^{(2k+2)}}+\gamma^{2k}\frac{1-\gamma^{2}}{1-\gamma^{2k+2}}+\gamma^{k+1}\frac{1-\gamma^{2}}{1-\gamma^{2k+2}}}\infn{U^0-U^{\pi}}\\
     &= \frac{\para{\gamma^{-1}-\gamma}\para{1+2\gamma-\gamma^{k+1}}}{\para{\gamma^{k+1}}^{-1}-\gamma^{k+1}}\infn{U^0-U^{\pi}},
     %\frac{\para{1+\gamma}\para{1+2\gamma-\gamma^{k+1}}}{\para{1+\gamma^{k+1}}}\frac{\gamma^{k}}{\sum_{i=0}^k \gamma^{i}}\infn{U^0-U^{\pi}}, 
\end{align*}
where the first inequality comes from triangular inequality, second inequality is from 
Lemma \ref{lem::nonexp_prob_kernel}, and equality come from calculations.
\end{proof}

For the second rate of Theorem \ref{thm::Anc-VIE}, we introduce following lemma.
\begin{lemma}\label{lem::Inital_cond}
Let $0<\gamma<1$. Let $T$ be Bellman consistency or optimality operator. For the iterates $\{U^k\}_{k=0,1,\dots}$ of \ref{eq:anc-vi}, if $U^0 \le TU^0$,  then $U_{k-1} \le U_k \le TU_{k-1} \le TU_k \le U^{\star}$ for $1 \le k$. Also, if $U^0 \ge TU^0$, then $U_{k-1} \ge U_k \ge TU_{k-1} \ge TU_k \ge U^{\star}$ for $1 \le k$.
\end{lemma}
\begin{proof}
First, let $U^0 \le TU^0$. If $k=1$, $ U^0\le \beta_1U^0+(1-\beta_1) TU^0=U_1\ \le TU^0$ by assumption. Since $U^0 \le U^1$, $T U^0 \le T U^1$ by monotonicity of Bellman consistency and optimality operators.

By induction, 
    \[ U^k=\beta_k U^{0}+(1-\beta_k) TU^{k-1} \le T U^{k-1}, \]
    and since  $\beta_k \le \beta_{k-1}$,
    \begin{align*}
        \beta_kU^{0}+(1-\beta_k) TU^{k-1} &\ge \beta_{k-1}U^{0}+(1-\beta_{k-1}) TU^{k-1}\\
        &\ge \beta_{k-1}U^{0}+(1-\beta_{k-1}) TU^{k-2}\\
        &=U^{k-1}.
    \end{align*}
    Also, $U^{k-1} \le U^k$ implies $TU^{k-1} \le TU^k $ by monotonicity of Bellman consistency and optimality operators, and $U^k\le TU^k$ implies that $ U^k \le \lim_{m \rightarrow \infty} \para{T}^m U^k = U^{\star}$ for all $k=0,1,\dots$.

Now, suppose $U^0 \ge TU^0$. If $k=1$, $ U^0\ge \beta_1U^0+(1-\beta_1) TU^0=U_1\ \ge TU^0$ by assumption. Since $U^0 \ge U^1$, $T U^0 \ge T U^1$ by monotonicity of Bellman consistency and optimality operators.
  
    By induction, 
    \[ U^k=\beta_k U^{0}+(1-\beta_k) TU^{k-1} \ge T U^{k-1}, \]
    and since  $\beta_k \le \beta_{k-1}$,
    \begin{align*}
        \beta_kU^{0}+(1-\beta_k) TU^{k-1} &\le \beta_{k-1}U^{0}+(1-\beta_{k-1}) TU^{k-1}\\
        &\le \beta_{k-1}U^{0}+(1-\beta_{k-1}) TU^{k-2}\\
        &=U^{k-1}.
    \end{align*}
    Also, $U^{k-1} \ge U^k$ implies $TU^{k-1} \ge TU^k $ by monotonicity of Bellman consistency and optimality operators, and $U_k \ge TU_k$ implies that $ U^k \ge \lim_{m \rightarrow \infty} \para{T}^m U^k = U^{\star}$ for all $k=0,1,\dots$.
    \end{proof}

Now, we prove following key lemmas.
\begin{lemma}\label{lem::Anc-VIE-2-1}
Let $0<\gamma\le 1$, and assume a fixed point $U^{\pi}$ exists if $\gamma=1$. For the iterates $\{U^k\}_{k=0,1,\dots}$ of \ref{eq:anc-vi}, if $U^0 \le U^{\pi}$,  
    \begin{align*}
        T^{\pi}U^k-U^k& \le \sum_{i=1}^k\left[\para{\beta_i-\beta_{i-1}(1-\beta_{i})}\para{\Pi^k_{j=i+1}(1-\beta_j)}\para{\gamma\cP^{\pi}}^{k-i+1}(U^0-U^{\pi})\right]\\
    &\quad -\beta_k(U^0-U^{\pi}),
    \end{align*}
    where %$\sum_{i=1}^{k=0}\left[\para{\beta_i-\beta_{i-1}(1-\beta_{i})}\para{\Pi^k_{j=i+1}(1-\beta_j)}\para{\gamma\cP^{\pi}}^{k-i+1}(U^0-U^{\pi})\right]=0, 
    $\para{\Pi^k_{j=k+1}(1-\beta_j)}=1$ and $\beta_0=1$.
\end{lemma}
\begin{lemma}\label{lem::Anc-VIE-2-2}
Let $0<\gamma < 1$. For the iterates $\{U^k\}_{k=0,1,\dots}$ of \ref{eq:anc-vi}, if $U^0 \ge T^{\pi}U^0$,
\begin{align*}
        T^{\pi}U^k-U^k& \ge \sum_{i=1}^k\left[\para{\beta_i-\beta_{i-1}(1-\beta_{i})}\para{\Pi^k_{j=i+1}(1-\beta_j)}\para{\gamma\cP^{\pi}}^{k-i+1}(U^0-U^{\pi})\right]\\
    &\quad -\beta_k(U^0-U^{\pi}),
    \end{align*}
    where %$\sum_{i=1}^{k=0}\left[\para{\beta_i-\beta_{i-1}(1-\beta_{i})}\para{\Pi^k_{j=i+1}(1-\beta_j)}\para{\gamma\cP^{\pi}}^{k-i+1}(U^0-U^{\pi})\right]=0, 
    $\para{\Pi^k_{j=k+1}(1-\beta_j)}=1$ and $\beta_0=1$.
\end{lemma}
\begin{proof}[Proof of Lemma \ref{lem::Anc-VIE-2-1}]
    If $U^0\le U^{\pi},$ we get 
    \begin{align*}
        T^{\pi}U^k-U^k&=\sum_{i=1}^k\left[\para{\beta_i-\beta_{i-1}(1-\beta_{i})}\para{\Pi^k_{j=i+1}(1-\beta_j)}\para{\gamma\cP^{\pi}}^{k-i+1}(U^0-U^{\pi})\right]\\
    &\quad -\beta_k(U^0-U^{\pi})+\para{\Pi^k_{j=1}(1-\beta_j)} \para{\gamma\cP^{\pi}}^{k+1}(U^0-U^{\pi})\\
    &\le \sum_{i=1}^k\left[\para{\beta_i-\beta_{i-1}(1-\beta_{i})}\para{\Pi^k_{j=i+1}(1-\beta_j)}\para{\gamma\cP^{\pi}}^{k-i+1}(U^0-U^{\pi})\right] -\beta_k(U^0-U^{\pi}),  
    \end{align*}
by Lemma \ref{lem::Anc-VIE-1} and the fact that $\para{\Pi^k_{j=1}(1-\beta_j)} \para{\gamma\cP^{\pi}}^{k+1}(U^0-U^{\pi})\le 0$.
\end{proof}
\begin{proof}[Proof of Lemma \ref{lem::Anc-VIE-2-2}]   
If $U^0 \ge TU^0$, $U^0-U^{\pi} \ge 0$ by Lemma \ref{lem::Inital_cond}. Hence, by Lemma \ref{lem::Anc-VIE-1}, we have 
    \begin{align*}
       T^{\pi}U^k-U^k \ge\sum_{i=1}^k\left[\para{\beta_i-\beta_{i-1}(1-\beta_{i})}\para{\Pi^k_{j=i+1}(1-\beta_j)}\para{\gamma\cP^{\pi}}^{k-i+1}(U^0-U^{\pi})\right]-\beta_k(U^0-U^{\pi}),
    \end{align*}
    since $0 \le \para{\Pi^k_{j=1}(1-\beta_j)} \para{\gamma\cP^{\pi}}^{k+1}(U^0-U^{\pi}) $.
\end{proof}
Now, we prove the second rates of Theorem \ref{thm::Anc-VIE}.

\begin{proof}[Proof of second rates in Theorem \ref{thm::Anc-VIE}]
  Let $0<\gamma<1$. By Lemma \ref{lem::Inital_cond}, if $U^0 \le T^{\pi}U^0$, then $U^0 \le U^{\pi}$. Hence,
    \begin{align*}
        0 &\le T^{\pi}U^k-U^k\\
    &\le \sum_{i=1}^k\left[\para{\beta_i-\beta_{i-1}(1-\beta_{i})}\para{\Pi^k_{j=i+1}(1-\beta_j)}\para{\gamma\cP^{\pi}}^{k-i+1}(U^0-U^{\pi})\right] -\beta_k(U^0-U^{\pi}),  
    \end{align*}
by Lemma \ref{lem::Anc-VIE-2-1}. Taking $\infn{\cdot}$-norm both sides, we have 
\[\infn{T^{\pi}U^k-U^k} \le  \frac{\para{\gamma^{-1}-\gamma}\para{1+\gamma-\gamma^{k+1}}}{\para{\gamma^{k+1}}^{-1}-\gamma^{k+1}}\infn{U^0-U^{\pi}}.\]
Otherwise, if $U^0 \ge TU^0$, $U^k \ge TU^k$ by Lemma \ref{lem::Inital_cond}. Since 
    \begin{align*}
        0 &\ge T^{\pi}U^k-U^k \\&\ge\sum_{i=1}^k\left[\para{\beta_i-\beta_{i-1}(1-\beta_{i})}\para{\Pi^k_{j=i+1}(1-\beta_j)}\para{\gamma\cP^{\pi}}^{k-i+1}(U^0-U^{\pi})\right]-\beta_k(U^0-U^{\pi}),
    \end{align*}
by Lemma \ref{lem::Anc-VIE-2-2}, taking $\infn{\cdot}$-norm both sides, we obtain same rate as before. 

Lastly, Taylor series expansion for both rates at $\gamma=1$ is  
\begin{align*}
     \frac{\para{\gamma^{-1}-\gamma}\para{1+2\gamma-\gamma^{k+1}}}{\para{\gamma^{k+1}}^{-1}-\gamma^{k+1}} &= \frac{2}{k+1}-\frac{k-1}{k+1}(\gamma-1)+O((\gamma-1)^2),\\ 
      \frac{\para{\gamma^{-1}-\gamma}\para{1+\gamma-\gamma^{k+1}}}{\para{\gamma^{k+1}}^{-1}-\gamma^{k+1}} &= \frac{1}{k+1}-\frac{k}{k+1}(\gamma-1)+O((\gamma-1)^2).
\end{align*}
\end{proof}

For the analyses of \ref{eq:anc-vi} for Bellman optimality operator, we first prove following two lemmas. 
\begin{lemma}\label{lem::Anc-VIC-1-1-H1} Let $0<\gamma\le 1$. If $\gamma=1$, assume a fixed point $U^{\star}$ exists. Then, if $ 0\le \alpha \le 1$ and $U-(1-\alpha) \tilde{U}-\alpha U^{\star} \le \bar{U} $, there exist nonexpansive linear operator $\cP_H$ such that 
\begin{align*}
    T^{\star}U-(1-\alpha) T^{\star}\tilde{U}-\alpha T^{\star}U^{\star} &\le \gamma\cP_H\bar{U}.
\end{align*}    
\end{lemma}
\begin{lemma}\label{lem::Anc-VIC-1-1-H2} Let $0<\gamma< 1$. If $ 0\le \alpha \le 1$ and  $\bar{U} \le U-(1-\alpha) \tilde{U}-\alpha \hat{U}^{\star}  $, then there exist nonexpansive linear operator $\hat{\cP}_H$ such that 
\begin{align*}
    \gamma\hat{\cP}_H(\bar{U}) & \le T^{\star}U-\alpha T^{\star}\tilde{U}-(1-\alpha)\hat{T}^{\star}\hat{U}^{\star}.
\end{align*}
\end{lemma}
\begin{proof}[Proof of Lemma \ref{lem::Anc-VIC-1-1-H1}]
First, let $U=V, \tilde{U}=\tilde{V}, U^{\star}=V^{\star}, \bar{U}=\bar{V}$, and $ V-(1-\alpha) \tilde{V}-\alpha V^{\star} \le \bar{V}$. 

If action space is finite, 
\begin{align*}
    T^{\star}V-(1-\alpha)T^{\star}\tilde{V}-\alpha T^{\star}V^{\star} &\le T^{\pi}V-(1-\alpha)T^{\pi}\tilde{V}-\alpha T^{\pi}V^{\star}\\
    &= \gamma\cP^{\pi}\para{V-(1-\alpha) \tilde{V}-\alpha V^{\star} }\\
     &\le \gamma\cP^{\pi}\bar{V} 
\end{align*}
where $\pi$ is the greedy policy satisfying $T^{\pi}V=T^{\star}V$, first inequality is from $T^{\pi}\tilde{V} \le T^{\star}\tilde{V}$ and $T^{\pi}V^{\star} \le T^{\star}V^{\star}$, and second inequality comes from Lemma \ref{lem::monotonicity}. Thus, we can conclude $\cP_H = \cP^{\pi}$.

Otherwise, if action space is infinite, define $\cP(c\bar{V})=c \sup_{s \in \cS}\bar{V}(s)$ for $c\in \real$ and previously given $\bar{V}$. Let $M$ be linear space spanned by $\bar{V}$ with $\infn{\cdot}$-norm. Then, $\cP$ is linear functional on $M$ and $\|\cP\|_{\text{op}} \le 1$ %where $\|\cdot \|_{\text{op}}$ is operator norm, 
since $\frac{|c \sup_{s \in \cS}\bar{V}(s)|}{\infn{c\bar{V}}} \le 1$. Due to Hahn--Banach extension Theorem, there exist linear functional $\cP_h \colon \cF(\cS) \rightarrow \real$ with $\cP_h(\bar{V})=\sup_{s \in \cS}\bar{V}(s)$ and $\|\cP_h\|_{\text{op}} \le 1$. Furthermore, we can define $\cP_H \colon \cF(\cS) \rightarrow \cF(\cS)$ such that $\cP_H V(s)=\cP_h(V)$  for all $s \in \cS$. Then, since $\|\cP_H(V)\|_{\infty}=|\cP_h(V)| \le \|\cP_h\|_{\text{op}} \le 1$ for $\infn{V} \le 1$, we have $\|\cP_H\|_{\infty} \le 1$. %$\frac{|c \sup_{s \in \cS}\bar{V}(s)|}{\infn{c\bar{V}}} \le 1$
Therefore, $\cP_H$ is nonexpansive linear operator in $\infn{\cdot}$-norm. Then,
\begin{align*}
    &T^{\star}V(s)-(1-\alpha)T^{\star}\tilde{V}(s)-\alpha T^{\star}V^{\star}(s)\\& = \sup_{a \in \cA} \bigg\{r(s,a)+\gamma \mathbb{E}_{s'\sim P(\cdot\,|\,s,a) }\left[V(s')\right]\bigg\}- \sup_{a \in \cA} \bigg\{(1-\alpha) r(s,a)+(1-\alpha) \gamma\mathbb{E}_{s'\sim P(\cdot\,|\,s,a) }\left[\tilde{V}(s')\right]\bigg\}\\
    &\quad - \sup_{a \in \cA} \bigg\{\alpha r(s,a)+\alpha  \gamma\mathbb{E}_{s'\sim P(\cdot\,|\,s,a) }\left[V^{\star}(s')\right]\bigg\}\\
    & \le \sup_{a \in \cA} \bigg\{r(s,a)+\gamma \mathbb{E}_{s'\sim P(\cdot\,|\,s,a) }\left[V(s')\right]- (1-\alpha) r(s,a)-(1-\alpha) \gamma\mathbb{E}_{s'\sim P(\cdot\,|\,s,a) }\left[\tilde{V}(s')\right]\bigg\}\\
    &\quad - \sup_{a \in \cA} \bigg\{\alpha r(s,a)+\alpha \gamma\mathbb{E}_{s'\sim P(\cdot\,|\,s,a) }\left[V^{\star}(s')\right]\bigg\}\\
    & \le \gamma \sup_{a \in \cA} \bigg\{ \mathbb{E}_{s'\sim P(\cdot\,|\,s,a) }\left[V(s')-(1-\alpha) \tilde{V}(s')-\alpha V^{\star}(s')\right]\bigg\}\\ &\le 
\gamma \sup_{s' \in \cS} \{ V(s')-(1-\alpha) \tilde{V}(s')-\alpha V^{\star}(s')\}\\& \le \gamma \sup_{s' \in \cS}\bar{V}(s').
\end{align*}
 for all $s \in \cS$. Therefore, we have 
% Define $\cP^{s}(c \bar{V})=c \sup_{a \in \cA} \left\{ \mathbb{E}_{s'\sim P(\cdot\,|\,s,a) }[\bar{V}] \right\}$ for $c\in \real$, and $M$ be linear space spanned by $\bar{V}$ with $\infn{\cdot}$-norm. Then, $\cP^{s}$ is linear functional on $M$ and $\|\cP^{s}\|= 1$ since 
% \begin{align*}
%   \left|c \sup_{a \in \cA} \left\{ \mathbb{E}_{s'\sim P(\cdot\,|\,s,a) }[\bar{V}(s')] \right\}\right| &\le   \left|\sup_{a \in \cA} \left\{ \mathbb{E}_{s'\sim P(\cdot\,|\,s,a) }[|c| \bar{V}(s')] \right\}\right| \le \left|\sup_{a \in \cA} \left\{ \mathbb{E}_{s'\sim P(\cdot\,|\,s,a) }[\infn{c \bar{V}}] \right\}\right|.
% \end{align*}
% Due to Hahn--Banach extension Theorem, there exist linear functional $\cP_H^{s} \colon \cB(\cS) \rightarrow \real $ such that $\|\cP_H^{s}\|=1$ and $\cP_H^{s}(\bar{V})=\sup_{a \in \cA} \left\{ \mathbb{E}_{s'\sim P(\cdot\,|\,s,a) }\left[\bar{V}(s')\right]\right\}$. Furthermore, applying Hahn--Banach extension Theorem to all $s \in \cS$, we can define $\cP_H \colon \cF(\cS) \rightarrow \cF(\cS)$ such that $\cP_H V(s)=\cP_H^{s}(V)$ and $|\cP_H(V)(s)| \le \infn{V}$ for all $s \in \cS$. Thus, $\cP_H$ is nonexpansive linear operator in $\infn{\cdot}$-norm and finally, we have
\begin{align*}
    T^{\star}V-(1-\alpha) T^{\star}\tilde{V}-\alpha T^{\star}V^{\star} \le \gamma \cP_H(\bar{V}). 
\end{align*}
Similarly, let $U=Q, \tilde{U}=\tilde{Q}, U^{\star}=Q^{\star}, \bar{U}=\bar{Q}$, and $Q-(1-\alpha) \tilde{Q}-\alpha Q^{\star} \le \bar{Q}$.  

If action space is finite, 
\begin{align*}
    T^{\star}Q-(1-\alpha)T^{\star}\tilde{Q}-\alpha T^{\star}Q^{\star} 
    &\le \gamma\cP^{\pi}\para{Q-(1-\alpha) \tilde{Q}-\alpha Q^{\star} }\\
    &\le \gamma\cP^{\pi}\bar{Q}
\end{align*}
where $\pi$ is the greedy policy satisfying $T^{\pi}Q=T^{\star}Q$, first inequality is from $T^{\pi}\tilde{Q} \le T^{\star}\tilde{Q}$ and $T^{\pi}Q^{\star} \le T^{\star}Q^{\star}$, and second inequality comes from Lemma \ref{lem::monotonicity}. Then, we can conclude $\cP_H = \cP^{\pi}$.

Otherwise, if action space is infinite, define $\cP(c \bar{Q})=c \sup_{(s',a') \in \cS \times \cA} \bar{Q}(s',a')$ for $c \in \real$ and previously given $\bar{Q}$. Let $M$ be linear space spanned by $\bar{Q}$ with $\infn{\cdot}$-norm. Then, $\cP$ is linear functional on $M$ and $\|\cP\|_{\text{op}} \le 1$.
 Due to Hahn--Banach extension Theorem, there exist linear functional $\cP_h\colon \cF(\cS \times \cA) \rightarrow \real $ with $\cP_h(\bar{Q})= \sup_{(s',a') \in \cS \times \cA} \bar{Q}(s',a')$ and $\|\cP_h\|_{\text{op}} \le 1$.  Furthermore, we can define $\cP_H \colon \cF(\cS \times \cA) \rightarrow \cF(\cS \times \cA)$ such that $\cP_H Q(s,a)=\cP_h(Q)$ for all $(s,a) \in \cS \times \cA$ and $\|P_H\|_{\infty} \le 1$. Therefore, $\cP_H$ is nonexpansive linear operator in $\infn{\cdot}$-norm. Then,
\begin{align*}
    &T^{\star}Q(s,a)-(1-\alpha)T^{\star}\tilde{Q}(s,a)-\alpha T^{\star}Q^{\star}(s,a)\\& = r(s,a)+\gamma\mathbb{E}_{s'\sim P(\cdot\,|\,s,a)}\left[\sup_{a' \in \cA} Q(s',a')\right] - (1-\alpha) r(s,a)-(1-\alpha)\gamma\mathbb{E}_{s'\sim P(\cdot\,|\,s,a)}\left[\sup_{a' \in \cA} \tilde{Q}(s',a')\right]\\
    &\quad - \alpha r(s,a)-\alpha \gamma\mathbb{E}_{s'\sim P(\cdot\,|\,s,a)}\left[\sup_{a' \in \cA} Q^{\star}(s',a')\right]\\
    & \le \gamma\mathbb{E}_{s'\sim P(\cdot\,|\,s,a)}\left[\sup_{a' \in \cA} \left\{Q(s',a')-(1-\alpha)\tilde{Q}(s',a')\right\}\right]-\gamma\mathbb{E}_{s'\sim P(\cdot\,|\,s,a)}\left[\sup_{a' \in \cA} \alpha Q(s',a')\right]\\
    & \le \gamma\mathbb{E}_{s'\sim P(\cdot\,|\,s,a)}\left[\sup_{a' \in \cA} \left\{Q(s',a')-(1-\alpha)\tilde{Q}(s',a')-\alpha Q^{\star}(s',a')\right\}\right]\\
     & \le \gamma \sup_{(s',a') \in \cS \times \cA}\left\{Q(s',a')-(1-\alpha)\tilde{Q}(s',a')-\alpha Q^{\star}(s',a')\right\},\\
    & \le \gamma \sup_{(s',a') \in \cS \times \cA}\bar{Q}(s',a')
\end{align*}
for all $(s,a) \in \cS \times \cA$. Therefore, we have 
\begin{align*}
    T^{\star}Q-(1-\alpha) T^{\star}\tilde{Q}-\alpha T^{\star}Q^{\star}\le \gamma\cP_H (\bar{Q}).
\end{align*}
\end{proof}
\begin{proof}[Proof of Lemma \ref{lem::Anc-VIC-1-1-H2}]
Note that  $\hat{T}^{\star}$ is Bellman anti-optimality operators for $V$ or $Q$, and $\hat{U}^\star$ is the fixed point of $\hat{T}^{\star}$. First, let $U=V, \tilde{U}=\tilde{V}, \hat{U}^{\star}=\hat{V}^{\star}, \bar{U}=\bar{V}$, and $\bar{V} \le V-(1-\alpha) \tilde{V}-\alpha \hat{V}^{\star} $. Then,
\begin{align*}
    &T^{\star}V(s)-(1-\alpha) T^{\star}\tilde{V}(s)-\alpha \hat{T}^{\star}\hat{V}^{\star}(s)\\& = \sup_{a \in \cA} \bigg\{r(s,a)+\gamma \mathbb{E}_{s'\sim P(\cdot\,|\,s,a) }\left[V(s')\right]\bigg\} - \sup_{a \in \cA} \bigg\{(1-\alpha) r(s,a)+(1-\alpha) \gamma\mathbb{E}_{s'\sim P(\cdot\,|\,s,a) }\left[\tilde{V}(s')\right]\bigg\}\\
    &\quad - \inf_{a \in \cA} \bigg\{\alpha r(s,a)+\alpha \gamma\mathbb{E}_{s'\sim P(\cdot\,|\,s,a) }\left[\hat{V}^{\star}(s')\right]\bigg\}\\
    & \ge \inf_{a \in \cA} \bigg\{r(s,a)+\gamma \mathbb{E}_{s'\sim P(\cdot\,|\,s,a) }\left[V(s')\right]-(1-\alpha) r(s,a)-(1-\alpha) \gamma\mathbb{E}_{s'\sim P(\cdot\,|\,s,a) }\left[\tilde{V}(s')\right]\bigg\}\\
    &\quad - \inf_{a \in \cA} \bigg\{\alpha r(s,a)+\alpha \gamma\mathbb{E}_{s'\sim P(\cdot\,|\,s,a) }\left[\hat{V}^{\star}(s')\right]\bigg\}\\
    & \ge \gamma\inf_{a \in \cA} \bigg\{ \mathbb{E}_{s'\sim P(\cdot\,|\,s,a) }\left[V(s')-(1-\alpha) \tilde{V}(s')-\alpha \hat{V}^{\star}(s')\right]\bigg\}.
\end{align*}
Then, if action space is finite, 
\begin{align*}
    T^{\star}V-(1-\alpha)T^{\star}\tilde{V}-\alpha T^{\star}V^{\star} &\ge \gamma\cP^{\hat{\pi}}\para{V-(1-\alpha) \tilde{V}-\alpha \hat{V}^{\star} }
    \\&\ge \gamma\cP^{\hat{\pi}}\bar{V} 
\end{align*}
where $\hat{\pi}$ is the policy satisfying $\hat{\pi}(\cdot \,|\, s) = \argmin_{a \in \cA} \mathbb{E}_{s'\sim P(\cdot\,|\,s,a) }\left[V(s')-(1-\alpha) \tilde{V}(s')-\alpha \hat{V}^{\star}(s')\right]$ and second inequality comes from Lemma \ref{lem::monotonicity}. Thus, we can conclude $\cP_H = \cP^{\pi}$.

Otherwise, if action space is infinite, define $\hat{\cP}(c \bar{V})=c\inf_{s \in \cS} \bar{V}(s) $ for $c\in \real$ and previously given $\bar{V}$. Let $M$ be linear space spanned by $\bar{V}$ with $\infn{\cdot}$-norm. Then, $\hat{\cP}$ is linear functional on $M$ and $\|\hat{\cP}\|_{\text{op}} \le 1$ since $\frac{|c \inf_{s \in \cS}\bar{V}(s)|}{\infn{c\bar{V}}} \le 1$. Due to Hahn--Banach extension Theorem, there exist linear functional $\hat{\cP}_h \colon \cF(\cS) \rightarrow \real$ with $\hat{\cP}_h(\bar{V})=\inf_{s \in \cS}\bar{V}(s)$ and $\|\hat{\cP}_h\|_{\text{op}} \le 1$. Furthermore, we can define $\hat{\cP}_H \colon \cF(\cS) \rightarrow \cF(\cS)$ such that $\hat{\cP}_H V(s)=\hat{\cP}_h(V)$  for all $s \in \cS$. Then $\|\hat{\cP}_H\|_{\infty} \le 1$ since $\|\hat{\cP}_H(V)\|_{\infty}=|\hat{\cP}_h(V)| \le \|\hat{\cP}_h\|_{\text{op}} \le 1$ for $\infn{V} \le 1$.  . Thus, $\hat{\cP}_H$ is nonexpansive linear operator in $\infn{\cdot}$-norm. Then, we have
\begin{align*}
    T^{\star}V(s)-(1-\alpha) T^{\star}\tilde{V}(s)-\alpha \hat{T}^{\star}\hat{V}^{\star}(s)&\ge \gamma\inf_{a \in \cA} \bigg\{ \mathbb{E}_{s'\sim P(\cdot\,|\,s,a) }\left[V(s')-(1-\alpha) \tilde{V}(s')-\alpha \hat{V}^{\star}(s')\right]\bigg\}\\
    &\ge \gamma\inf_{s' \in \cS} \{V(s')-(1-\alpha) \tilde{V}(s')-\alpha \hat{V}^{\star}(s')\}\\
    &\ge \gamma \inf_{s' \in \cS} \{\bar{V}(s')\}
\end{align*}
for all $s \in \cS$. Therefore, we have
\begin{align*}
    \gamma \hat{\cP}_H (\bar{V}) \le T^{\star}V(s)-(1-\alpha) T^{\star}\tilde{V}(s)-\alpha \hat{T}^{\star}\hat{V}^{\star}(s).  
\end{align*}
% Otherwise, if action space is infinite, we have
% \[\gamma\inf_{a \in \cA} \bigg\{ \mathbb{E}_{s'\sim P(\cdot\,|\,s,a) }\left[V(s')-(1-\alpha) \tilde{V}(s')-\alpha V^{\star}(s')\right]\bigg\} \ge \gamma\inf_{a \in \cA} \bigg\{ \mathbb{E}_{s'\sim P(\cdot\,|\,s,a) }\left[\bar{V}(s')\right]\bigg\},\]
% and let $\hat{\cP}^{s}(c \bar{V})=c\inf_{a \in \cA} \left\{ \mathbb{E}_{s'\sim P(\cdot\,|\,s,a) }[\bar{V}] \right\}$ for $c\in \real$, and $M$ be linear space spanned by $\bar{V}$ with $\infn{\cdot}$-norm. Then, $\hat{\cP}^{s}$ is linear functional on $M$ and $\|\hat{\cP}^{s}\|= 1$. Then, with same argument in proof of Lemma \ref{lem::Anc-VIC-1-1-H1}, using Hahn--Banach extension Theorem, we can construct nonexpansive linear operator $\hat{\cP}_H \colon \cF(\cS) \rightarrow \cF(\cS)$ such that 
% \begin{align*}
%     \gamma \hat{\cP}_H (\bar{V}) \le T^{\star}V(s)-(1-\alpha) T^{\star}\tilde{V}(s)-\alpha \hat{T}^{\star}\hat{V}^{\star}(s).  
% \end{align*}

Similarly, let $U=Q, \tilde{U}=\tilde{Q}, \hat{U}^{\star}=\hat{Q}^{\star}, \bar{U}=\bar{Q}$, and $ \bar{Q} \le Q-(1-\alpha) \tilde{Q}-\alpha \hat{Q}^{\star}$. Then,
\begin{align*}
    &T^{\star}Q(s,a)-\alpha T^{\star}\tilde{Q}(s,a)-(1-\alpha)\hat{T}^{\star}\hat{Q}^{\star}(s,a)\\& = r(s,a)+\gamma\mathbb{E}_{s'\sim P(\cdot\,|\,s,a)}\left[\sup_{a' \in \cA} Q(s',a')\right]- (1-\alpha) r(s,a)-(1-\alpha)\gamma\mathbb{E}_{s'\sim P(\cdot\,|\,s,a)}\left[\sup_{a' \in \cA} \tilde{Q}(s',a')\right]\\
    &\quad - \alpha r(s,a)- \alpha \gamma\mathbb{E}_{s'\sim P(\cdot\,|\,s,a)}\left[\inf_{a' \in \cA} \hat{Q}^{\star}(s',a')\right]\\
    & \ge \gamma\mathbb{E}_{s'\sim P(\cdot\,|\,s,a)}\left[\inf_{a' \in \cA} \left\{Q(s',a')-(1-\alpha) \tilde{Q}(s',a')\right\}\right] -\gamma\mathbb{E}_{s'\sim P(\cdot\,|\,s,a)}\left[\inf_{a' \in \cA} \alpha \hat{Q}(s',a')\right]\\
    & \ge \gamma\mathbb{E}_{s'\sim P(\cdot\,|\,s,a)}\left[\inf_{a' \in \cA} \left\{Q(s',a')-(1-\alpha) \tilde{Q}(s',a')-\alpha \hat{Q}^{\star}(s',a')\right\}\right].
\end{align*}
Hence, if action space is finite, 
\begin{align*}
    T^{\star}Q-(1-\alpha)T^{\star}\tilde{Q}-\alpha T^{\star}Q^{\star} &\ge \gamma\cP^{\hat{\pi}}\para{Q-(1-\alpha) \tilde{Q}-\alpha Q^{\star} },\\&\ge \gamma\cP^{\hat{\pi}}\bar{Q},
\end{align*}
where $\hat{\pi}$ is the policy satisfying $\hat{\pi}(\cdot \,|\, s) = \argmin_{a \in \cA} \mathbb{E}_{s'\sim P(\cdot\,|\,s,a) }\left[Q(s')-(1-\alpha) \tilde{Q}(s')-\alpha Q^{\star}(s')\right]$ and second inequality comes from Lemma \ref{lem::monotonicity}.  Then, we can conclude $\cP_H = \cP^{\hat{\pi}}$.

Otherwise, if action space is infinite, define $\hat{\cP}(c \bar{Q})= c \inf_{(s',a') \in \cS \times \cA}  \bar{Q}(s',a')$ for $c \in \real^n$ and previously given $\bar{Q}$. Let $M$ be linear space spanned by $\bar{Q}$ with $\infn{\cdot}$-norm. Then, $\cP$ is linear functional on $M$ with $\|\hat{\cP}\|_{\text{op}} \le 1$. Due to Hahn--Banach extension Theorem, there exist linear functional $\hat{\cP}_h\colon \cF(\cS \times \cA) \rightarrow \real $ with $\hat{\cP}_h(\bar{Q})= \inf_{(s',a') \in \cS \times \cA} \bar{Q}(s',a')$ and $\|\hat{\cP}_h\|_{\text{op}} \le 1$.  Furthermore, we can define $\hat{\cP}_H \colon \cF(\cS \times \cA) \rightarrow \cF(\cS \times \cA)$ such that $\cP_H Q(s,a)=\hat{\cP}_h(Q)$ for all $(s,a) \in \cS \times \cA$ and $\|\hat{P}_H\|_{\infty} \le 1$. Thus $\hat{\cP}_H$ is nonexpansive linear operator in $\infn{\cdot}$-norm. Then, we have 
\begin{align*}
 & T^{\star}Q(s,a)-\alpha T^{\star}\tilde{Q}(s,a)-(1-\alpha)\hat{T}^{\star}\hat{Q}^{\star}(s,a)\\& \ge \gamma\mathbb{E}_{s'\sim P(\cdot\,|\,s,a)}\left[\inf_{a' \in \cA} \left\{Q(s',a')-(1-\alpha) \tilde{Q}(s',a')-\alpha \hat{Q}^{\star}(s',a')\right\}\right] \\&\ge \gamma\inf_{(s', a') \in \cS \times \cA} \left\{Q(s',a')-(1-\alpha) \tilde{Q}(s',a')-\alpha \hat{Q}^{\star}(s',a')\right\}\\&\ge \gamma \inf_{(s', a') \in \cS \times \cA} \bar{Q}(s',a'), 
 \end{align*}
for all $(s,a) \in \cS \times \cA$. Therefore, we have 
\begin{align*}
    \gamma \hat{\cP}_H (\bar{Q}) \le T^{\star}Q-(1-\alpha) T^{\star}\tilde{Q}-\alpha \hat{T}^{\star}\hat{Q}^{\star}.
\end{align*}
% Otherwise, if action space is infinite, we have
% \[ \gamma\mathbb{E}_{s'\sim P(\cdot\,|\,s,a)}\left[\inf_{a' \in \cA} \left\{Q(s',a')-(1-\alpha) \tilde{Q}(s',a')-\alpha \hat{Q}^{\star}(s',a')\right\}\right] \ge \gamma\mathbb{E}_{s'\sim P(\cdot\,|\,s,a)}\left[\inf_{a' \in \cA} \left\{\bar{Q}(s',a')\right\}\right], \]
% and let $\cP^{s,a}(c \bar{Q})= c\mathbb{E}_{s'\sim P(\cdot\,|\,s,a) }[\inf_{a' \in \cA} \bar{Q}(s',a')] $ for $c \in \real^n$,  and $M$ be linear space spanned by $\bar{Q}$ with $\infn{\cdot}$-norm. Then, $\cP^{s,a}$ is linear functional on $M$ and $\|\cP^{s,a}\|= 1$. Then, with same argument in proof of Lemma \ref{lem::Anc-VIC-1-1-H1}, using Hahn--Banach extension Theorem, we can construct nonexpansive linear operator $\hat{\cP}_H \colon \cF(\cS \times \cA) \rightarrow \cF(\cS \times \cA) $ such that 
% \begin{align*}
%     \gamma \hat{\cP}_H (\bar{Q}) \le T^{\star}Q-(1-\alpha) T^{\star}\tilde{Q}-\alpha \hat{T}^{\star}\hat{Q}^{\star}.
% \end{align*}
\end{proof}
Now, we present our key lemmas for the first rate of Theorem \ref{thm::Anc-VIC}. 
\begin{lemma}\label{lem::Anc-VIC-1-1}
Let $0<\gamma\le 1$. If $\gamma=1$, assume a fixed point $U^{\star}$ exists. For the iterates $\{U^k\}_{k=0,1,\dots}$ of \ref{eq:anc-vi}, there exist nonexpansive linear operators $\{\cP^l\}_{l=0,1,\dots,k}$ such that
\begin{align*}
    T^{\star}U^k-U^k &\le \sum_{i=1}^k \left[(\beta_i-\beta_{i-1}(1-\beta_{i}))\para{\Pi^k_{j=i+1}(1-\beta_j)}\para{\Pi^i_{l=k}\gamma\cP^l}(U^0-U^{\star})\right]
    \\&\quad -\beta_k(U^0-U^{\star})+\para{\Pi^k_{j=1}(1-\beta_j)} \para{\Pi^0_{l=k}\gamma\cP^l}(U^0-U^{\star})
\end{align*}
where $\Pi^k_{j=k+1}(1-\beta_j)=1$ and $\beta_0=1$.
\end{lemma}
\begin{lemma}\label{lem::Anc-VIC-1-2}
Let $0<\gamma< 1$. For the iterates $\{U^k\}_{k=0,1,\dots}$ of \ref{eq:anc-vi}, there exist nonexpansive linear operators $\{\hat{\cP}^l\}_{l=0,1,\dots,k}$ such that
\begin{align*}
     T^{\star}U^k-U^k &  \ge \sum_{i=1}^k \left[(\beta_i-\beta_{i-1}(1-\beta_{i}))\para{\Pi^k_{j=i+1}(1-\beta_j)}\para{\Pi^i_{l=k}\gamma\hat{\cP}^l}(U^0-\hat{U}^{\star})\right]
    \\&\quad -\beta_k(U^0-\hat{U}^{\star})+\para{\Pi^k_{j=1}(1-\beta_j)} \para{\Pi^0_{l=k}\gamma\hat{\cP}^l}(U^0-\hat{U}^{\star}), 
\end{align*}
where $\Pi^k_{j=k+1}(1-\beta_j)=1$ and $\beta_0=1$.
\end{lemma}
We prove previous lemmas by induction.
\begin{proof}[Proof of Lemma \ref{lem::Anc-VIC-1-1}]
    If $k=0$, 
    \begin{align*}
    T^{\star}U^0-U^0&=T^{\star}U^0-U^{\star}-(U^0-U^{\star})
    \\&=T^{\star}U^0-T^{\star}U^{\star}-(U^0-U^{\star})
    \\& \le \gamma\cP^{0}(U^0-U^{\star})-(U^0-U^{\star}).
\end{align*}
where inequality comes from first inequality in Lemma \ref{lem::Anc-VIC-1-1-H1} with $\alpha=1, U=U^0, \bar{U}=U^0-U^{\star}$. 

By induction, 
\begin{align*}
    &U^k-(1-\beta_k)U^{k-1}-\beta_k U^{\star}\\&=\beta_k \para{U^{0}-U^{\star}}+(1-\beta_k)(T^{\star}U^{k-1}-U^{k-1})\\
    &\le  (1-\beta_k) \sum_{i=1}^{k-1} \left[(\beta_i-\beta_{i-1}(1-\beta_{i}))\para{\Pi^{k-1}_{j=i+1}(1-\beta_j)}\para{\Pi^{i}_{l=k-1}\gamma\cP^l}(U^0-U^{\star})\right]\\
    &\quad-(1-\beta_k)\beta_{k-1}(U^0-U^{\star})+(1-\beta_k)\para{\Pi^{k-1}_{j=1}(1-\beta_j)} \para{\Pi^{0}_{l=k-1}\gamma\cP^l}(U^0-U^{\star})\\& \quad +\beta_k(U^0-U^{\star}),
\end{align*}
and let $\bar{U}$ be the entire right hand side of inequality. Then, we have
    \begin{align*}
    &T^{\star}U^k-U^k\\
    &=T^{\star}U^k-(1-\beta_k)T^{\star}U^{k-1}-\beta_k U^0\\
     &=T^{\star}U^k-(1-\beta_k)T^{\star}U^{k-1}-\beta_kU^{\star}-\beta_k(U^0-U^{\star})\\
    &=T^{\star}U^k-(1-\beta_k)T^{\star}U^{k-1}-\beta_kT^{\star}U^{\star}-\beta_k(U^0-U^{\star})\\
     &\le \gamma \cP^k \bigg( (1-\beta_k) \sum_{i=1}^{k-1} \left[(\beta_i-\beta_{i-1}(1-\beta_{i}))\para{\Pi^{k-1}_{j=i+1}(1-\beta_j)}\para{\Pi^{i}_{l=k-1}\gamma\cP^l}(U^0-U^{\star})\right]\\
    &\quad-(1-\beta_k)\beta_{k-1}(U^0-U^{\star})+(1-\beta_k)\para{\Pi^{k-1}_{j=1}(1-\beta_j)} \para{\Pi^{0}_{l=k-1}\gamma\cP^l}(U^0-U^{\star})\\& \quad +\beta_k(U^0-U^{\star})\bigg)-\beta_k(U^0-U^{\star})\\
    % & = (1-\beta_k)\gamma\cP^k\sum_{i=1}^{k-1} \left[(\beta_i-\beta_{i-1}(1-\beta_{i}))\para{\Pi^{k-1}_{j=i+1}(1-\beta_j)}\para{\Pi^{i}_{l=k-1}\gamma\cP^l}(U^0-U^{\star})\right]
    % \\&\quad -\beta_{k-1}(1-\beta_k)\gamma\cP^k(U^0-U^{\star})+(1-\beta_k)\gamma\cP^k\para{\Pi^{k-1}_{j=1}(1-\beta_j)} \para{\Pi^{0}_{l=k-1}\gamma\cP^l}(U^0-U^{\star})\\
    % &\quad +\beta_k\gamma\cP^k\para{U^{0}-U^{\star}}-\beta_k(U^0-U^{\star})\\
    & = \sum_{i=1}^{k-1} \left[(\beta_i-\beta_{i-1}(1-\beta_{i}))\para{\Pi^{k}_{j=i+1}(1-\beta_j)}\para{\Pi^{i}_{l=k}\gamma\cP^l}(U^0-U^{\star})\right]
    \\&\quad -\beta_{k-1}(1-\beta_k)\gamma\cP^k(U^0-U^{\star})+\beta_k\gamma\cP^k\para{U^{0}-U^{\star}}\\
    &\quad -\beta_k(U^0-U^{\star})+\para{\Pi^k_{j=1}(1-\beta_j)} \para{\Pi^0_{l=k}\gamma\cP^l}(U^0-U^{\star})\\
    & = \sum_{i=1}^k \left[(\beta_i-\beta_{i-1}(1-\beta_{i}))\para{\Pi^k_{j=i+1}(1-\beta_j)}\para{\Pi^i_{l=k}\gamma\cP^l}(U^0-U^{\star})\right]
    \\&\quad -\beta_k(U^0-U^{\star})+\para{\Pi^k_{j=1}(1-\beta_j)} \para{\Pi^0_{l=k}\gamma\cP^l}(U^0-U^{\star}).
\end{align*}
where inequality comes from first inequality in Lemma \ref{lem::Anc-VIC-1-1-H1} with $\alpha=\beta_k, U=U^k, \tilde{U}=U^{k-1}$, and previously defined $\bar{U}$. 
\end{proof}
\begin{proof}[Proof of Lemma \ref{lem::Anc-VIC-1-2}]
Note that  $\hat{T}^{\star}$ is Bellman anti-optimality operators for $V$ or $Q$, and $\hat{U}^\star$ is the fixed point of $\hat{T}^{\star}$. If $k=0$, 
    \begin{align*}
    T^{\star}U^0-U^0&=T^{\star}U^0-\hat{U}^{\star}-(U^0-\hat{U}^{\star})
    \\&=T^{\star}U^0-\hat{T}^{\star}\hat{U}^{\star}-(U^0-\hat{U}^{\star})
    \\& \ge \gamma\hat{\cP}^{0}(U^0-\hat{U}^{\star})-(U^0-\hat{U}^{\star}).
\end{align*}
where inequality comes from second inequality in Lemma \ref{lem::Anc-VIC-1-1-H2} with $\alpha=1, U=U^0, \bar{U}=U^0-\hat{U}^{\star}$.

By induction,
\begin{align*}
    &U^k-(1-\beta_k)U^{k-1}-\beta_k \hat{U}^{\star}\\&=\beta_k (U^{0}-\hat{U}^{\star})+(1-\beta_k)(T^{\star}U^{k-1}-U^{k-1})\\&\ge  (1-\beta_k) \sum_{i=1}^{k-1} \left[(\beta_i-\beta_{i-1}(1-\beta_{i}))\para{\Pi^{k-1}_{j=i+1}(1-\beta_j)}\para{\Pi^{i}_{l=k-1}\gamma\hat{\cP}^l}(U^0-\hat{U}^{\star})\right]\\
    &\quad-(1-\beta_k)\beta_{k-1}(U^0-\hat{U}^{\star})+(1-\beta_k)\para{\Pi^{k-1}_{j=1}(1-\beta_j)} \para{\Pi^{0}_{l=k-1}\gamma\hat{\cP}^l}(U^0-\hat{U}^{\star})\\& \quad +\beta_k(U^0-\hat{U}^{\star}), 
\end{align*}
and let $\bar{U}$ be the entire right hand side of inequality. Then, we have

    \begin{align*}
    &T^{\star}U^k-U^k\\
    &=T^{\star}U^k-(1-\beta_k)T^{\star}U^{k-1}-\beta_k U^0\\
     &=T^{\star}U^k-(1-\beta_k)T^{\star}U^{k-1}-\beta_k\hat{U}^{\star}-\beta_k(U^0-\hat{U}^{\star})\\
    &=T^{\star}U^k-(1-\beta_k)T^{\star}U^{k-1}-\beta_k\hat{T}^{\star}\hat{U}^{\star}-\beta_k(U^0-\hat{U}^{\star})\\
     &\ge \gamma \hat{\cP}^k \bigg( (1-\beta_k) \sum_{i=1}^{k-1} \left[(\beta_i-\beta_{i-1}(1-\beta_{i}))\para{\Pi^{k-1}_{j=i+1}(1-\beta_j)}\para{\Pi^{i}_{l=k-1}\gamma\hat{\cP}^l}(U^0-\hat{U}^{\star})\right]\\
    &\quad-(1-\beta_k)\beta_{k-1}(U^0-\hat{U}^{\star})+(1-\beta_k)\para{\Pi^{k-1}_{j=1}(1-\beta_j)} \para{\Pi^{0}_{l=k-1}\gamma\hat{\cP}^l}(U^0-\hat{U}^{\star})\\& \quad +\beta_k(U^0-\hat{U}^{\star})\bigg)-\beta_k(U^0-\hat{U}^{\star})\\
    % & = (1-\beta_k)\gamma\hat{\cP}^k\sum_{i=1}^{k-1} \left[(\beta_i-\beta_{i-1}(1-\beta_{i}))\para{\Pi^{k-1}_{j=i+1}(1-\beta_j)}\para{\Pi^{i}_{l=k-1}\gamma\hat{\cP}^l}(U^0-\hat{U}^{\star})\right]
    % \\&\quad -\beta_{k-1}(1-\beta_k)\gamma\hat{\cP}^k(U^0-\hat{U}^{\star})+(1-\beta_k)\gamma\hat{\cP}^k\para{\Pi^{k-1}_{j=1}(1-\beta_j)} \para{\Pi^{0}_{l=k-1}\gamma\hat{\cP}^l}(U^0-\hat{U}^{\star})\\
    % &\quad +\beta_k\gamma\hat{\cP}^k\para{U^{0}-\hat{U}^{\star}}-\beta_k(U^0-\hat{U}^{\star})\\
    & = \sum_{i=1}^{k-1} \left[(\beta_i-\beta_{i-1}(1-\beta_{i}))\para{\Pi^{k}_{j=i+1}(1-\beta_j)}\para{\Pi^{i}_{l=k}\gamma\hat{\cP}^l}(U^0-\hat{U}^{\star})\right]
    \\&\quad -\beta_{k-1}(1-\beta_k)\gamma\hat{\cP}^k(U^0-\hat{U}^{\star})+\beta_k\gamma\hat{\cP}^k\para{U^{0}-\hat{U}^{\star}}\\
    &\quad -\beta_k(U^0-\hat{U}^{\star})+\para{\Pi^k_{j=1}(1-\beta_j)} \para{\Pi^0_{l=k}\gamma\hat{\cP}^l}(U^0-\hat{U}^{\star})\\
    & = \sum_{i=1}^k \left[(\beta_i-\beta_{i-1}(1-\beta_{i}))\para{\Pi^k_{j=i+1}(1-\beta_j)}\para{\Pi^i_{l=k}\gamma\hat{\cP}^l}(U^0-\hat{U}^{\star})\right]
    \\&\quad -\beta_k(U^0-\hat{U}^{\star})+\para{\Pi^k_{j=1}(1-\beta_j)} \para{\Pi^0_{l=k}\gamma\hat{\cP}^l}(U^0-\hat{U}^{\star}).
\end{align*}
where inequality comes from second inequality in Lemma \ref{lem::Anc-VIC-1-1-H2} with $\alpha=\beta_k, U=U^k, \tilde{U}=U^{k-1}$, and previously defined $\bar{U}$.
\end{proof}
Now, we prove the first rate of Theorem \ref{thm::Anc-VIC}.

\begin{proof}[Proof of first rate in Theorem \ref{thm::Anc-VIC}] 
Since $B_1 \le A \le B_2$ implies $ \infn{A} \le \sup\{\infn{B_1}, \infn{B_2}\}$ for $A,B \in \cF(\cX)$, if we take $\infn{\cdot}$ right side first inequality of Lemma \ref{lem::Anc-VIC-1-1},  we have 
\begin{align*}
    & \sum_{i=1}^k \abs{\beta_i-\beta_{i-1}(1-\beta_{i})}\para{\Pi^k_{j=i+1}(1-\beta_j)}\infn{\para{\Pi^i_{l=k}\gamma\cP^{l}}(U^0-U^{\star})}
    \\&\quad +\beta_k\infn{U^0-U^{\pi}}+\para{\Pi^k_{j=1}(1-\beta_j)} \infn{\para{\Pi^0_{l=k}\gamma\cP^{l}}(U^0-U^{\star})}\\&\le \para{\sum_{i=1}^k \gamma^{k-i+1}\abs{\beta_i-\beta_{i-1}(1-\beta_{i})}\para{\Pi^k_{j=i+1}(1-\beta_j)}+\beta_k+\gamma^{k+1}\Pi^k_{j=1}(1-\beta_j)}\\&\quad\infn{U^0-U^{\star}}\\&= \frac{\para{\gamma^{-1}-\gamma}\para{1+2\gamma-\gamma^{k+1}}}{\para{\gamma^{k+1}}^{-1}-\gamma^{k+1}}\infn{U^0-U^{\star}},
\end{align*}
where the first inequality comes from triangular inequality, second inequality is from 
nonexpansiveness of $\cP^l$, and last equality comes from calculations.

If we take $\infn{\cdot}$ right side of second inequality of Lemma \ref{lem::Anc-VIC-1-1}, similarly, we have
\begin{align*}
 &\sum_{i=1}^k \abs{\beta_i-\beta_{i-1}(1-\beta_{i})}\para{\Pi^k_{j=i+1}(1-\beta_j)}\infn{\para{\Pi^i_{l=k}\gamma \hat{\cP}^{l}}(U^0-\hat{U}^{\star})}
    \\&\quad +\beta_k\infn{U^0-U^{\pi}}+\para{\Pi^k_{j=1}(1-\beta_j)} \infn{\para{\Pi^0_{l=k}\gamma\hat{\cP}^{l}}(U^0-\hat{U}^{\star})}\\&\le \para{\sum_{i=1}^k \gamma^{k-i+1}\abs{\beta_i-\beta_{i-1}(1-\beta_{i})}\para{\Pi^k_{j=i+1}(1-\beta_j)}+\beta_k+\gamma^{k+1}\Pi^k_{j=1}(1-\beta_j)}\\&= \frac{\para{\gamma^{-1}-\gamma}\para{1+2\gamma-\gamma^{k+1}}}{\para{\gamma^{k+1}}^{-1}-\gamma^{k+1}}\infn{U^0-\hat{U}^{\star}},
\end{align*}
where the first inequality comes from triangular inequality, second inequality is from 
from 
nonexpansiveness of $\hat{\cP}^l$, and last equality comes from calculations. Therefore, we conclude  
\begin{align*}
    \infn{T^{\star}U^k-U^k} &\le  \frac{\para{\gamma^{-1}-\gamma}\para{1+2\gamma-\gamma^{k+1}}}{\para{\gamma^{k+1}}^{-1}-\gamma^{k+1}}\max{\left\{\infn{U^0-U^{\star}},\infn{U^0-\hat{U}^\star}\right\}}.
\end{align*}
\end{proof}
     Next, for the second rate in Theorem \ref{thm::Anc-VIC}, we prove following lemmas by induction.
    \begin{lemma}\label{lem::Anc-VIC-2-1}
Let $0<\gamma\le1$. If $\gamma=1$, assume a fixed point $U^{\star}$ exists. For the iterates $\{U^k\}_{k=0,1,\dots}$ of \ref{eq:anc-vi}, if $T^{\star}U^0 \le U^{\star}$, there exist nonexpansive linear operators $\{\cP^l\}_{l=0,1,\dots,k}$ such that
\begin{align*}
    T^{\star}U^k-U^k &\le \sum_{i=1}^k \left[(\beta_i-\beta_{i-1}(1-\beta_{i}))\para{\Pi^k_{j=i+1}(1-\beta_j)}\para{\Pi^i_{l=k}\gamma\cP^l}(U^0-U^{\star})\right] -\beta_k(U^0-U^{\star}) 
\end{align*}
where $\Pi^k_{j=k+1}(1-\beta_j)=1$ and $\beta_0=1$. 
\end{lemma}    
\begin{lemma}\label{lem::Anc-VIC-2-2}
Let $0<\gamma<1$. For the iterates $\{U^k\}_{k=0,1,\dots}$ of \ref{eq:anc-vi}, if $U^0 \ge T^{\star}U^0 $, there exist nonexpansive linear operators $\{\hat{\cP}^l\}_{l=0,1,\dots,k}$ such that
\begin{align*}
     T^{\star}U^k-U^k &  \ge \sum_{i=1}^k \left[(\beta_i-\beta_{i-1}(1-\beta_{i}))\para{\Pi^k_{j=i+1}(1-\beta_j)}\para{\Pi^i_{l=k}\gamma\hat{\cP}^l}(U^0-\hat{U}^{\star})\right] -\beta_k(U^0-\hat{U}^{\star}), 
\end{align*}
where $\Pi^k_{j=k+1}(1-\beta_j)=1$ and $\beta_0=1$.
\end{lemma}     
\begin{proof}[Proof of Lemma \ref{lem::Anc-VIC-2-1}]
    %Proof is basically same with proof of Lemma \ref{lem::Anc-VIC-1-1} except first step of induction.  
    If $k=0$, 
    \begin{align*}
    T^{\star}U^0-U^0&=T^{\star}U^0-U^{\star}-(U^0-U^{\star})
 \\& \le -(U^0-U^{\star})
\end{align*}
 where the second inequality is from the condition.
  
By induction, 
\begin{align*}
    &U^k-(1-\beta_k)U^{k-1}-\beta_k U^{\star}\\
    &\le  (1-\beta_k) \sum_{i=1}^{k-1} \left[(\beta_i-\beta_{i-1}(1-\beta_{i}))\para{\Pi^{k-1}_{j=i+1}(1-\beta_j)}\para{\Pi^{i}_{l=k-1}\gamma\cP^l}(U^0-U^{\star})\right]\\
    &\quad-(1-\beta_k)\beta_{k-1}(U^0-U^{\star})+\beta_k(U^0-U^{\star}),
\end{align*}
and let $\bar{U}$ be the entire right hand side of inequality. Then, we have
    \begin{align*}
    &T^{\star}U^k-U^k\\
    &=T^{\star}U^k-(1-\beta_k)T^{\star}U^{k-1}-\beta_kT^{\star}U^{\star}-\beta_k(U^0-U^{\star})\\
     &\le \gamma \cP^k \bigg( (1-\beta_k) \sum_{i=1}^{k-1} \left[(\beta_i-\beta_{i-1}(1-\beta_{i}))\para{\Pi^{k-1}_{j=i+1}(1-\beta_j)}\para{\Pi^{i}_{l=k-1}\gamma\cP^l}(U^0-U^{\star})\right]\\
    &\quad-(1-\beta_k)\beta_{k-1}(U^0-U^{\star})+\beta_k(U^0-U^{\star})\bigg)-\beta_k(U^0-U^{\star})\\
    % & = (1-\beta_k)\gamma\cP^k\sum_{i=1}^{k-1} \left[(\beta_i-\beta_{i-1}(1-\beta_{i}))\para{\Pi^{k-1}_{j=i+1}(1-\beta_j)}\para{\Pi^{i}_{l=k-1}\gamma\cP^l}(U^0-U^{\star})\right]
    % \\&\quad -\beta_{k-1}(1-\beta_k)\gamma\cP^k(U^0-U^{\star})+(1-\beta_k)\gamma\cP^k\para{\Pi^{k-1}_{j=1}(1-\beta_j)} \para{\Pi^{0}_{l=k-1}\gamma\cP^l}(U^0-U^{\star})\\
    % &\quad +\beta_k\gamma\cP^k\para{U^{0}-U^{\star}}-\beta_k(U^0-U^{\star})\\
    & = \sum_{i=1}^k \left[(\beta_i-\beta_{i-1}(1-\beta_{i}))\para{\Pi^k_{j=i+1}(1-\beta_j)}\para{\Pi^i_{l=k}\gamma\cP^l}(U^0-U^{\star})\right] -\beta_k(U^0-U^{\star}),
\end{align*}
where inequality comes from  first inequality in Lemma \ref{lem::Anc-VIC-1-1-H1} with $\alpha=\beta_k, U=U^k, \tilde{U}=U^{k-1}$, and previously defined $\bar{U}$. 
\end{proof}
\begin{proof}[Proof of Lemma \ref{lem::Anc-VIC-2-2}]
If $k=0$, 
    \begin{align*}
    T^{\star}U^0-U^0&=T^{\star}U^0-\hat{U}^{\star}-(U^0-\hat{U}^{\star})
    \\& \ge -(U^0-\hat{U}^{\star}).
\end{align*}
where the second inequality is from the fact that $U^0 \ge T^{\star} U^0$ implies $T^{\star}U^0 \ge U^\star$ by Lemma \ref{lem::Inital_cond} and $ U^\star\ge\hat{U}^{\star} $ by Lemma \ref{lem::anti-fixed-point}.

By induction,
\begin{align*}
    &U^k-(1-\beta_k)U^{k-1}-\beta_k \hat{U}^{\star}\\&\ge  (1-\beta_k) \sum_{i=1}^{k-1} \left[(\beta_i-\beta_{i-1}(1-\beta_{i}))\para{\Pi^{k-1}_{j=i+1}(1-\beta_j)}\para{\Pi^{i}_{l=k-1}\gamma\hat{\cP}^l}(U^0-\hat{U}^{\star})\right]\\
    &\quad-(1-\beta_k)\beta_{k-1}(U^0-\hat{U}^{\star})+\beta_k(U^0-\hat{U}^{\star}), 
\end{align*}
and let $\bar{U}$ be the entire right hand side of inequality. Then, we have

    \begin{align*}
    &T^{\star}U^k-U^k\\
    &=T^{\star}U^k-(1-\beta_k)T^{\star}U^{k-1}-\beta_k\hat{T}^{\star}\hat{U}^{\star}-\beta_k(U^0-\hat{U}^{\star})\\
     &\ge \gamma \hat{\cP}^k \bigg( (1-\beta_k) \sum_{i=1}^{k-1} \left[(\beta_i-\beta_{i-1}(1-\beta_{i}))\para{\Pi^{k-1}_{j=i+1}(1-\beta_j)}\para{\Pi^{i}_{l=k-1}\gamma\hat{\cP}^l}(U^0-\hat{U}^{\star})\right]\\
    &\quad-(1-\beta_k)\beta_{k-1}(U^0-\hat{U}^{\star})+\beta_k(U^0-\hat{U}^{\star})\bigg)-\beta_k(U^0-\hat{U}^{\star})\\
    % & = (1-\beta_k)\gamma\hat{\cP}^k\sum_{i=1}^{k-1} \left[(\beta_i-\beta_{i-1}(1-\beta_{i}))\para{\Pi^{k-1}_{j=i+1}(1-\beta_j)}\para{\Pi^{i}_{l=k-1}\gamma\hat{\cP}^l}(U^0-\hat{U}^{\star})\right]
    % \\&\quad -\beta_{k-1}(1-\beta_k)\gamma\hat{\cP}^k(U^0-\hat{U}^{\star})+(1-\beta_k)\gamma\hat{\cP}^k\para{\Pi^{k-1}_{j=1}(1-\beta_j)} \para{\Pi^{0}_{l=k-1}\gamma\hat{\cP}^l}(U^0-\hat{U}^{\star})\\
    % &\quad +\beta_k\gamma\hat{\cP}^k\para{U^{0}-\hat{U}^{\star}}-\beta_k(U^0-\hat{U}^{\star})\\
    & = \sum_{i=1}^k \left[(\beta_i-\beta_{i-1}(1-\beta_{i}))\para{\Pi^k_{j=i+1}(1-\beta_j)}\para{\Pi^i_{l=k}\gamma\hat{\cP}^l}(U^0-\hat{U}^{\star})\right]-\beta_k(U^0-\hat{U}^{\star}),
\end{align*}
where inequality comes from second inequality in  Lemma \ref{lem::Anc-VIC-1-1-H2} with $\alpha=\beta_k, U=U^k, \tilde{U}=U^{k-1}$, and previously defined $\bar{U}$.
\end{proof}
Now, we prove the second rates of Theorem \ref{thm::Anc-VIC}.

\begin{proof}[Proof of second rates in Theorem \ref{thm::Anc-VIC}] 
     Let $0<\gamma<1$. Then, if $U^0 \le T^{\star}U^0$, then $T^{\star}U^0 \le U^{\star}$ and $U^k \le T^{\star}U^k$ by Lemma \ref{lem::Inital_cond}. Hence, taking $\infn{\cdot}$-norm both sides of first inequality in Lemma \ref{lem::Anc-VIC-2-1}, we have
% \begin{align*}
%     TU^k-U^k &\le \sum_{i=1}^k \left[(\beta_i-\beta_{i-1}(1-\beta_{i}))\para{\Pi^k_{j=i+1}(1-\beta_j)}\para{\Pi^i_{l=k}\gamma\cP^{\pi_l}}(U^0-U^{\star})\right]
%     \\&\quad -\beta_k(U^0-U^{\star})+\para{\Pi^k_{j=1}(1-\beta_j)} \para{\Pi^0_{l=k}\gamma\cP^{\pi_l}}(U^0-U^{\star})\\
%     &\le \sum_{i=1}^k \left[(\beta_i-\beta_{i-1}(1-\beta_{i}))\para{\Pi^k_{j=i+1}(1-\beta_j)}\para{\Pi^i_{l=k}\gamma\cP^{\pi_l}}(U^0-U^{\star})\right] -\beta_k(U^0-U^{\star})
% \end{align*}
%  Taking $l_\infty$ both sides, we get 
\[\infn{T^{\star}U^k-U^k} \le  \frac{\para{\gamma^{-1}-\gamma}\para{1+\gamma-\gamma^{k+1}}}{\para{\gamma^{k+1}}^{-1}-\gamma^{k+1}}\infn{U^0-U^{\star}}.\]
Otherwise, if $U^0 \ge TU^0$, $U^k \ge TU^k$ by Lemma \ref{lem::Inital_cond}. taking $\infn{\cdot}$-norm both sides of second inequality in Lemma \ref{lem::Anc-VIC-2-2}, we have  
\[\infn{T^{\star}U^k-U^k} \le  \frac{\para{\gamma^{-1}-\gamma}\para{1+\gamma-\gamma^{k+1}}}{\para{\gamma^{k+1}}^{-1}-\gamma^{k+1}}\infn{U^0-\hat{U}^{\star}}.\]    
\end{proof}

\section{Omitted proofs in Section \ref{sec::nonexp} }\label{s::omitted-nonexp-proofs}
First, we present the following lemma.
\begin{lemma}\label{lem::boundedness_Anc-VI}
Let $\gamma =1$. Assume a fixed point $U^{\star}$ exists. For the iterates $\{U^k\}_{k=0,1,\dots}$ of \ref{eq:anc-vi}, $\infn{U^{k}-U^{\star}} \le \infn{U^{0}-U^{\star}} $.
\end{lemma}
\begin{proof}
If $k=0$, it is obvious. By induction,
    \begin{align*}
        \infn{U^{k}-U^{\star}} &=\infn{\beta_{k} U^{0}+(1-\beta_{k})TU^{k-1}-U^{\star}}\\&=\infn{(1-\beta_{k})(TU^{k-1}-U^{\star})+\beta_k\para{U^0-U^{\star}}}\\& \le (1-\beta_{k})\infn{TU^{k-1}-U^{\star}}+\beta_{k}\infn{U^0-U^{\star}}\\ &\le (1-\beta_{k})\infn{U^{k-1}-U^{\star}}+\beta_{k}\infn{U^0-U^{\star}}\\& =\infn{U^0-U^{\star}}
    \end{align*}
    where the second inequality comes form nonexpansiveness of $T$.
\end{proof}

Now, we present the proof of Theorem \ref{thm::nonexp_finite}.
\begin{proof}[Proof of Theorem \ref{thm::nonexp_finite}]
     %Under assumption of existence of fixed points ${U}^{\pi}$, we note that Lemma \ref{lem::Anc-VIE-1} holds for $\gamma=1$ with same argument in its proof. Then, by plugging $\gamma=1$ in first rate of Theorem \ref{thm::Anc-VIE}, we get \begin{align*}     \infn{T^{\pi}U^k-U^k} \le \frac{2}{k+1}\infn{U^0-U^{\pi}}. \end{align*} for any fixed point $U^{\pi}$. 

First, if $U^0 \le T U^0 $, with same argument in proof of Lemma \ref{lem::Inital_cond}, we can show that $U^{k-1} \le U^k \le TU^{k-1} \le TU^k$ for $k=1,2,\dots.$ 

Since fixed point $U^{\star}$ exists by assumption, Lemma \ref{lem::Anc-VIE-1} and \ref{lem::Anc-VIC-1-1} hold. Note that $\gamma=1$ implies $\beta_k=\frac{1}{k+1}$ and if we take $\infn{\cdot}$-norm both sides for those inequalities in lemmas, by simple calculation, we have 
     \begin{align*}
     \infn{TU^k-U^k} \le \frac{2}{k+1}\infn{U^0-U^{\star}}
 \end{align*} 
for any fixed point $U^{\star}$ (since $0 \le T^{\star}U^k-U^k$, we can get upper bound of $\infn{T^{\star}U^k-U^k}$ from Lemma \ref{lem::Anc-VIC-1-1}).

Suppose that there exist $\{k_j\}_{j=0,1,\dots}$ such that $U^{k_j}$ converges to some $\tilde{U}^{\star}$. Then, $\lim_{j\rightarrow \infty} (T-I)U^{k_j}=(T-I)\tilde{U}^{\star}=0$ since $T-I$ is continuous. This implies that $\tilde{U}^{\star}$ is a fixed point. By Lemma \ref{lem::boundedness_Anc-VI} and previous argument, $U^k$ is increasing and bounded sequence in $\real^n$. Thus, $U^k$ has single limit point, some fixed point $\tilde{U}^{\star}$. Furthermore, the fact that $U^{0} \le TU^{0} \le\tilde{U}^{\star}$ implies that Lemma \ref{lem::Anc-VIE-2-1} and \ref{lem::Anc-VIC-2-1} hold. Therefore, we have
     \begin{align*}
     \infn{TU^k-U^k} \le \frac{1}{k+1}\infn{U^0-\tilde{U}^{\star}}.
 \end{align*}  
\end{proof}

Next, we prove the Theorem \ref{thm::nonexp_infinite}.
\begin{proof}[Proof of Theorem \ref{thm::nonexp_infinite}]
By same argument in the proof of Theorem \ref{thm::nonexp_finite}, if $U^0 \le T U^0 $, we can show that $U^{k-1} \le U^k \le TU^{k-1} \le TU^k$ for $k=1,2,\dots.$, and 
     \begin{align*}
     \infn{T U^k-U^k} \le \frac{2}{k+1}\infn{U^0-U^{\star}}
 \end{align*} 
 for any fixed point $U^{\star}$. Since $U^k$ is increasing and bounded by Lemma \ref{lem::boundedness_Anc-VI} and previous argument, $U^k$ converges pointwise to some $\tilde{U}^{\star}$ in general action-state space. We now show that $TU^k$ also converges pointwise to $T\tilde{U}^{\star}$. First, let $T$ be Bellman consistency operator and $U=V, \tilde{U}^{\star}=\tilde{V}^{\pi}$. By monontone convergence theorem, 
  \begin{align*}
      \lim_{k\rightarrow \infty} T^{\pi}V^k(s)&=\lim_{k\rightarrow \infty}\mathbb{E}_{a \sim \pi(\cdot\,|\,s) }\left[\mathbb{E}_{ s'\sim P(\cdot\,|\,s,a) }\left[r(s,a)+\gamma V^k(s')\right]\right]\\
       &=\mathbb{E}_{a \sim \pi(\cdot\,|\,s) }\left[\lim_{k\rightarrow \infty}\mathbb{E}_{ s'\sim P(\cdot\,|\,s,a) }\left[r(s,a)+\gamma  V^k(s')\right]\right]\\
      &=\mathbb{E}_{a \sim \pi(\cdot\,|\,s) }\left[\mathbb{E}_{ s'\sim P(\cdot\,|\,s,a) }\left[r(s,a)+\gamma \lim_{k\rightarrow \infty} V^k(s')\right]\right]\\
      &=T^{\pi}\tilde{V}^{\pi}(s)
  \end{align*} 
    for any fixed $s\in \cS$. With same argument, case $U=Q$ also holds. If $T$ is Bellman optimality operator, we use following lemma. 
    \begin{lemma}
    Let $W, W^k \in \cF(\cX)$ for $k=0,1,\dots$. If $W^k(x) \le W^{k+1}(x)$ for all $x \in \cX$, and $\{W^k\}_{k=0,1,\dots,}$ converge pointwise to  $W$, then $\lim_{k\rightarrow \infty} \{\sup_x W^k(x)\}=\sup_x W(x)$.
    \end{lemma} 
    \begin{proof}
        $W^k(x) \le W(x)$ implies that $\sup_x W^k(x)\le \sup_x W(x)$. If $\sup_x W(x)=a$, there exist $x$ which satisfying $a-W(x)< \frac{\epsilon}{2}$, and by definition of $W$, there exist $W^k$ such that $a-W^k(x)< \epsilon$ for any $\epsilon>0$.
    \end{proof}
    If $U=V$ and $\tilde{U}^{\star}=\tilde{V}^{\star}$, by previous lemma and monotone convergence theorem, we have 
  \begin{align*}
      \lim_{k\rightarrow \infty} T^{\star}V^k(s)&=\lim_{k\rightarrow \infty}\sup_a\left\{\mathbb{E}_{ s'\sim P(\cdot\,|\,s,a) }\left[r(s,a)+\gamma V^k(s')\right]\right\}\\
      &=\sup_a\left\{\lim_{k\rightarrow \infty}\mathbb{E}_{ s'\sim P(\cdot\,|\,s,a) }\left[r(s,a)+\gamma  V^k(s')\right]\right\}\\
      &=\sup_a\left\{\mathbb{E}_{ s'\sim P(\cdot\,|\,s,a) }\left[r(s,a)+\gamma \lim_{k\rightarrow \infty} V^k(s')\right]\right\}\\
      &=T^{\star}\tilde{U}^{\star}(s)
  \end{align*} 
    for any fixed $s\in \cS$. With similar argument, case $U=Q$ also holds. 
    
    Since $TU^k \rightarrow T\tilde{U}^{\star}$ and $U^k \rightarrow \tilde{U}^{\star}$ pointwisely, $TU^k-U^k $ converges pointwise to $ T\tilde{U}^{\star}-\tilde{U}^{\star}=0$. Thus, $\tilde{U}^{\star}$ is indeed fixed point of $T$. Furthermore, the fact that $U^{0} \le TU^{0} \le \tilde{U}^{\star}$ implies that  Lemma \ref{lem::Anc-VIE-2-1} and \ref{lem::Anc-VIC-2-1} hold. Therefore, we have
     \begin{align*}
     \infn{TU^k-U^k} \le \frac{1}{k+1}\infn{U^0-\tilde{U}^{\star}}.
 \end{align*}  
\end{proof}

\section{Omitted proofs in Section \ref{sec::complexity} }\label{s::omitted-complexity-proofs}
We present the proof of Theorem \ref{thm::complexity_contractive}.
\begin{proof}[Proof of Theorem \ref{thm::complexity_contractive}]
First, we prove the case $U^0=0$ for $n\ge k+2$. Consider the MDP $(\cS, \cA, P, r, \gamma)$ such that  
\begin{align*}
        \cS=\{s_1,\dots, s_{n}\}, \quad \cA=\{a_1\}, \quad P(s_i\,|\,s_j,a_1)=\mathds{1}_{\{i=j=1, \,j=i+1\}}, \quad r(s_i, a_1)=\mathds{1}_{\{i=2\}}.
\end{align*}
Then, $T= \gamma\cP^{\pi}U+[0,1,0, \dots, 0]^{\intercal}, U^{\star}=[0,1,\gamma,\dots, \gamma^{n-2}]^{\intercal}$, and $\infn{U^0-U^{\star}}=1$. Under the span condition, we can show that $\para{U^k}_{1}=\para{U^k}_{l}=0$ for $ k+2 \le l  \le n$ by following lemma. 
\begin{lemma}
Let $T \colon\real^{n} \rightarrow \real^{n}$ be defined as before. Then, under span condition, $\para{U^i}_{1}=0$ for $0\le i \le k $, and $\para{U^i}_{j}=0$ for $ 0\le i \le k$ and $i+2 \le j \le n$.
\end{lemma}
\begin{proof}
     Case $k=0$ is obvious. By induction, $\para{U^l}_1=0$ for $0 \le l \le i-1$. Then $\para{TU^l}_1=0$ for $0 \le l \le i-1$.  This implies that  $\para{TU^l-U^l}_1=0$ for $0 \le l \le i-1$. Hence $\para{U^i}_1=\para{U^0}_1=0$. Again, by induction, $\para{U^l}_j=0$ for $0 \le l \le i-1$, $l+2 \le j \le n$. Then $\para{TU^l}_j=0$ for $0 \le l \le i-1$, $l+3 \le j \le n$ and this implies that 
    $\para{TU^l-U^l}_j=0$ for $0 \le l \le i-1$, $l+3 \le j \le n$. Therefore, $\para{U^i}_{j}=0$ for $ i+2 \le j \le n$. 
\end{proof}
Then, we get
\begin{align*}
    TU^k-U^k =\Big(0, 1-\para{U^k}_2,\gamma\para{U^k}_2 -\para{U^k}_3, \dots, \gamma\para{U^k}_{k}-\para{U^k}_{k+1}, \gamma\para{U^k}_{k+1}, \underbrace{0, \dots, 0}_{n-k-2}\Big),
\end{align*}
and this implies 
\[\para{TU^k-U^k}_2+\gamma^{-1}\para{TU^k-U^k}_3+\cdots+ \gamma^{-k}\para{TU^k-U^k}_{k+2}=1.\] 
Taking the absolute value on both sides,
\begin{align*}
     (1+\cdots+\gamma^{-k})\max_{1\le i\le n} {\{|TU^k-U^k|_i\}} \ge 1.
\end{align*}
Therefore, we conclude
\begin{align*}\|TU^{k}-U^{k}\|_{\infty} \ge \frac{\gamma^k}{\sum_{i=0}^k \gamma^{i}} \infn{U^0-U^{\star}}.
\end{align*}

Now, we show that for any initial point $U^0 \in \real^n$, there exists an MDP which exhibits same lower bound with the case $U^0=0$. Denote by MDP($0$) and $T_0$ the worst-case MDP and Bellman consistency or opitmality operator constructed for $U^0=0$. Define an MDP($U^0$) $(\cS, \cA, P, r, \gamma)$ for $U^0\neq 0$ as
\begin{align*}
        \cS=\{s_1,\dots, s_{n}\}, \,\, \cA=\{a_1\}, \,\, P(s_i\,|\,s_j,a_1)=\mathds{1}_{\{i=j=1, \,j=i+1\}}, \,\, r(s_i, a_1)=\para{U^0-\cP^{\pi}U^0}_i+\mathds{1}_{\{i=2\}}.
\end{align*}
Then, Bellman consistency or optimality operator $T$ satisfies 
\[TU=T_0(U-U^0)+U^0.\]
Let $\tilde{U}^{\star}$ be fixed point of $T_0$. Then, if $U^{\star}=\tilde{U}^\star+U^0$, $U^{\star}$ is fixed point of $T$. Furthermore, if $\{U^i\}^k_{i=0}$ satisfies span condition 
\[U^i \in U^0+span\{TU^0-U^0, TU^1-U^1, \dots, TU^{i-1}-U^{i-1} \}, \qquad i=1,\dots, k, \]
$\tilde{U^i}=U^i-U^0$ is a sequence satisfying 
\[\tilde{U}^i \in \underbrace{\,\tilde{U}^0}_{=0}+span\{T_0\tilde{U}^0-\tilde{U}^0, T_0\tilde{U}^1-\tilde{U}^1, \dots, T_0\tilde{U}^{i-1}-\tilde{U}^{i-1} \}, \qquad i=1,\dots, k, \]
which is the same span condition in Theorem \ref{thm::complexity_contractive} with respect to $T_0$. This is because
\[TU^i-U^i=T_0(U^i-U^0)-(U^i-U^0)=T\tilde{U}^i-\tilde{U}^i\]
for $i=0,\dots,k$.
Thus, $\{\tilde{U}^i\}^k_{i=0}$
is a sequence starting from $0$ and satisfy the span condition for $T_0$. This implies that
\begin{align*}
     \infn{TU^k-U^k}&=\infn{T\tilde{U}^k-\tilde{U}^k}\\&
     \ge \frac{\gamma^k}{\sum^k_{i=0}\gamma^{i}}\infn{\tilde{U}^0-\tilde{U}^{\star}}\\&=\frac{\gamma^k}{\sum^k_{i=0}\gamma^{i}}\infn{U^0-U^{\star}}.
\end{align*}
Hence, MDP($U^0$) is indeed our desired worst-case instance. Lastly, the fact that $U^0-U^{\star}=\tilde{U}^0-\tilde{U}^{\star}=-(0,1,\gamma,\dots,\gamma^{n-2})$ implies $U^0 \le U^{\star}$.

\end{proof}

% \begin{theorem}\label{thm::complexity_nonexpansive_condition}
% Let $k\ge 0$, $n \ge k+2$, $\gamma=1$, and $U^0 \in \real^n$. Then there exists an MDP with $|\mathcal{S}|=n$ and $|\mathcal{A}|=1$ (which implies the Bellman consistency and optimality operator for V and Q all coincide as $T\colon \real^n \rightarrow \real^n$) such that $T$ has a fixed point and 
% \[\infn{TU^k-U^k} \ge \frac{1}{k+1}\infn{U^0-U^{\star}}\]
% for any fixed point $U_0 \le U^{\star}$ and iterates $\{U^i\}^{k}_{i=0}$ satisfying 
% \[U^i \in U^0+\mathrm{span}\{TU^0-U^0, TU^1-U^1, \dots, TU^{i-1}-U^{i-1} \} \qquad\text{ for }i=1,\dots,k.\]
% \end{theorem}
% \begin{proof}
%     XXX
% \end{proof}

\section{Omitted proofs in Section \ref{sec::Apx-Anc-VI} }
First, we prove following key lemma.
\begin{lemma}\label{lem::Apx-Anc-VI}
Let $0<\gamma<1$. For the iterates $\{U^k\}_{k=0,1,\dots}$ of \hyperlink{Anc-VI}{Anc-VI}, there exist nonexpansive linear operators $\{\cP^l\}_{l=0,1,\dots,k}$ and $\{\hat{\cP}^l\}_{l=0,1,\dots,k}$ such that
\begin{align*}
    T^{\star}U^k-U^k &\le \sum_{i=1}^k \left[(\beta_i-\beta_{i-1}(1-\beta_{i}))\para{\Pi^k_{j=i+1}(1-\beta_j)}\para{\Pi^i_{l=k}\gamma\cP^l}(U^0-U^{\star})\right]-\beta_k(U^0-U^{\star})
    \\&\quad +\Pi^k_{j=1}(1-\beta_j) \Pi^0_{l=k}\gamma\cP^l(U^0-U^{\star})- \sum_{i=1}^k\Pi^k_{j=i}(1-\beta_j)\Pi^{i+1}_{l=k}\gamma\cP^{l}\para{I-\gamma\cP^{{i}}}\epsilon^{i-1},\\
     T^{\star}U^k-U^k &  \ge \sum_{i=1}^k \left[(\beta_i-\beta_{i-1}(1-\beta_{i}))\para{\Pi^k_{j=i+1}(1-\beta_j)}\para{\Pi^i_{l=k}\gamma\hat{\cP}^l}(U^0-\hat{U}^{\star})\right]-\beta_k(U^0-\hat{U}^{\star})
    \\&\quad +\Pi^k_{j=1}(1-\beta_j) \Pi^0_{l=k}\gamma\hat{\cP}^l(U^0-\hat{U}^{\star})- \sum_{i=1}^k\Pi^k_{j=i}(1-\beta_j)\Pi^{i+1}_{l=k}\gamma\hat{\cP}^{l}\para{I-\gamma\hat{\cP}^{{i}}}\epsilon^{i-1}, 
\end{align*}
for $1 \le k$, where  $\Pi^k_{j=k+1}(1-\beta_j)=1$, $\Pi^{k+1}_{l=k}\gamma\cP^{l}=\Pi^{k+1}_{l=k}\gamma\hat{\cP}^{l}=I$, and $\beta_0=1$.
\end{lemma}
\begin{proof}[Proof of Lemma \ref{lem::Apx-Anc-VI}] 
 First, we prove the first inequality in Lemma \ref{lem::Apx-Anc-VI} by induction. 
 
 If $k=1$, 
\begin{align*}
    U^1-(1-\beta_1)U^0-\beta_1 U^{\star}&=(1-\beta_1)\epsilon^0+ \beta_1(U^0-U^{\star})+(1-\beta_1)(T^{\star}U^0-U^0)\\& \le (1-\beta_1)\epsilon^0+(1-\beta_1)\gamma\cP^{0}(U^0-U^{\star})+(2\beta_1-1)(U^0-U^{\star}),
\end{align*} 
 where inequality comes from Lemma \ref{lem::Anc-VIC-1-1-H1} with $\alpha=1, U=U^0, \bar{U}=U^0-U^{\star}$, and let $\bar{U}$ be the entire right hand side of inequality. Then, we have 
    \begin{align*}
    T^{\star}U^1-U^1&=T^{\star}U^1-(1-\beta_1)T^{\star}U^0-\beta_1 U^\star-\beta_1 (U^0-U^\star)-(1-\beta_1)\epsilon^0
    \\& \le \gamma\cP^{1}((1-\beta_1)\epsilon^0+(1-\beta_1)\gamma\cP^{0}(U^0-U^{\star})+(2\beta_1-1)(U^0-U^{\star}))-\beta_1(U^0-U^{\star})\\&\quad-(1-\beta_1)\epsilon^0\\&=(1-\beta_1)\gamma\cP^{1}\gamma\cP^{0}(U^0-U^{\star})+\gamma\cP^{1}(2\beta_1-1)(U^0-U^{\star})-\beta_1(U^0-U^{\star})\\&\quad-(I-\gamma\cP^1)(1-\beta_1)\epsilon^0.
\end{align*}
where inequality comes from Lemma \ref{lem::Anc-VIC-1-1-H1} with $\alpha=\beta_1, U=U^1,\tilde{U}=U^0$, and previously defined $\bar{U}$. 

By induction, 
\begin{align*}
    &U^k-(1-\beta_k)U^{k-1}-\beta_k U^{\star}\\&
    =\beta_k \para{U^{0}-U^{\star}}+(1-\beta_k)(T^{\star}U^{k-1}-U^{k-1})+(1-\beta_k)\epsilon^{k-1} \\
    &\le  (1-\beta_k) \sum_{i=1}^{k-1} \left[(\beta_i-\beta_{i-1}(1-\beta_{i}))\para{\Pi^{k-1}_{j=i+1}(1-\beta_j)}\para{\Pi^{i}_{l=k-1}\gamma\cP^l}(U^0-U^{\star})\right]\\
    &\quad-(1-\beta_k)\beta_{k-1}(U^0-U^{\star})+(1-\beta_k)\para{\Pi^{k-1}_{j=1}(1-\beta_j)} \para{\Pi^{0}_{l=k-1}\gamma\cP^l}(U^0-U^{\star})\\& \quad +\beta_k(U^0-U^{\star})-(1-\beta_k)\sum_{i=1}^{k-1}\Pi^{k-1}_{j=i}(1-\beta_j)\Pi^{i+1}_{l=k-1}\gamma\cP^{l}\para{I-\gamma\cP^{{i}}}\epsilon^{i-1}+(1-\beta_k)\epsilon^{k-1},
\end{align*}
and let $\bar{U}$ be the entire right hand side of inequality. Then, we have
\begin{align*}
    &T^{\star}U^k-U^k\\
    &=T^{\star}U^k-(1-\beta_k)T^{\star}U^{k-1}-\beta_k U^0-(1-\beta_k)\epsilon^{k-1}\\
    &=T^{\star}U^k-(1-\beta_k)T^{\star}U^{k-1}-\beta_kT^{\star}U^{\star}-\beta_k(U^0-U^{\star})-(1-\beta_k)\epsilon^{k-1}\\
    &\le \gamma \cP^k \bigg( (1-\beta_k) \sum_{i=1}^{k-1} \left[(\beta_i-\beta_{i-1}(1-\beta_{i}))\para{\Pi^{k-1}_{j=i+1}(1-\beta_j)}\para{\Pi^{i}_{l=k-1}\gamma\cP^l}(U^0-U^{\star})\right]\\
    &\quad-(1-\beta_k)\beta_{k-1}(U^0-U^{\star})+(1-\beta_k)\para{\Pi^{k-1}_{j=1}(1-\beta_j)} \para{\Pi^{0}_{l=k-1}\gamma\cP^l}(U^0-U^{\star})\\& \quad +\beta_k(U^0-U^{\star})-(1-\beta_k)\sum_{i=1}^{k-1}\Pi^{k-1}_{j=i}(1-\beta_j)\Pi^{i+1}_{l=k-1}\gamma\cP^{l}\para{I-\gamma\cP^{{i}}}\epsilon^{i-1}+(1-\beta_k)\epsilon^{k-1}\bigg)\\&\quad-\beta_k(U^0-U^{\star})-(1-\beta_k)\epsilon^{k-1}\\
    % & = (1-\beta_k)\gamma\cP^k\sum_{i=1}^{k-1} \left[(\beta_i-\beta_{i-1}(1-\beta_{i}))\para{\Pi^{k-1}_{j=i+1}(1-\beta_j)}\para{\Pi^{i}_{l=k-1}\gamma\cP^l}(U^0-U^{\star})\right]
    % \\&\quad -\beta_{k-1}(1-\beta_k)\gamma\cP^k(U^0-U^{\star})+(1-\beta_k)\gamma\cP^k\para{\Pi^{k-1}_{j=1}(1-\beta_j)} \para{\Pi^{0}_{l=k-1}\gamma\cP^l}(U^0-U^{\star})\\
    % &\quad +\beta_k\gamma\cP^k\para{U^{0}-U^{\star}}-\beta_k(U^0-U^{\star})\\
    & =  \sum_{i=1}^k \left[(\beta_i-\beta_{i-1}(1-\beta_{i}))\para{\Pi^k_{j=i+1}(1-\beta_j)}\para{\Pi^i_{l=k}\gamma\cP^l}(U^0-U^{\star})\right]-\beta_k(U^0-U^{\star})
    \\&\quad +\Pi^k_{j=1}(1-\beta_j) \Pi^0_{l=k}\gamma\cP^l(U^0-U^{\star})- \sum_{i=1}^k\Pi^k_{j=i}(1-\beta_j)\Pi^{i+1}_{l=k}\gamma\cP^{l}\para{I-\gamma\cP^{{i}}}\epsilon^{i-1},
\end{align*}
where inequality comes from Lemma \ref{lem::Anc-VIC-1-1-H1} with $\alpha=\beta_k, U=U^k, \tilde{U}=U^{k-1}$, and previously defined $\bar{U}$. 

  Now, we prove second inequality in Lemma \ref{lem::Apx-Anc-VI} by induction.

  If $k=1$,
\begin{align*}
    U^1-(1-\beta_1)U^0-\beta_1 \hat{U}^{\star}&=(1-\beta_1)\epsilon^0+ \beta_1(U^0-\hat{U}^{\star})+(1-\beta_1)(T^{\star}U^0-U^0)\\& \ge (1-\beta_1)\epsilon^0+(1-\beta_1)\gamma\hat{\cP}^{0}(U^0-\hat{U}^{\star})+(2\beta_1-1)(U^0-\hat{U}^{\star}),
\end{align*} 
 where inequality comes from Lemma \ref{lem::Anc-VIC-1-1-H2} with $\alpha=1, U=U^0, \bar{U}=U^0-\hat{U}^{\star}$, and let $\bar{U}$ be the entire right hand side of inequality. Then, we have 
    \begin{align*}
    T^{\star}U^1-U^1&=T^{\star}U^1-(1-\beta_1)T^{\star}U^0-\beta_1 \hat{U}^{\star}-\beta_1 (U^0-\hat{U}^{\star})-(1-\beta_1)\epsilon^0
    \\& \ge \gamma\hat{\cP}^{1}((1-\beta_1)\epsilon^0+(1-\beta_1)\gamma\hat{\cP}^{0}(U^0-\hat{U}^{\star})+(2\beta_1-1)(U^0-\hat{U}^{\star}))-\beta_1(U^0-\hat{U}^{\star})\\&\quad-(1-\beta_1)\epsilon^0\\&=(1-\beta_1)\gamma\hat{\cP}^{1}\gamma\hat{\cP}^{0}(U^0-\hat{U}^{\star})+\gamma\hat{\cP}^{1}(2\beta_1-1)(U^0-\hat{U}^{\star})-\beta_1(U^0-\hat{U}^{\star})\\&\quad-(I-\gamma\hat{\cP}^1)(1-\beta_1)\epsilon^0.
\end{align*}
where inequality comes from Lemma \ref{lem::Anc-VIC-1-1-H2} with $\alpha=\beta_1, U=U^1,\tilde{U}=U^0$, and previously defined $\bar{U}$. 

By induction, 
\begin{align*}
    &U^k-(1-\beta_k)U^{k-1}-\beta_k \hat{U}^{\star}\\&
    =\beta_k \para{U^{0}-\hat{U}^{\star}}+(1-\beta_k)(T^{\star}U^{k-1}-U^{k-1})+(1-\beta_k)\epsilon^{k-1} \\
    &\ge  (1-\beta_k) \sum_{i=1}^{k-1} \left[(\beta_i-\beta_{i-1}(1-\beta_{i}))\para{\Pi^{k-1}_{j=i+1}(1-\beta_j)}\para{\Pi^{i}_{l=k-1}\gamma\hat{\cP}^l}(U^0-\hat{U}^{\star})\right]\\
    &\quad-(1-\beta_k)\beta_{k-1}(U^0-\hat{U}^{\star})+(1-\beta_k)\para{\Pi^{k-1}_{j=1}(1-\beta_j)} \para{\Pi^{0}_{l=k-1}\gamma\hat{\cP}^l}(U^0-\hat{U}^{\star})\\& \quad +\beta_k(U^0-\hat{U}^{\star})-(1-\beta_k)\sum_{i=1}^{k-1}\Pi^{k-1}_{j=i}(1-\beta_j)\Pi^{i+1}_{l=k-1}\gamma\hat{\cP}^{l}\para{I-\gamma\hat{\cP}^{{i}}}\epsilon^{i-1}+(1-\beta_k)\epsilon^{k-1},
\end{align*}
and let $\bar{U}$ be the entire right hand side of inequality. Then, we have
\begin{align*}
    &T^{\star}U^k-U^k\\
    &=T^{\star}U^k-(1-\beta_k)T^{\star}U^{k-1}-\beta_k U^0-(1-\beta_k)\epsilon^{k-1}\\
    &=T^{\star}U^k-(1-\beta_k)T^{\star}U^{k-1}-\beta_kT^{\star}\hat{U}^{\star}-\beta_k(U^0-\hat{U}^{\star})-(1-\beta_k)\epsilon^{k-1}\\
    &\ge \gamma \hat{\cP}^k \bigg( (1-\beta_k) \sum_{i=1}^{k-1} \left[(\beta_i-\beta_{i-1}(1-\beta_{i}))\para{\Pi^{k-1}_{j=i+1}(1-\beta_j)}\para{\Pi^{i}_{l=k-1}\gamma\hat{\cP}^l}(U^0-\hat{U}^{\star})\right]\\
    &\quad-(1-\beta_k)\beta_{k-1}(U^0-\hat{U}^{\star})+(1-\beta_k)\para{\Pi^{k-1}_{j=1}(1-\beta_j)} \para{\Pi^{0}_{l=k-1}\gamma\hat{\cP}^l}(U^0-\hat{U}^{\star})\\& \quad +\beta_k(U^0-\hat{U}^{\star})-(1-\beta_k)\sum_{i=1}^{k-1}\Pi^{k-1}_{j=i}(1-\beta_j)\Pi^{i+1}_{l=k-1}\gamma\hat{\cP}^{l}\para{I-\gamma\hat{\cP}^{{i}}}\epsilon^{i-1}+(1-\beta_k)\epsilon^{k-1}\bigg)\\&\quad-\beta_k(U^0-\hat{U}^{\star})-(1-\beta_k)\epsilon^{k-1}\\
    % & = (1-\beta_k)\gamma\hat{\cP}^k\sum_{i=1}^{k-1} \left[(\beta_i-\beta_{i-1}(1-\beta_{i}))\para{\Pi^{k-1}_{j=i+1}(1-\beta_j)}\para{\Pi^{i}_{l=k-1}\gamma\hat{\cP}^l}(U^0-\hat{U}^{\star})\right]
    % \\&\quad -\beta_{k-1}(1-\beta_k)\gamma\hat{\cP}^k(U^0-\hat{U}^{\star})+(1-\beta_k)\gamma\hat{\cP}^k\para{\Pi^{k-1}_{j=1}(1-\beta_j)} \para{\Pi^{0}_{l=k-1}\gamma\hat{\cP}^l}(U^0-\hat{U}^{\star})\\
    % &\quad +\beta_k\gamma\hat{\cP}^k\para{U^{0}-\hat{U}^{\star}}-\beta_k(U^0-\hat{U}^{\star})\\
    & =  \sum_{i=1}^k \left[(\beta_i-\beta_{i-1}(1-\beta_{i}))\para{\Pi^k_{j=i+1}(1-\beta_j)}\para{\Pi^i_{l=k}\gamma\hat{\cP}^l}(U^0-\hat{U}^{\star})\right]-\beta_k(U^0-\hat{U}^{\star})
    \\&\quad +\Pi^k_{j=1}(1-\beta_j) \Pi^0_{l=k}\gamma\hat{\cP}^l(U^0-\hat{U}^{\star})- \sum_{i=1}^k\Pi^k_{j=i}(1-\beta_j)\Pi^{i+1}_{l=k}\gamma\hat{\cP}^{l}\para{I-\gamma\hat{\cP}^{{i}}}\epsilon^{i-1},
\end{align*}
where inequality comes from Lemma \ref{lem::Anc-VIC-1-1-H2} with $\alpha=\beta_k, U=U^k, \tilde{U}=U^{k-1}$, and previously defined $\bar{U}$
\end{proof}

Now, we prove the first rate in Theorem \ref{thm::Apx-Anc-VI}. 
\begin{proof}[Proof of first rate in Theorem \ref{thm::Apx-Anc-VI}] 
Since $B_1 \le A \le B_2$ implies $ \infn{A} \le \sup\{\infn{B_1}, \infn{B_2}\}$ for $A,B \in \cF(\cX)$, if we take $\infn{\cdot}$ right side of first inequality in Lemma \ref{lem::Apx-Anc-VI}, we have 
\begin{align*}
    & \frac{\para{\gamma^{-1}-\gamma}\para{1+2\gamma-\gamma^{k+1}}}{\para{\gamma^{k+1}}^{-1}-\gamma^{k+1}}\infn{U^0-U^{\star}} +(1+\gamma)\sum_{i=1}^k\para{\Pi^k_{j=i}(1-\beta_j)}\gamma^{k-i}\infn{\epsilon^{i-1}}\\
    &\le  \frac{\para{\gamma^{-1}-\gamma}\para{1+2\gamma-\gamma^{k+1}}}{\para{\gamma^{k+1}}^{-1}-\gamma^{k+1}}\infn{U^0-U^{\star}}+\frac{1+\gamma}{1+\gamma^{k+1}}\frac{1-\gamma^k}{1-\gamma}\max_{0\le i\le k-1 } \infn{\epsilon^i}.
\end{align*}
If we apply second inequality of Lemma \ref{lem::Apx-Anc-VI} and take $\infn{\cdot}$-norm right side, we have
\begin{align*}
      \frac{\para{\gamma^{-1}-\gamma}\para{1+2\gamma-\gamma^{k+1}}}{\para{\gamma^{k+1}}^{-1}-\gamma^{k+1}}\infn{U^0-\hat{U}^{\star}}+\frac{1+\gamma}{1+\gamma^{k+1}}\frac{1-\gamma^k}{1-\gamma}\max_{0\le i\le k-1 } \infn{\epsilon^i}.
\end{align*}
Therefore, we get  
\begin{align*}
    \infn{T^{\star}U^k-U^k} &\le \frac{\para{\gamma^{-1}-\gamma}\para{1+2\gamma-\gamma^{k+1}}}{\para{\gamma^{k+1}}^{-1}-\gamma^{k+1}}\max{\left\{\infn{U^0-U^{\star}},\infn{U^0-\hat{U}^\star}\right\}}\\&\quad +\frac{1+\gamma}{1+\gamma^{k+1}}\frac{1-\gamma^k}{1-\gamma}\max_{0\le i\le k-1 } \infn{\epsilon^i}.
\end{align*}
\end{proof}
Now,  for the second rate in Theorem \ref{thm::Apx-Anc-VI}, we present following key lemma.
 \begin{lemma}\label{lem::Apx-Anc-VI-2}
Let $0<\gamma<1$. For the iterates $\{U^k\}_{k=0,1,\dots}$ of \hyperlink{Anc-VI}{Anc-VI}, if $U^0 \ge T^{\star}U^0 $, there exist nonexpansive linear operators $\{\cP^l\}_{l=0,1,\dots,k}$ and $\{\hat{\cP}^l\}_{l=0,1,\dots,k}$ such that
\begin{align*}
    T^{\star}U^k-U^k &\le \Pi^k_{j=1}(1-\beta_j) \Pi^0_{l=k}\gamma\cP^l(U^0-U^{\star})- \sum_{i=1}^k\Pi^k_{j=i}(1-\beta_j)\Pi^{i+1}_{l=k}\gamma\cP^{l}\para{I-\gamma\cP^{{i}}}\epsilon^{i-1} \\&\quad-\beta^k(U^0-U^{\star}),\\
     T^{\star}U^k-U^k &  \ge \sum_{i=1}^k \left[(\beta_i-\beta_{i-1}(1-\beta_{i}))\para{\Pi^k_{j=i+1}(1-\beta_j)}\para{\Pi^i_{l=k}\gamma\hat{\cP}^l}(U^0-\hat{U}^{\star})\right] -\beta_k(U^0-\hat{U}^{\star})\\&\quad - \sum_{i=1}^k\Pi^k_{j=i}(1-\beta_j)\Pi^{i+1}_{l=k}\gamma\hat{\cP}^{l}\para{I-\gamma\hat{\cP}^{{i}}}\epsilon^{i-1},
\end{align*}
for $1\le k$, where $\Pi^k_{j=k+1}(1-\beta_j)=1$, $\Pi^{k+1}_{l=k}\gamma\cP^{l}=\Pi^{k+1}_{l=k}\gamma\hat{\cP}^{l}=I$, and $\beta_0=1$.
\end{lemma}     
\begin{proof}[Proof of Lemma \ref{lem::Apx-Anc-VI-2}]
If $U^0 \ge T^{\star}U^0$, $U^0 \ge \lim_{m\rightarrow \infty}(T^{\star})^mU^0= U^{\star}$ by Lemma \ref{lem::monotonicity}. By Lemma \ref{lem::anti-fixed-point}, this also implies $U^0 \ge \hat{U}^{\star}$.

 First, we prove first inequality in Lemma \ref{lem::Apx-Anc-VI-2} by induction.
  If $k=1$, 
\begin{align*}
    U^1-(1-\beta_1)U^0-\beta_1 U^{\star}&=(1-\beta_1)\epsilon^0+ \beta_1(U^0-U^{\star})+(1-\beta_1)(T^{\star}U^0-U^0)\\& \le (1-\beta_1)\epsilon^0+(1-\beta_1)\gamma\cP^{0}(U^0-U^{\star})+(2\beta_1-1)(U^0-U^{\star})\\& \le (1-\beta_1)\epsilon^0+(1-\beta_1)\gamma\cP^{0}(U^0-U^{\star}),
\end{align*} 
 where the second inequality is from the $ (2\beta_1-1)(U^0-U^{\star})\le 0$, and first inequality comes from Lemma \ref{lem::Anc-VIC-1-1-H1} with $\alpha=1, U=U^0, \bar{U}=U^0-U^{\star}$, and let $\bar{U}$ be the entire right hand side of inequality. Then, we have 
    \begin{align*}
    T^{\star}U^1-U^1&=T^{\star}U^1-(1-\beta_1)T^{\star}U^0-\beta_1 U^\star-\beta_1 (U^0-U^\star)-(1-\beta_1)\epsilon^0
    \\& \le \gamma\cP^{1}((1-\beta_1)\epsilon^0+(1-\beta_1)\gamma\cP^{0}(U^0-U^{\star}))-\beta_1(U^0-U^{\star})-(1-\beta_1)\epsilon^0\\&=(1-\beta_1)\gamma\cP^{1}\gamma\cP^{0}(U^0-U^{\star})-\beta_1(U^0-U^{\star})-(I-\gamma\cP^1)(1-\beta_1)\epsilon^0.
\end{align*}
where inequality comes from Lemma \ref{lem::Anc-VIC-1-1-H1} with $\alpha=\beta_1, U=U^1,\tilde{U}=U^0$, and previously defined $\bar{U}$. 

By induction, 
\begin{align*}
    &U^k-(1-\beta_k)U^{k-1}-\beta_k U^{\star}\\&
    =\beta_k \para{U^{0}-U^{\star}}+(1-\beta_k)(T^{\star}U^{k-1}-U^{k-1})+(1-\beta_k)\epsilon^{k-1} \\
    &\le \beta_k(U^0-U^{\star})-(1-\beta_k)\beta_{k-1}(U^0-U^{\star})+(1-\beta_k)\para{\Pi^{k-1}_{j=1}(1-\beta_j)} \para{\Pi^{0}_{l=k-1}\gamma\cP^l}(U^0-U^{\star})\\& \quad+(1-\beta_k)\epsilon^{k-1} -(1-\beta_k)\sum_{i=1}^{k-1}\Pi^{k-1}_{j=i}(1-\beta_j)\Pi^{i+1}_{l=k-1}\gamma\cP^{l}\para{I-\gamma\cP^{{i}}}\epsilon^{i-1}\\&\le (1-\beta_k)\para{\Pi^{k-1}_{j=1}(1-\beta_j)} \para{\Pi^{0}_{l=k-1}\gamma\cP^l}(U^0-U^{\star})+(1-\beta_k)\epsilon^{k-1}\\&-(1-\beta_k)\sum_{i=1}^{k-1}\Pi^{k-1}_{j=i}(1-\beta_j)\Pi^{i+1}_{l=k-1}\gamma\cP^{l}\para{I-\gamma\cP^{{i}}}\epsilon^{i-1},
\end{align*}
where the second inequality is from $\beta_k-(1-\beta_k)\beta_{k-1} \le 0$ and let $\bar{U}$ be the entire right hand side of inequality. Then, we have
\begin{align*}
    &T^{\star}U^k-U^k\\
    &=T^{\star}U^k-(1-\beta_k)T^{\star}U^{k-1}-\beta_kT^{\star}U^{\star}-\beta_k(U^0-U^{\star})-(1-\beta_k)\epsilon^{k-1}\\
    &\le \gamma \cP^k \bigg((1-\beta_k)\para{\Pi^{k-1}_{j=1}(1-\beta_j)} \para{\Pi^{0}_{l=k-1}\gamma\cP^l}(U^0-U^{\star})+(1-\beta_k)\epsilon^{k-1}\\&-(1-\beta_k)\sum_{i=1}^{k-1}\Pi^{k-1}_{j=i}(1-\beta_j)\Pi^{i+1}_{l=k-1}\gamma\cP^{l}\para{I-\gamma\cP^{{i}}}\epsilon^{i-1}\bigg)-\beta_k(U^0-U^{\star})-(1-\beta_k)\epsilon^{k-1}
    % & = (1-\beta_k)\gamma\cP^k\sum_{i=1}^{k-1} \left[(\beta_i-\beta_{i-1}(1-\beta_{i}))\para{\Pi^{k-1}_{j=i+1}(1-\beta_j)}\para{\Pi^{i}_{l=k-1}\gamma\cP^l}(U^0-U^{\star})\right]
    % \\&\quad -\beta_{k-1}(1-\beta_k)\gamma\cP^k(U^0-U^{\star})+(1-\beta_k)\gamma\cP^k\para{\Pi^{k-1}_{j=1}(1-\beta_j)} \para{\Pi^{0}_{l=k-1}\gamma\cP^l}(U^0-U^{\star})\\
    % &\quad +\beta_k\gamma\cP^k\para{U^{0}-U^{\star}}-\beta_k(U^0-U^{\star})\\
    \\&\quad =   \Pi^k_{j=1}(1-\beta_j) \Pi^0_{l=k}\gamma\cP^l(U^0-U^{\star})- \sum_{i=1}^k\Pi^k_{j=i}(1-\beta_j)\Pi^{i+1}_{l=k}\gamma\cP^{l}\para{I-\gamma\cP^{{i}}}\epsilon^{i-1} \\&\quad-\beta^k(U^0-U^{\star}),
\end{align*}
where the first inequality comes from Lemma \ref{lem::Anc-VIC-1-1-H1} with $\alpha=\beta_k, U=U^k, \tilde{U}=U^{k-1}$, and previously defined $\bar{U}$.

For the second inequality in Lemma \ref{lem::Apx-Anc-VI-2}, if $k=1$, 
\begin{align*}
    U^1-(1-\beta_1)U^0-\beta_1 \hat{U}^{\star}&=(1-\beta_1)\epsilon^0+ \beta_1(U^0-\hat{U}^{\star})+(1-\beta_1)(T^{\star}U^0-U^0)\\& =(1-\beta_1)\epsilon^0+ \beta_1(U^0-\hat{U}^{\star})+(1-\beta_1)(T^{\star}U^0-\hat{U}^{\star}-(U^0-\hat{U}^{\star}))
    \\& \ge (1-\beta_1)\epsilon^0+ \beta_1(U^0-\hat{U}^{\star})-(1-\beta_1)(U^0-\hat{U}^{\star})
\end{align*} 
 where the second inequality is from $U^0 \ge T^{\star} U^0\ge\hat{U}^{\star}$, and let $\bar{U}$ be the entire right hand side of inequality. Then, we have 
     \begin{align*}
    T^{\star}U^1-U^1&=T^{\star}U^1-(1-\beta_1)T^{\star}U^0-\beta_1 U^\star-\beta_1 (U^0-\hat{U}^\star)-(1-\beta_1)\epsilon^0
    \\& \ge \gamma\cP^{1}( (1-\beta_1)\epsilon^0+\beta_1(U^0-\hat{U}^{\star})-(1-\beta_1)(U^0-\hat{U}^{\star}))-\beta_1(U^0-\hat{U}^{\star})-(1-\beta_1)\epsilon^0\\&=(2\beta_1-1)\gamma\cP^{1}(U^0-\hat{U}^{\star})-\beta_1(U^0-\hat{U}^{\star})-(I-\gamma\cP^1)(1-\beta_1)\epsilon^0.
\end{align*}
where inequality comes from Lemma \ref{lem::Anc-VIC-1-1-H2} with $\alpha=\beta_1, U=U^1,\tilde{U}=U^0$, and previously defined $\bar{U}$. 

By induction,
\begin{align*}
    &U^k-(1-\beta_k)U^{k-1}-\beta_k \hat{U}^{\star}\\&\ge  (1-\beta_k) \sum_{i=1}^{k-1} \left[(\beta_i-\beta_{i-1}(1-\beta_{i}))\para{\Pi^{k-1}_{j=i+1}(1-\beta_j)}\para{\Pi^{i}_{l=k-1}\gamma\hat{\cP}^l}(U^0-\hat{U}^{\star})\right]\\
    &\quad+(\beta_k-(1-\beta_k)\beta_{k-1})(U^0-\hat{U}^{\star})+(1-\beta_k)\epsilon^{k-1}\\&\quad-(1-\beta_k)\sum_{i=1}^{k-1}\Pi^{k-1}_{j=i}(1-\beta_j)\Pi^{i+1}_{l=k-1}\gamma\hat{\cP}^{l}\para{I-\gamma\hat{\cP}^{{i}}}\epsilon^{i-1}, 
\end{align*}
and let $\bar{U}$ be the entire right hand side of inequality. Then, we have

    \begin{align*}
    &T^{\star}U^k-U^k\\
    &=T^{\star}U^k-(1-\beta_k)T^{\star}U^{k-1}-\beta_k\hat{T}^{\star}\hat{U}^{\star}-\beta_k(U^0-\hat{U}^{\star})-(1-\beta_k)\epsilon^{k-1}\\
     &\ge \gamma \hat{\cP}^k \bigg( (1-\beta_k) \sum_{i=1}^{k-1} \left[(\beta_i-\beta_{i-1}(1-\beta_{i}))\para{\Pi^{k-1}_{j=i+1}(1-\beta_j)}\para{\Pi^{i}_{l=k-1}\gamma\hat{\cP}^l}(U^0-\hat{U}^{\star})\right]\\
    &\quad+(\beta_k-(1-\beta_k)\beta_{k-1})(U^0-\hat{U}^{\star})-(1-\beta_k)\sum_{i=1}^{k-1}\Pi^{k-1}_{j=i}(1-\beta_j)\Pi^{i+1}_{l=k-1}\gamma\hat{\cP}^{l}\para{I-\gamma\hat{\cP}^{{i}}}\epsilon^{i-1}\\&\quad+(1-\beta_k)\epsilon^{k-1}\bigg)-(1-\beta_k)\epsilon^{k-1}-\beta_k(U^0-\hat{U}^{\star})\\
    % & = (1-\beta_k)\gamma\hat{\cP}^k\sum_{i=1}^{k-1} \left[(\beta_i-\beta_{i-1}(1-\beta_{i}))\para{\Pi^{k-1}_{j=i+1}(1-\beta_j)}\para{\Pi^{i}_{l=k-1}\gamma\hat{\cP}^l}(U^0-\hat{U}^{\star})\right]
    % \\&\quad -\beta_{k-1}(1-\beta_k)\gamma\hat{\cP}^k(U^0-\hat{U}^{\star})+(1-\beta_k)\gamma\hat{\cP}^k\para{\Pi^{k-1}_{j=1}(1-\beta_j)} \para{\Pi^{0}_{l=k-1}\gamma\hat{\cP}^l}(U^0-\hat{U}^{\star})\\
    % &\quad +\beta_k\gamma\hat{\cP}^k\para{U^{0}-\hat{U}^{\star}}-\beta_k(U^0-\hat{U}^{\star})\\
    & = \sum_{i=1}^k \left[(\beta_i-\beta_{i-1}(1-\beta_{i}))\para{\Pi^k_{j=i+1}(1-\beta_j)}\para{\Pi^i_{l=k}\gamma\hat{\cP}^l}(U^0-\hat{U}^{\star})\right]-\beta_k(U^0-\hat{U}^{\star})\\&
     \quad- \sum_{i=1}^k\Pi^k_{j=i}(1-\beta_j)\Pi^{i+1}_{l=k}\gamma\hat{\cP}^{l}\para{I-\gamma\hat{\cP}^{{i}}}\epsilon^{i-1},
\end{align*}
where inequality comes from Lemma \ref{lem::Anc-VIC-1-1-H2} with $\alpha=\beta_k, U=U^k, \tilde{U}=U^{k-1}$, and previously defined $\bar{U}$.
\end{proof}
Now, we prove the second rate in Theorem \ref{thm::Apx-Anc-VI}. 
\begin{proof}[Proof of second rate in Theorem \ref{thm::Apx-Anc-VI}] 
If we take $\infn{\cdot}$ right side of first inequality in Lemma \ref{lem::Apx-Anc-VI-2}, we have 
\begin{align*}
   \frac{\para{\gamma^{-1}-\gamma}\gamma}{\para{\gamma^{k+1}}^{-1}-\gamma^{k+1}}\infn{U^0-U^{\star}}
    +\frac{1+\gamma}{1+\gamma^{k+1}}\frac{1-\gamma^k}{1-\gamma}\max_{0\le i\le k-1 } \infn{\epsilon^i}.
\end{align*}
If we apply second inequality of Lemma \ref{lem::Apx-Anc-VI-2} and take $\infn{\cdot}$-norm right side, we have
\begin{align*}  \frac{\para{\gamma^{-1}-\gamma}\para{1+\gamma-\gamma^{k+1}}}{\para{\gamma^{k+1}}^{-1}-\gamma^{k+1}}\infn{U^0-\hat{U}^{\star}} +\frac{1+\gamma}{1+\gamma^{k+1}}\frac{1-\gamma^k}{1-\gamma}\max_{0\le i\le k-1 } \infn{\epsilon^i}.
\end{align*}
Therefore, we get  
\begin{align*}
    \infn{T^{\star}U^k-U^k} &\le \frac{\para{\gamma^{-1}-\gamma}\para{1+\gamma-\gamma^{k+1}}}{\para{\gamma^{k+1}}^{-1}-\gamma^{k+1}}\infn{U^0-\hat{U}^\star}+\frac{1+\gamma}{1+\gamma^{k+1}}\frac{1-\gamma^k}{1-\gamma}\max_{0\le i\le k-1 } \infn{\epsilon^i},
\end{align*}
since $\hat{U}^{\star} \le U^{\star}\le U^0 $ implies that
\[\frac{\para{\gamma^{-1}-\gamma}\gamma}{\para{\gamma^{k+1}}^{-1}-\gamma^{k+1}}\infn{U^0-U^{\star}} \le\frac{\para{\gamma^{-1}-\gamma}\para{1+\gamma-\gamma^{k+1}}}{\para{\gamma^{k+1}}^{-1}-\gamma^{k+1}}\infn{U^0-\hat{U}^\star}. \]
\end{proof}

\section{Omitted proofs in Section \ref{sec::GS-Anc-VI}}
For the analyses, we first define $\hat{T}_{GS}^{\star}\colon \reals^n\rightarrow \reals^n$ as
\[\hat{T}_{GS}^{\star}=\hat{T}^{\star}_{n}\cdots \hat{T}^{\star}_{2}\hat{T}^{\star}_{1},\]
where $\hat{T}_j^{\star}: \reals^n \rightarrow \reals^n$ is defined as
\begin{align*}
\hat{T}^{\star}_j(U)=(U_1,\dots,U_{j-1},\para{\hat{T}^{\star}(U)}_j,U_{j+1},\dots,U_n)
    \end{align*}
for $j=1,\dots,n$, where $\hat{T}^{\star}$ is Bellman anti-optimality operator. 

\begin{fact}\label{lem::GS}[Classical result, \protect{\cite[Proposition~1.3.2]{bertsekas2015dynamic}}]
$\hat{T}_{GS}^{\star}$ is a $\gamma$-contractive operator and has the same fixed point as $\hat{T}^{\star}$.
\end{fact}
%Note that $\hat{T}_{GS}^{\star}\colon \reals^n\rightarrow \reals^n$ is $\gamma$-contractive and has the same fixed point as $\hat{T}^{\star}$.

Now, we introduce the following lemmas.
 \begin{lemma}\label{lem::GS-Anc-VIC-H1} Let $0<\gamma<1$. If $ 0\le \alpha \le 1$, then there exist $\gamma$-contractive nonnegative matrix $\cP_{GS}$ such that 
\begin{align*}
    T_{GS}^{\star}U-(1-\alpha) T_{GS}^{\star}\tilde{U}-\alpha T_{GS}^{\star}U^{\star} &\le \cP_{GS}(U-(1-\alpha) \tilde{U}-\alpha U^{\star}).
\end{align*}  
\end{lemma}
\begin{lemma}\label{lem::GS-Anc-VIC-H2} Let $0<\gamma<1$. If $ 0\le \alpha \le 1$, then there exist $\gamma$-contractive nonnegative matrix $\hat{\cP}_{GS}$ such that 
\begin{align*}
    \hat{\cP}_{GS}(U-(1-\alpha) \tilde{U}-\alpha \hat{U}^{\star}) \le T_{GS}^{\star}U-(1-\alpha) T_{GS}^{\star}\tilde{U}-\alpha \hat{T}_{GS}^{\star}\hat{U}^{\star}.
\end{align*}
\end{lemma}
\begin{proof}[Proof of Lemma \ref{lem::GS-Anc-VIC-H1}]
First let $U=V, \tilde{U}=\tilde{V}, U^{\star}=V^{\star}$. For $1 \le i \le n$, we have
\begin{align*}
    T_i^{\star}V(s_i)-(1-\alpha)T_i^{\star}\tilde{V}(s_i)-\alpha T_i^{\star}V^{\star}(s_i) &\le T_i^{\pi_i}V(s_i)-(1-\alpha)T_i^{\pi_i}\tilde{V}(s_i)-\alpha T_i^{\pi_i}V^{\star}(s_i)\\
    &= \gamma\cP^{\pi_i}\para{V-(1-\alpha) \tilde{V}-\alpha V^{\star} }(s_i), 
\end{align*}
where $\pi_i$ is the greedy policy satisfying $T^{\pi_i}V=T^{\star}V$ and first inequality is from $T^{\pi_i}\tilde{V} \le T^{\star}\tilde{V}$ and $T^{\pi_i}V^{\star} \le T^{\star}V^{\star}$.
% \begin{align*}
%     &T_i^{\star}V(s_i)-(1-\alpha)T_i^{\star}\tilde{V}(s_i)-\alpha T_i^{\star}V^{\star}(s_i)\\& = \sup_{a \in \cA} \bigg\{r(s_i,a)+\gamma \mathbb{E}_{s'\sim P(\cdot\,|\,s_i,a) }\left[V(s')\right]\bigg\}- \sup_{a \in \cA} \bigg\{(1-\alpha) r(s_i,a)+(1-\alpha) \gamma\mathbb{E}_{s'\sim P(\cdot\,|\,s_i,a) }\left[\tilde{V}(s')\right]\bigg\}\\
%     &\quad - \sup_{a \in \cA} \bigg\{\alpha r(s_i,a)+\alpha  \gamma\mathbb{E}_{s'\sim P(\cdot\,|\,s_i,a) }\left[V^{\star}(s')\right]\bigg\}\\
%     & \le \gamma \sup_{a \in \cA} \bigg\{ \mathbb{E}_{s'\sim P(\cdot\,|\,s_i,a) }\left[V(s')-(1-\alpha) \tilde{V}(s')-\alpha V^{\star}(s')\right]\bigg\}.
% \end{align*}
% Let $\pi_i(\cdot\,|\,s)=\argmax_{a\in \cA} \mathbb{E}_{s'\sim P(\cdot\,|\,s,a) }\left[V(s')-(1-\alpha) \tilde{V}(s')-\alpha V^{\star}(s')\right]$ and 
Then, define matrix $\cP_i$ as
\begin{align*}
\cP_i(V)=(V_1,\dots,V_{i-1},\para{\gamma\cP^{\pi_i}(V)}_i,V_{i+1},\dots,V_n)
    \end{align*}
for $i=1,\dots,n$. Note that $\cP_i$ is nonnegative matrix since $\cP^{\pi_i}$ is nonnegative matrix. Then, we have
\[T_i^{\star}V-(1-\alpha)T_i^{\star}\tilde{V}-\alpha T_i^{\star}V^{\star}\le \cP_i(V-(1-\alpha) \tilde{V}-\alpha V^{\star}).\] 
By induction, there exist a sequence of matrices $\{\cP_i\}_{i=1,\dots,n}$ satisfying 
\[T_{GS}^{\star}V-(1-\alpha)T_{GS}^{\star}\tilde{V}-\alpha T_{GS}^{\star}V^{\star} \le  \cP_n\cdots\cP_1(V-(1-\alpha) \tilde{V}-\alpha V^{\star})\]
since $T_i^{\star} V^{\star}=V^{\star}$ for all $i$. Denote $P_{GS}$ as $\cP_n \cdots \cP_1$. Then, $P_{GS}$
is $\gamma$-contractive nonnegative matrix since
\[\sum^n_{j=1}\para{P_{GS}}_{ij}=\sum^n_{j=1}\para{\cP_i\cdots\cP_1}_{ij} \le \sum^n_{j=1}\para{\cP_i}_{ij}=\gamma\]
for $1\le i\le n$, where first equality is from definition of $\cP_l$ for $i+1\le l\le n$, inequality comes from definition of $\cP_l$ for $1\le l\le i-1$, and last equality is induced by definition of $\cP_i$. Therefore, this implies that $\infn{P_{GS}} \le \gamma$. 

If $U=Q$, with similar argument of case $U=V$, let $\pi_i$ be the greedy policy, define matrix $\cP_i$ as
\begin{align*}
\cP_i(Q)=(Q_1,\dots,Q_{i-1},\para{\gamma\cP^{\pi_i}(Q)}_i,Q_{i+1},\dots,Q_n),
\end{align*}
and denote $P_{GS}$ as $\cP_n \cdots \cP_1$. Then, $P_{GS}$
is $\gamma$-contractive nonnegative matrix satisfying
\[T_{GS}^{\star}Q-(1-\alpha)T_{GS}^{\star}\tilde{Q}-\alpha T_{GS}^{\star}Q^{\star} \le  \cP_{GS}(Q-(1-\alpha) \tilde{Q}-\alpha Q^{\star}).\]
\end{proof}

\begin{proof}[Proof of Lemma \ref{lem::GS-Anc-VIC-H2}]
    First let $U=V, \tilde{U}=\tilde{V}, \hat{U}^{\star}=\hat{V}^{\star}$. For $1 \le i \le n$, we have
\begin{align*}
    &T^{\star}_iV(s_i)-(1-\alpha) T_i^{\star}\tilde{V}(s_i)-\alpha \hat{T}_i^{\star}\hat{V}^{\star}(s_i)\\& = \sup_{a \in \cA} \bigg\{r(s_i,a)+\gamma \mathbb{E}_{s'\sim P(\cdot\,|\,s_i,a) }\left[V(s')\right]\bigg\} - \sup_{a \in \cA} \bigg\{(1-\alpha) r(s_i,a)+(1-\alpha) \gamma\mathbb{E}_{s'\sim P(\cdot\,|\,s_i,a) }\left[\tilde{V}(s')\right]\bigg\}\\
    &\quad - \inf_{a \in \cA} \bigg\{\alpha r(s_i,a)+\alpha \gamma\mathbb{E}_{s'\sim P(\cdot\,|\,s_i,a) }\left[\hat{V}^{\star}(s')\right]\bigg\}\\
    & \ge \gamma\inf_{a \in \cA} \bigg\{ \mathbb{E}_{s'\sim P(\cdot\,|\,s_i,a) }\left[V(s')-(1-\alpha) \tilde{V}(s')-\alpha \hat{V}^{\star}(s')\right]\bigg\}.
\end{align*}
Let $\hat{\pi}_i(\cdot\,|\,s)=\argmin_{a\in \cA} \mathbb{E}_{s'\sim P(\cdot\,|\,s,a) }\left[V(s')-(1-\alpha) \tilde{V}(s')-\alpha \hat{V}^{\star}(s')\right]$ and define matrix $\hat{\cP}_i$ as
\begin{align*}
\hat{\cP}_i(V)=(V_1,\dots,V_{i-1},\para{\gamma\cP^{\hat{\pi}_i}(V)}_i,V_{i+1},\dots,V_n)
    \end{align*}
for $i=1,\dots,n$. Note that $\hat{\cP}_i$ is nonnegative matrix since $\cP^{\hat{\pi}_i}$ is nonnegative matrix. Then, we have
\[\hat{\cP}_i(V-(1-\alpha) \tilde{V}-\alpha \hat{V}^{\star}) \le T_i^{\star}V-(1-\alpha)T_i^{\star}\tilde{V}-\alpha T_i^{\star}\hat{V}^{\star}.\]
By induction, there exist a sequence of matrices $\{\hat{\cP}_i\}_{i=1,\dots,n}$ satisfying 
\[\hat{\cP}_n\cdots\hat{\cP}_1(V-(1-\alpha) \tilde{V}-\alpha \hat{V}^{\star})\le T_{GS}^{\star}V-(1-\alpha)T_{GS}^{\star}\tilde{V}-\alpha \hat{T}_{GS}^{\star}\hat{V}^{\star},\]
and denote $\hat{P}_{GS}$ as $\hat{\cP}_n \cdots \hat{\cP}_1$. With same argument in proof of Lemma \ref{lem::GS-Anc-VIC-H1}, $\hat{P}_{GS}$
is $\gamma$-contractive nonnegative matrix.

If $U=Q$, with similar argument, let $\hat{\pi}_{i}(\cdot\,|\,s)=\argmin_{a\in \cA}\{ Q(s,a)-(1-\alpha)\tilde{Q}(s,a)-\alpha \hat{Q}^{\star}(s,a)\}$ and define matrix $\hat{\cP}_i$ as
\begin{align*}
\cP_i(Q)=(U_1,\dots,Q_{i-1},\para{\gamma\cP^{\hat{\pi}_i}(Q)}_i,Q_{i+1},\dots,Q_n).
\end{align*}
Denote $\hat{P}_{GS}$ as $\hat{\cP}_n \cdots \hat{\cP}_1$. Then, with same argument in proof of Lemma \ref{lem::GS-Anc-VIC-H1}, $\hat{P}_{GS}$
is $\gamma$-contractive nonnegative matrix satisfying 
\[\hat{\cP}_{GS}(Q-(1-\alpha) \tilde{Q}-\alpha \hat{Q}^{\star})\le T_{GS}^{\star}Q-(1-\alpha)T_{GS}^{\star}\tilde{Q}-\alpha \hat{T}_{GS}^{\star}\hat{Q}^{\star}.\]
\end{proof}
Next, we prove following key lemma.
\begin{lemma}\label{lem::GS-Anc-VIC}
Let $0<\gamma<1$. For the iterates $\{U^k\}_{k=0,1,\dots}$ of  \eqref{eq:GS-Anc-VI}, there exist $\gamma$-contractive nonnegative matrices $\{\cP_{GS}^l\}_{l=0,1,\dots,k}$ and $\{\hat{\cP}_{GS}^l\}_{l=0,1,\dots,k}$ such that
\begin{align*}
    T_{GS}^{\star}U^k-U^k &\le \sum_{i=1}^k \left[(\beta_i-\beta_{i-1}(1-\beta_{i}))\para{\Pi^k_{j=i+1}(1-\beta_j)}\para{\Pi^i_{l=k}\cP_{GS}^l}(U^0-U^{\star})\right]
    \\&\quad -\beta_k(U^0-U^{\star})+\para{\Pi^k_{j=1}(1-\beta_j)} \para{\Pi^0_{l=k}\cP_{GS}^l}(U^0-U^{\star}),\\
     T_{GS}^{\star}U^k-U^k &  \ge \sum_{i=1}^k \left[(\beta_i-\beta_{i-1}(1-\beta_{i}))\para{\Pi^k_{j=i+1}(1-\beta_j)}\para{\Pi^i_{l=k}\hat{\cP}_{GS}^l}(U^0-\hat{U}^{\star})\right]
    \\&\quad -\beta_k(U^0-\hat{U}^{\star})+\para{\Pi^k_{j=1}(1-\beta_j)} \para{\Pi^0_{l=k}\hat{\cP}_{GS}^l}(U^0-\hat{U}^{\star}), 
\end{align*}
where $\Pi^k_{j=k+1}(1-\beta_j)=1$ and $\beta_0=1$.
\end{lemma}
\begin{proof}[Proof of Lemma \ref{lem::GS-Anc-VIC}]
First, we prove first inequality in Lemma \ref{lem::GS-Anc-VIC} by induction.
 
    If $k=0$, 
    \begin{align*}
    T_{GS}^{\star}U^0-U^0&=T_{GS}^{\star}U^0-U^{\star}-(U^0-U^{\star})
    \\&=T_{GS}^{\star}U^0-T_{GS}^{\star}U^{\star}-(U^0-U^{\star})
    \\& \le \cP_{GS}^{0}(U^0-U^{\star})-(U^0-U^{\star}).
\end{align*}
where inequality comes from Lemma \ref{lem::GS-Anc-VIC-H1} with $\alpha=1, U=U^0$. 

By induction,
 \begin{align*}
    &T_{GS}^{\star}U^k-U^k\\
    &=T_{GS}^{\star}U^k-(1-\beta_k)T_{GS}^{\star}U^{k-1}-\beta_kT_{GS}^{\star}U^{\star}-\beta_k(U^0-U^{\star})\\
    &\le \cP_{GS}^k(U^k-(1-\beta_k)U^{k-1}-\beta_kU^{\star})-\beta_k(U^0-U^{\star})\\
    &= \cP_{GS}^k(\beta_k(U^0-U^{\star})+(1-\beta_k)(T_{GS}^{\star}U^{k-1}-U^{k-1}))-\beta_k(U^0-U^{\star})\\
     &\le (1-\beta_k)\cP_{GS}^k \bigg(  \sum_{i=1}^{k-1} \left[(\beta_i-\beta_{i-1}(1-\beta_{i}))\para{\Pi^{k-1}_{j=i+1}(1-\beta_j)}\para{\Pi^i_{l=k-1}\cP_{GS}^l}(U^0-U^{\star})\right]
    \\&\quad -\beta_{k-1}(U^0-U^{\star})+\para{\Pi^{k-1}_{j=1}(1-\beta_j)} \para{\Pi^0_{l=k-1}\cP_{GS}^l}(U^0-U^{\star})\bigg)\\& \quad +\beta_{k}\cP_{GS}^k(U^0-U^{\star})-\beta_{k}(U^0-U^{\star})\\
    % & = (1-\beta_k)\cP_{GS}^k\sum_{i=1}^{k-1} \left[(\beta_i-\beta_{i-1}(1-\beta_{i}))\para{\Pi^{k-1}_{j=i+1}(1-\beta_j)}\para{\Pi^{i}_{l=k-1}\cP_{GS}^l}(U^0-U^{\star})\right]
    % \\&\quad -\beta_{k-1}(1-\beta_k)\cP_{GS}^k(U^0-U^{\star})+(1-\beta_k)\cP_{GS}^k\para{\Pi^{k-1}_{j=1}(1-\beta_j)} \para{\Pi^{0}_{l=k-1}\cP_{GS}^l}(U^0-U^{\star})\\
    % &\quad +\beta_k\cP_{GS}^k\para{U^{0}-U^{\star}}-\beta_k(U^0-U^{\star})\\
    & = \sum_{i=1}^k \left[(\beta_i-\beta_{i-1}(1-\beta_{i}))\para{\Pi^k_{j=i+1}(1-\beta_j)}\para{\Pi^i_{l=k}\cP_{GS}^l}(U^0-U^{\star})\right]
    \\&\quad -\beta_k(U^0-U^{\star})+\para{\Pi^k_{j=1}(1-\beta_j)} \para{\Pi^0_{l=k}\cP_{GS}^l}(U^0-U^{\star})
\end{align*}
where the first inequality comes from Lemma \ref{lem::GS-Anc-VIC-H1} with $\alpha=\beta_k, U=U^k, \tilde{U}=U^{k-1}$, and second inequality comes from nonnegativeness of $\cP_{GS}^k$.

First, we prove second inequality in Lemma \ref{lem::GS-Anc-VIC} by induction.

If $k=0$, 
    \begin{align*}
    T_{GS}^{\star}U^0-U^0&=T_{GS}^{\star}U^0-\hat{U}^{\star}-(U^0-\hat{U}^{\star})
    \\&=T_{GS}^{\star}U^0-\hat{T}_{GS}^{\star}\hat{U}^{\star}-(U^0-\hat{U}^{\star})
    \\& \ge \hat{\cP}_{GS}^{0}(U^0-\hat{U}^{\star})-(U^0-\hat{U}^{\star}),
\end{align*}
where inequality comes from Lemma \ref{lem::GS-Anc-VIC-H2} with $\alpha=1, U=U^0$. 

By induction,
 \begin{align*}
    &T_{GS}^{\star}U^k-U^k\\
    &=T_{GS}^{\star}U^k-(1-\beta_k)T_{GS}^{\star}U^{k-1}-\beta_k\hat{T}_{GS}^{\star}\hat{U}^{\star}-\beta_k(U^0-\hat{U}^{\star})\\
    &\ge \hat{\cP}_{GS}^k(U^k-(1-\beta_k)U^{k-1}-\beta_k\hat{U}^{\star})-\beta_k(U^0-\hat{U}^{\star})\\
    &= \hat{\cP}_{GS}^k(\beta_k(U^0-\hat{U}^{\star})+(1-\beta_k)(T_{GS}^{\star}U^{k-1}-U^{k-1}))-\beta_k(U^0-\hat{U}^{\star})\\
     &\ge (1-\beta_k)\hat{\cP}_{GS}^k \bigg(  \sum_{i=1}^{k-1} \left[(\beta_i-\beta_{i-1}(1-\beta_{i}))\para{\Pi^{k-1}_{j=i+1}(1-\beta_j)}\para{\Pi^i_{l=k-1}\hat{\cP}_{GS}^l}(U^0-\hat{U}^{\star})\right]
    \\&\quad -\beta_{k-1}(U^0-\hat{U}^{\star})+\para{\Pi^{k-1}_{j=1}(1-\beta_j)} \para{\Pi^0_{l=k-1}\hat{\cP}_{GS}^l}(U^0-\hat{U}^{\star})\bigg)\\& \quad +\beta_{k}\hat{\cP}_{GS}^k(U^0-\hat{U}^{\star})-\beta_{k}(U^0-\hat{U}^{\star})\\
    % & = (1-\beta_k)\hat{\cP}_{GS}^k\sum_{i=1}^{k-1} \left[(\beta_i-\beta_{i-1}(1-\beta_{i}))\para{\Pi^{k-1}_{j=i+1}(1-\beta_j)}\para{\Pi^{i}_{l=k-1}\hat{\cP}_{GS}^l}(U^0-\hat{U}^{\star})\right]
    % \\&\quad -\beta_{k-1}(1-\beta_k)\hat{\cP}_{GS}^k(U^0-\hat{U}^{\star})+(1-\beta_k)\hat{\cP}_{GS}^k\para{\Pi^{k-1}_{j=1}(1-\beta_j)} \para{\Pi^{0}_{l=k-1}\hat{\cP}_{GS}^l}(U^0-\hat{U}^{\star})\\
    % &\quad +\beta_k\hat{\cP}_{GS}^k\para{U^{0}-\hat{U}^{\star}}-\beta_k(U^0-\hat{U}^{\star})\\
    & = \sum_{i=1}^k \left[(\beta_i-\beta_{i-1}(1-\beta_{i}))\para{\Pi^k_{j=i+1}(1-\beta_j)}\para{\Pi^i_{l=k}\hat{\cP}_{GS}^l}(U^0-\hat{U}^{\star})\right]
    \\&\quad -\beta_k(U^0-\hat{U}^{\star})+\para{\Pi^k_{j=1}(1-\beta_j)} \para{\Pi^0_{l=k}\hat{\cP}_{GS}^l}(U^0-\hat{U}^{\star})
\end{align*}
where the first inequality comes from Lemma \ref{lem::GS-Anc-VIC-H2} with $\alpha=\beta_k, U=U^k, \tilde{U}=U^{k-1}$, and nonnegativeness of $\hat{\cP}_{GS}^k$.
\end{proof}

Now, we prove the first rate in Theorem \ref{thm::GS-Anc-VI}.
\begin{proof}[Proof of first rate in Theorem \ref{thm::GS-Anc-VI}]
Since $B_1 \le A \le B_2$ implies $ \infn{A} \le \sup\{\infn{B_1}, \infn{B_2}\}$ for $A,B \in \cF(\cX)$, if we take $\infn{\cdot}$ right side of first inequality in Lemma \ref{lem::GS-Anc-VIC}, we have 
\begin{align*}
    &\sum_{i=1}^k \abs{\beta_i-\beta_{i-1}(1-\beta_{i})}\para{\Pi^k_{j=i+1}(1-\beta_j)}\infn{\para{\Pi^i_{l=k}\cP_{GS}^{l}}(U^0-U^{\star})}
    \\&\quad +\beta_k\infn{U^0-U^{\pi}}+\para{\Pi^k_{j=1}(1-\beta_j)} \infn{\para{\Pi^0_{l=k}\cP_{GS}^{l}}(U^0-U^{\star})}\\&\le \para{\sum_{i=1}^k \gamma^{k-i+1}\abs{\beta_i-\beta_{i-1}(1-\beta_{i})}\para{\Pi^k_{j=i+1}(1-\beta_j)}+\beta_k+\gamma^{k+1}\Pi^k_{j=1}(1-\beta_j)}\\&\quad\infn{U^0-U^{\star}}\\&=\frac{\para{\gamma^{-1}-\gamma}\para{1+2\gamma-\gamma^{k+1}}}{\para{\gamma^{k+1}}^{-1}-\gamma^{k+1}}\infn{U^0-U^{\star}},
\end{align*}
where the first inequality comes from triangular inequality, second inequality is from 
$\gamma$-contraction of $\cP_{GS}^l$, and last equality comes from calculations. If we take $\infn{\cdot}$ right side of second inequality in Lemma \ref{lem::GS-Anc-VIC}, we have
\begin{align*}
    &\sum_{i=1}^k \abs{\beta_i-\beta_{i-1}(1-\beta_{i})}\para{\Pi^k_{j=i+1}(1-\beta_j)}\infn{\para{\Pi^i_{l=k} \hat{\cP}_{GS}^{l}}(U^0-\hat{U}^{\star})}
    \\&\quad +\beta_k\infn{U^0-U^{\pi}}+\para{\Pi^k_{j=1}(1-\beta_j)} \infn{\para{\Pi^0_{l=k}\hat{\cP}_{GS}^{l}}(U^0-\hat{U}^{\star})}\\&\le \para{\sum_{i=1}^k \gamma^{k-i+1}\abs{\beta_i-\beta_{i-1}(1-\beta_{i})}\para{\Pi^k_{j=i+1}(1-\beta_j)}+\beta_k+\gamma^{k+1}\Pi^k_{j=1}(1-\beta_j)}\\&=\frac{\para{\gamma^{-1}-\gamma}\para{1+2\gamma-\gamma^{k+1}}}{\para{\gamma^{k+1}}^{-1}-\gamma^{k+1}}\infn{U^0-\hat{U}^{\star}},
\end{align*}
where the first inequality comes from triangular inequality, second inequality is from 
from 
$\gamma$-contraction of $\hat{\cP}_{GS}^l$, and last equality comes from calculations. Therefore, we conclude  
\begin{align*}
    \infn{T_{GS}^{\star}U^k-U^k} &\le \frac{\para{\gamma^{-1}-\gamma}\para{1+2\gamma-\gamma^{k+1}}}{\para{\gamma^{k+1}}^{-1}-\gamma^{k+1}}\max{\left\{\infn{U^0-U^{\star}},\infn{U^0-\hat{U}^\star}\right\}}.
\end{align*}
\end{proof}

For the second rates of Theorem \ref{thm::GS-Anc-VI}, we introduce following lemma.
\begin{lemma}\label{lem::Inital_cond_2}
Let $0<\gamma<1$. For the iterates $\{U^k\}_{k=0,1,\dots}$ of \eqref{eq:GS-Anc-VI}, if $U^0 \le T_{GS}^{\star}U^0$,  then $U^{k-1} \le U^k \le T^{\star}_{GS}U^{k-1} \le T^{\star}_{GS}U^k \le U^{\star}$ for $1 \le k$. Also, if $U^0 \ge T^{\star}_{GS}U^0$, then $U^{k-1} \ge U^k \ge T^{\star}_{GS}U^{k-1} \ge T^{\star}_{GS}U^k \ge U^{\star}$ for $1 \le k$.
\end{lemma}
\begin{proof}
By Fact \ref{lem::GS}, $\lim_{m\rightarrow\infty} T_{GS}^{\star}U = U^{\star}$. By definition, if $U \le \tilde{U}$, $T^{\star}_iU \le T_i^{\star}\tilde{U}$ for any $1\le i \le n$ and this implies that if $U \le \tilde{U}$, then $T_{GS}^{\star}U \le T_{GS}^{\star}\tilde{U}$. Hence, with same argument in proof of Lemma \ref{lem::Inital_cond}, we can obtain desired results.   
    \end{proof}

Now, we prove the second rates in Theorem \ref{thm::GS-Anc-VI}.

\begin{proof}[Proof of second rates in Theorem \ref{thm::GS-Anc-VI}]
If $U^0 \le T_{GS}^{\star}U^0$, then $U^0-U^{\star} \le 0$ and $U^k \le T_{GS}^{\star}U^k$ by Lemma \ref{lem::Inital_cond_2}. Hence, by Lemma \ref{lem::GS-Anc-VIC}, we get 
    \begin{align*}
        0 &\le T_{GS}^{\star}U^k-U^k\\&=\sum_{i=1}^k\left[\para{\beta_i-\beta_{i-1}(1-\beta_{i})}\para{\Pi^k_{j=i+1}(1-\beta_j)}\para{\Pi^i_{l=k}\cP_{GS}^l}(U^0-U^{\star})\right]\\
    &\quad -\beta_k(U^0-U^{\pi})+\para{\Pi^k_{j=1}(1-\beta_j)} \para{\Pi^0_{l=k}\cP_{GS}^{l}}(U^0-U^{\star})\\
    &\le \sum_{i=1}^k\left[\para{\beta_i-\beta_{i-1}(1-\beta_{i})}\para{\Pi^k_{j=i+1}(1-\beta_j)}\para{\Pi^i_{l=k}\cP_{GS}^l}(U^0-U^{\star})\right] -\beta_k(U^0-U^{\star}),  
    \end{align*}
where the second inequality follows from $\para{\Pi^k_{j=1}(1-\beta_j)} \para{\Pi^i_{l=k}\cP_{GS}^l}(U^0-U^{\star})\le 0$. Taking $\infn{\cdot}$-norm both sides, we have 
\[\infn{T_{GS}^{\star}U^k-U^k} \le\frac{\para{\gamma^{-1}-\gamma}\para{1+\gamma-\gamma^{k+1}}}{\para{\gamma^{k+1}}^{-1}-\gamma^{k+1}}\infn{U^0-U^{\star}}.\]
Otherwise, if $U^0 \ge T_{GS}^{\star}U^0$, $U^k \ge T_{GS}^{\star}U^k$ and $U^0 \ge U^{\star} \ge \hat{U}^{\star}$  by Lemma \ref{lem::Inital_cond_2} and \ref{lem::anti-fixed-point}. Thus, by Lemma \ref{lem::GS-Anc-VIC}, we get 
    \begin{align*}
        0 &\ge T_{GS}^{\star}U^k-U^k \\&\ge\sum_{i=1}^k\left[\para{\beta_i-\beta_{i-1}(1-\beta_{i})}\para{\Pi^k_{j=i+1}(1-\beta_j)}\para{\Pi^i_{l=k}\hat{\cP}_{GS}^l}(U^0-\hat{U}^{\star})\right]-\beta_k(U^0-\hat{U}^{\star}),
    \end{align*}
where the second inequality follows from $0 \le \para{\Pi^k_{j=1}(1-\beta_j)} \para{\Pi^0_{l=k}\hat{\cP}_{GS}^l}(U^0-\hat{U}^{\star})$. Taking $\infn{\cdot}$-norm both sides, we have
\[\infn{T_{GS}^{\star}U^k-U^k} \le \frac{\para{\gamma^{-1}-\gamma}\para{1+\gamma-\gamma^{k+1}}}{\para{\gamma^{k+1}}^{-1}-\gamma^{k+1}}\infn{U^0-\hat{U}^{\star}}.\]
\end{proof}

% \vspace{0.2in}

% \newpage

\section{Broader Impacts}

Our work focuses on the theoretical aspects of reinforcement learning. There are no negative social impacts that we anticipate from our theoretical results.

\section{Limitations}
Our analysis concerns value iteration. While value iteration is of theoretical interest, the analysis of value iteration is not sufficient to understand modern deep reinforcement learning practices.

\end{document}